\documentclass[accepted]{uai2024} % after acceptance, for a revised version; 
\usepackage[american]{babel}
\usepackage{natbib} % has a nice set of citation styles and commands
\usepackage{mathtools} % amsmath with fixes and additions
\usepackage{booktabs} % commands to create good-looking tables
\usepackage{tikz} % nice language for creating drawings and diagrams

\usepackage{ascmac}
\usepackage{float}
\usepackage{perpage}
\MakeSorted{figure}
\MakeSorted{table}

\usepackage{url}
\usepackage{natbib}
\usepackage{chapterbib}

\usepackage{color}
\usepackage{tikz}
\tikzset{%
mynode/.style={circle,minimum width=.5ex, fill=none,draw}, % no filling
myfillnode/.style={circle,minimum width=.5ex, fill=lightgray,draw}, % fill with black
}
\usepackage{amssymb}
\usepackage{natbib}

\newcommand{\indep}{\perp \!\!\! \perp}
\usepackage{amsmath}               
\usepackage{lscape}
\usepackage{algorithm}
\usepackage{color}
\usepackage{tikz}
\usepackage{amsmath,amsthm}
\newtheorem{theorem}{Theorem}[section]
\newtheorem{definition}{Definition}[section]
\newtheorem{assumption}{Assumption}[section]
\newtheorem{lemma}{Lemma}[section]

\newtheorem{corollary}{Corollary}[section]

\newtheorem{example}{Example}[section]

\usepackage{comment}
\usepackage{here}
\allowdisplaybreaks[4]
\usepackage{bbm}
\usepackage{caption}
\usepackage{bbding}
\usepackage{arydshln}
\usepackage{afterpage}

\usepackage{mathrsfs}

\usepackage{soul}
\newcommand{\jin}[1]{\textcolor{blue}{#1}}
\newcommand{\yuta}[1]{\textcolor{red}{#1}}

\usepackage{listings}
\floatstyle{ruled}
\newfloat{listing}{tb}{lst}{}
\floatname{listing}{Listing}
\usepackage[utf8]{inputenc} % allow utf-8 input
\usepackage[T1]{fontenc}    % use 8-bit T1 fonts
\usepackage{hyperref}       % hyperlinks
\usepackage{url}            % simple {\boldsymbol U}RL typesetting
\usepackage{booktabs}       % professional-quality tables
\usepackage{amsfonts}       % blackboard math symbols
\usepackage{nicefrac}       % compact symbols for 1/2, etc.
\usepackage{microtype}      % microtypography
\usepackage{xcolor}         % colors

\usepackage{stackengine}
\def\defeq{\mathrel{\ensurestackMath{\stackon[1pt]{=}{\scriptscriptstyle\Delta}}}}

\title{Probabilities of Causation for Continuous and Vector Variables}

\author[1,2]{Yuta Kawakami}
\author[1]{Manabu Kuroki}
\author[2]{Jin Tian}

\affil[1]{%
Department of Mathematics, Physics, Electrical Engineering and Computer Science, Yokohama National University, Yokohama, Kanagawa, JAPAN
}
\affil[2]{%
Department of Computer Science, Iowa State University, Ames, Iowa, USA
}

\begin{document}
  
\maketitle

\begin{abstract}
{\it Probabilities of causation} (PoC) are valuable concepts for explainable artificial intelligence and practical decision-making. PoC are originally defined for scalar binary variables.   
In this paper, we extend the concept of PoC to continuous treatment and outcome variables, and further generalize PoC  to capture causal effects between multiple treatments and multiple outcomes. 
In addition, we consider PoC for a sub-population and PoC with multi-hypothetical terms to capture more sophisticated counterfactual information useful for decision-making. 
We provide a nonparametric identification theorem for each type of PoC we introduce. 
Finally, we illustrate the application of our results on a real-world dataset about education.
\end{abstract}

%Keywords: Probabilities of causation, Continuous variables, Vector variales
%TL;DR In this paper, we extend the concept of the probabilities of causation to multiple continuous treatments and outcomes. 

\section{Introduction}

{\it Probabilities of causation} (PoC) are a family of probabilities quantifying whether one event was the real cause of another in a given scenario \citep{Robins1989,Pearl1999,Tian2000,Pearl09,Kuroki2011,Dawid2014,Dawid2016,Dawid2017,Murtas2017,Hannart2018,Shingaki2021,Kawakami2023b}.
PoC are valuable quantities for decision-making \citep{Li2019,Li2022b} and  for explainable artificial intelligence (XAI) that aims to reduce the opaqueness of AI-based decision-making systems \citep{Galhotra2021,Watson2021}.
%PoC is important not only in artificial intelligence but also in meteorology \citep{Hannart2018} and in law \citep{Dawid2017b}.
\citet{Pearl1999} introduced three types of PoC over binary events, namely 
%There are three types of PoC regarding treatment ($X$) and outcome ($Y$), named 
the probability of necessity and sufficiency (PNS), the probability of necessity (PN), and the probability of sufficiency (PS). 
They are defined based on the joint probability distribution of two potential outcomes. 
\citet{Tian2000} provided the bounds of PNS, PN, and PS in terms of observational and experimental data and showed that PNS, PN, and PS are  identifiable under the assumptions of exogeneity and monotonicity.
The problem of bounding PoC was further extended in \citep{Li2019,Li2022b,Li2023,MuellerLiPearl}. However, all these works are restricted to binary treatment and outcome. 
More recently, \citet{Li2022,Li2022c} extended the problem of bounding PoC to multi-valued discrete treatment and outcome and provided  bounds for various variants of PoC.

In this paper, we aim to extend the concept of PoC to continuous treatment and outcome. 
There is considerable interest in continuous treatment and outcome  
%Recently, there has been a growing interest in a continuous treatment beyond a binary treatment 
in causal inference \citep{Imbens2004,Kennedy2017,Bahadori2022}, e.g., dose-response studies \citep{Wong1996,Emilien2000,Ivanova2008} and policy evaluations with continuous actions \citep{Kallus2018,Krishnamurthy2019,Majzoubi2020}.
%\citep{Tian2000} showed bounds of PNS, PN, and PS and gave the identification theorem of PNS, PN, and PS under assumptions of exogeneity and monotonicity. \citep{Li2022} extended their results to a multi-valued discrete treatment.
For instance, doctors want to know the dose-response relationship between the amount of insulin and the blood sugar level.

We provide a nonparametric identification theorem for each type of PoC we introduced. The identification of binary PoC relies on a monotonicity assumption \citep{Tian2000}. 
We generalize the monotonicity assumption over binary treatment and outcome to continuous settings.
We discuss the relationship of our proposed monotonicity assumption with another commonly used assumption in the causal inference literature - monotonicity over structural functions \citep{Heckman1999,Vytlacil2002,Heckman2005,Chernozhukov2005,Chernozhukov2007,Imbens2009}.
%, and rank preservation assumption \citep{Robins1989,Robins1991,Have2007,Vansteelandt2014,Bothmann2023,Hernan2023}. 

%\yuta{[Comment: I have deleted all rank preservation assumptions.]}

We further extend the concept of PoC to capture causal effects between multiple treatments and multiple outcomes, which are drawing growing interests \citep{Kang1990,Zhang1998,Sammel1999,Segal2011,Lee2012,Kennedy2019b,Rimal2019}.
For instance, \citet{Hannart2018} investigated causal links between anthropogenic forcings, e.g., greenhouse gases (carbon dioxide, methane, nitrous oxide, halocarbons) emission and deforestation, and the observed climate changes, e.g., spatial–temporal vector of  Earth surface temperature. 
They used a multivariate linear regression model with Gaussian noise to evaluate PoC. 

%In addition, \citep{Hannart2018} considered continuous and multiple treatments and outcomes under the parametric Gaussian setting. We discuss the identification assumptions of PoC applicable to both discrete and continuous cases in this paper.

We also introduce more complicated variants of PoC and provide identification theorems for them.
They include PoC for a sub-population with specific covariates information considered by \citep{Li2022b} and PoC with multi-hypothetical terms  studied by \citet{Li2022} for discrete treatment and outcome. These variants of PoC capture more sophisticated counterfactual information useful for decision-making.

Finally, we show an application of our results to a real-world dataset on education.

\section{Background and Notation}

%In this section, we show the backgrounds and notations in this paper.
We represent each variable with a capital letter $(X)$ and its realized value with a small letter $(x)$.
Let $\mathbb{I}(x)$ be an indicator function that takes $1$ if $x$ is true; and $0$ if $x$ is false.
Denote $\Omega_Y$ be the domain of $Y$,
$\mathbb{E}[Y]$ be the expectation of $Y$, $\mathbb{P}(Y\leq y)$ be the cumulative distribution function (CDF) of continuous variable $Y$, and $\mathbb{P}(Y)$ be the probability of discrete variable $Y$.
%In addition, let $\mathbb{P}(Y\leq y|X=x)$ and $\mathbb{P}(Y= y|X=x)$ be the conditional CDF and probability given $X=x$.
We denote $X \indep Y|C$ if $X$ and $Y$ are conditionally independent given $C$.

{\bf Total order over vector space.} 
We denote a total order on vectors of variables by $\prec$. 
For example, according to the lexicographical order \citep{Harzheim2005}, we order two dimensional vectors $(y_1, y_2) \prec_{\text{lexi}} (y'_1, y'_2)$ if ``$y_1 < y'_1$'', or ``$y_1 = y'_1$ and $y_2 < y'_2$''. 
A formal definition of the lexicographical order is given in Appendix~\ref{app0}.

{\bf Structural Causal Models (SCM).} We use the language of SCMs as our basic semantic and inferential framework \citep{Pearl09}.
An SCM ${\cal M}$ is a tuple $\left<{\boldsymbol V},{\boldsymbol U}, {\cal F}, \mathbb{P}_{\boldsymbol U} \right>$, where ${\boldsymbol U}$ is a set of exogenous (unobserved) variables following a  distribution $\mathbb{P}_{\boldsymbol U}$, and ${\boldsymbol V}$ is a set of endogenous (observable) variables whose values are determined by structural functions ${\cal F}=\{f_{V_i}\}_{V_i \in {\boldsymbol V}}$ such that $v_i:= f_{V_i}({\mathbf{pa}}_{V_i},{\boldsymbol u}_{V_i})$ where ${\mathbf{PA}}_{V_i} \subseteq {\boldsymbol V}$ and $U_{V_i} \subseteq {\boldsymbol U}$. 
Each SCM ${\cal M}$ induces an observational distribution $\mathbb{P}_{\boldsymbol V}$ over ${\boldsymbol V}$, and a causal graph $G({\cal M})$ over ${\boldsymbol V}$ in which there exists a directed edge from every variable in ${\mathbf{PA}}_{V_i}$ to $V_i$.
An intervention of setting a set of endogenous variables ${\boldsymbol X}$ to constants ${\boldsymbol x}$, denoted by $do({\boldsymbol x})$, replaces the original equations of ${\boldsymbol X}$
 by the constants ${\boldsymbol x}$ and induces a \textit{sub-model}  ${\cal M}_{{\boldsymbol x}}$.
We denote the potential outcome $Y$ under intervention $do({\boldsymbol x})$ by $Y_{{\boldsymbol x}}({\boldsymbol u})$, which is the solution of $Y$ in the sub-model ${\cal M}_{{\boldsymbol x}}$ given ${\boldsymbol U}={\boldsymbol u}$.

{\bf Probabilities of Causation (PoC).} 
%We show the definitions and an identification theorem of probabilities of causation with binary treatment and outcome.
Let $X$ be a binary treatment taking values $x_0$ and $x_1$, and $Y$ be a binary outcome taking values $y_0$ and $y_1$. PoC are defined as follows:
\begin{definition}[PoC]
\label{def1}
Probability of necessity and sufficiency (PNS), probability of necessity (PN), and probability of sufficiency (PS) are defined by   \citep{Pearl1999}:
    \begin{equation}
    \begin{aligned}
    &\text{PNS}\defeq\mathbb{P}(Y_{x_0}=y_0,Y_{x_1}=y_1),\\
    &\text{PN}\defeq\mathbb{P}(Y_{x_0}=y_0|Y=y_1,X=x_1),\\
   &\text{PS}\defeq\mathbb{P}(Y_{x_1}=y_1|Y=y_0,X=x_0).
        \end{aligned}
    \end{equation}
\end{definition}
\citet{Tian2000} show that these PoC are identified under the following assumptions.
%\yuta{They give two assumptions for identifying PNS, PN, and PS.
\begin{assumption}[Exogeneity]
\label{BEXO}
    $Y_{x_0} \indep X$ and $Y_{x_1} \indep X$.
\end{assumption}
\begin{assumption}[Monotonicity]
\label{BMONO}
    $\mathbb{P}(Y_{x_0}=y_1,Y_{x_1}=y_0)=0$.
\end{assumption}
%\citet{Tian2000} gave the following identification results:
%\begin{theorem}[Identification of PoC]
    %If $X$ is exogenous
    %, i.e., $Y_{x_0} \indep X$ and $Y_{x_1} \indep X$, 
    %$X \indep U$,
     %and $Y$ is monotonic relative to $X$,
     %, i.e., $\mathbb{P}(Y_{x_0}=y_1,Y_{x_1}=y_0)=0$,
     Under Assumptions \ref{BEXO} and \ref{BMONO},
     the PoC are identifiable by \citep{Tian2000}
    \begin{equation}
        \begin{aligned}
        &\text{PNS}=\mathbb{P}(Y=y_1|X=x_1)-\mathbb{P}(Y=y_1|X=x_0),\\
        &\text{PN}=\frac{\mathbb{P}(Y=y_1|X=x_1)-\mathbb{P}(Y=y_1|X=x_0)}{\mathbb{P}(Y=y_1|X=x_1)},\\
        &\text{PS}=\frac{\mathbb{P}(Y=y_1|X=x_1)-\mathbb{P}(Y=y_1|X=x_0)}{\mathbb{P}(Y=y_0|X=x_0)}.
        \end{aligned}
    \end{equation}
%\end{theorem}

%\jin{Move the whole Orders part to the appendix.}\\

%\section{Scalar Structural Causal Model}
\section{PoC for Scalar Continuous Variables}
\label{sec3}
For ease of understanding, we will start with a single treatment variable $X$ and a single outcome $Y$. 
We extend binary PoC  for continuous variables, extend the monotonicity Assumption~\ref{BMONO} to continuous settings, and provide an identification theorem.

%In this section, we set up the problem of PoC with scalar $Y$ and $X$ using a scalar structural causal model and give an identification theorem of PNS, PN and PS.

%\subsection{Problem Setup}
\subsection{PoC definition}

%{\bf Scalar Structural Causal Model.}

%\jin{I wonder whether it's necessary to define such an SCM. It implies strong Exogeneity, maybe you just need weak Exogeneity? The original PNS is identified either under strong Exogeneity or weak Exogeneity + Monotonicity. Is the newly defined PNS identifiable under strong Exogeneity alone? }

Let $X$ be a continuous or discrete treatment variable, and $Y$ be a continuous or discrete outcome variable.
We assume the following SCM ${\cal M}_{S}$: 
\begin{eqnarray}
    Y:=f_Y(X,U),\  X:=f_X(\epsilon_X), 
    %\ \text{where}\ \epsilon_X \indep U.
\end{eqnarray}
%The functions $f_Y$ and $f_X$ are  real-valued functions. Potential outcome $Y_x(U)$ is defined by $f_Y(x,U)$ for any $x \in \Omega_X$.
where $U$ and $\epsilon_X$ are latent exogenous variables. %, and $\epsilon_X$ is irrelevant to potential outcome $Y_x(U)$.
%We name ${\cal M}_{S}$ scalar SCM.
%$U$ satisfies the following condition:
%\begin{assumption}[Absolute Completeness]
%\label{TOT}
%The probability distribution of $U$ is absolutely continuous.
    %$\Omega_U$ is totally ordered set, and 
%    $\sup_{u \in \Omega_U}\mathfrak{p}(u)<0$.
%\end{assumption}
%Assumption \ref{TOT} means CDF of $U$ is continuous and guarantees the existence of PDF of $U$.
%Assumption \ref{TOT} also means the probability of two subject have the same value of $U$ is zero, i.e., $\mathbb{P}(u_0=u_1)=0$ for any $u_0, u_1 \in \Omega_U$.
%$\mathfrak{p}(u)$ means the probability density function of $U$.
%\citet{Heckman1999,Heckman2005} also make this assumption, and consider $U$ is a uniform distribution on $[0,1]$.
%
%The domains of $X,Y,U$, $\Omega_X, \Omega_Y, \Omega_U$, are subsets of $\mathbb{N}, \mathbb{Z}$ or $\mathbb{R}$.

%Under SCM ${\cal M}_S$, 
%the potential outcome $Y_x$ is given as $f_{Y}(x,U)$, and 
%we assume
%\begin{assumption}[Exogeneity]
%\label{EXO}
%    $X$ is independent of $U$.
%\end{assumption}

We make the following assumption.
\begin{assumption}[Exogeneity]
\label{ASEXO}
%   Under SCM ${\cal M}_S$, 
%$X$ is independent of $U$.
$Y_x\indep X$ for all $x \in \Omega_X$.
\end{assumption}
We note that if $\epsilon_X \indep U$ then the exogeneity holds,   and randomized controlled trials (RCT) on $X$ ensure exogeneity. Exogeneity implies $\mathbb{P}(Y_x < y)=\mathbb{P}(Y < y|X=x)$. 

%\yuta{[Comment: I fixed the assumption slightly.]}

We define PoC for continuous or discrete $X$ and $Y$ as a generalization of Definition \ref{def1}.
\begin{definition}[Probabilities of causation]
\label{def2}
For any $x_0,x_1 \in \Omega_X$ and $y \in \Omega_Y$, we define three types of PoC  as below:\footnote{We can equally define PNS as $\text{PNS}(y;x_0,x_1)\defeq\mathbb{P}(Y_{x_0} \leq y < Y_{x_1})$. We will stay with Definition~\ref{def2} in this paper.   
%However, we only discuss the definition of $\text{PNS}(y;x_0,x_1)\defeq\mathbb{P}(Y_{x_0} < y \leq Y_{x_1})$ in this paper. \yuta{I will change the definition $\text{PNS}(y;x_0,x_1)\defeq\mathbb{P}(Y_{x_0} < y \leq Y_{x_1})$ to $\text{PNS}(y;x_0,x_1)\defeq\mathbb{P}(Y_{x_0} < y \leq Y_{x_1}).$}}
}
\begin{equation}
\begin{aligned}
    &\text{PNS}(y;x_0,x_1)\defeq\mathbb{P}(Y_{x_0} < y \leq Y_{x_1}),\\
    &\text{PN}(y;x_0,x_1)\defeq\mathbb{P}(Y_{x_0} < y |y \leq Y, X=x_1),\\
    &\text{PS}(y;x_0,x_1)\defeq\mathbb{P}(y \leq Y_{x_1} |Y < y,X=x_0).
\end{aligned}
\end{equation}
\end{definition}
%\begin{comment}
%\jin{Can PNS be expressed in terms of PN and PS as in the binary case?} 
PNS, PN, and PS are connected  in the special case of binary $X$: 
\begin{lemma} 
\label{LEM41}
If $X$ is binary, we have
\begin{equation}
\begin{aligned}
 \text{PNS}(y;x_0,x_1)
    &=\text{PN}(y;x_0,x_1)\mathbb{P}(y \leq Y, X=x_1)\\
    &\ \ +\text{PS}(y;x_0,x_1)\mathbb{P}(Y < y,X=x_0).
    \end{aligned}
\end{equation}
\end{lemma}

%\end{comment}
\paragraph{Remark on the connection of Def. \ref{def2} with the binary PoC in Def. \ref{def1}:} 
Suppose $Y$ is binary with values $y_0 < y_1$, then Def. \ref{def2} with $y=y_1$ reduces to Def. \ref{def1}. 
In general, the value of $y$ in Def.~\ref{def2} can be interpreted as an outcome threshold, such as the  passing score for a test or a diagnostic threshold for blood pressure or blood glucose levels. 
%\citet{Hannart2018} determined the value of $y$ by maximizing $\text{PNS}(y;x_0,x_1)$, i.e., $y^*= \argmax_{y \in \Omega_Y} \text{PNS}(y;x_0,x_1)$.
Def. \ref{def2} focuses on the necessity/sufficiency of treatment $x_1$ w.r.t. $x_0$ to produce the event $Y\geq y$. 
We may introduce a binary outcome variable $O=\mathbb{I}(Y \geq y)$. Then PNS$(y;x_0,x_1) = \mathbb{P}(O_{x_0}=0,O_{x_1}=1)$, PN$(y;x_0,x_1) = \mathbb{P}(O_{x_0}=0|O=1, X=x_1)$, and PS$(y;x_0,x_1) = \mathbb{P}(O_{x_1}=1| O=0, X=x_0)$. Therefore, Def. \ref{def2} reduces to the standard definition of binary PoC over $X$ and $O$. We note that this interpretation of PNS matches the use of PNS in \citep{Hannart2018}.

{Although Def. \ref{def2} can be interpreted in terms of a binarized outcome $O=\mathbb{I}(Y \geq y)$. It is more natural and consistent to have a formulation in terms of the original variable $Y$ rather than in terms of $O$. A major benefit of the proposed formulation is that it can be naturally extended to study more complex variants of PoC in Section~\ref{sec-variant} that are difficult to formulate in terms of a binarized outcome.}

%The definition of PNS in Def. \ref{def2} appeared in \citep{Hannart2018}. \jin{It doesn't look to me that \citep{Hannart2018} gave the same definition. They essentially introduced a binary variable by defining an event as $E= Y\geq y$. Could you use this idea to provide a justification/interpretation to your definition?}
%\yuta{They define PNS $\mathbb{P}(E_{x_0}=0,E_{x_1}=1)$ through binarized outcome $E=\mathbb{I}(Y \geq y)$, and $E_{x_1}$ and $E_{x_0}$ mean $\mathbb{I}(Y_{x_1} \geq y)$ and $\mathbb{I}(Y_{x_0} \geq y)$. Thus, $\mathbb{P}(E_{x_0}=0,E_{x_1}=1)=\mathbb{P}(\mathbb{I}(Y_{x_0} \geq y)=0,\mathbb{I}(Y_{x_1} \geq y)=1)=\mathbb{P}(\mathbb{I}(Y_{x_0} < y)=1,\mathbb{I}(Y_{x_1} \geq y)=1)=\mathbb{P}(Y_{x_1} \geq y,Y_{x_0} < y)=\mathbb{P}(Y_{x_0} < y\leq Y_{x_1})$.}

When $X$ and $Y$ are discrete variables taking values $\{x_1,\ldots,x_P\}$ and $\{y_1,\ldots,y_Q\}$,
\citet{Li2022} defined PNS by $\mathbb{P}(Y_{x_{i_1}}=y_{j_1},Y_{x_{i_2}}=y_{j_2})$ ($1\leq i_1, i_2 \leq P$, $1\leq j_1, j_2 \leq Q$, $i_1 \ne i_2$ and $j_1 \ne j_2$).
However, their definition is  %always takes $0$ if $Y$ is a continuous variable with bounded PDF and 
not suitable for a continuous outcome $Y$.

\begin{example} \label{ex-1}
Consider the dose-response relationship between the blood sugar level ($Y$) and the amount of insulin ($X$). 
Let $y$ be a blood sugar threshold, and  $x_0,x_1$ be two insulin amount ($x_0>x_1$). 
A doctor may want to know the probability (PNS) that the patient's  blood sugar level would be greater than or equal to the  threshold $y$ had they taken $x_1$ amount of insulin, and would be less than $y$ had they taken $x_0$ insulin.
PN represents the probability that the patient's  blood sugar level would be less than $y$ had they taken $x_0$ insulin when the patient took $x_1$ insulin with sugar level greater than or equal to $y$.
PS represents the probability that the patient's  blood sugar level would be greater than or equal to $y$ had they taken $x_1$ insulin when the patient took $x_0$ insulin with sugar level less than $y$.
\end{example}

\subsection{Identification Assumptions}

\citet{Tian2000} used montonicity Assumption~\ref{BMONO} for binary treatment and outcome to identify binary PoC. Here we generalize this assumption to continuous and discrete cases, and discuss connections with several commonly used assumptions in the literature.

%Next, we explain four nonparametric assumptions for identifying three types of PoC.

%($\mathrm{\bf I}$). {\bf Monotonicity w.r.t.  $X$.}
($\mathrm{\bf I}$). {\bf Monotonicity over  $Y_x$.}
%First, monotonicity on $X$ for a binary treatment, i.e., %$\mathbb{P}(Y_{x_0}=0, Y_{x_1}=1)=0$  or 
%$\mathbb{P}(Y_{x_0}=1, Y_{x_1}=0)=0$,
%$\mathbb{P}(Y_{x_0}\leq Y_{x_1})=0$  or $\mathbb{P}(Y_{x_0}\geq Y_{x_1})=0$,
% have  appeared in the studies \citep{Balke1997,Tian2000,Imbens1994,Angrist1996}.
We first propose the following assumption:
\begin{assumption}[Strong Monotonicity over $Y_x$]
\label{mono-S}
    The potential outcomes $Y_x$ satisfy,  for any $x_0,x_1 \in \Omega_X$,  either "$Y_{x_0}(u)\leq Y_{x_1}(u)$ $\mathbb{P}_U$-almost surely for every $u \in \Omega_U$" or "$Y_{x_1}(u)\leq Y_{x_0}(u)$ $\mathbb{P}_U$-almost surely for  every $u \in \Omega_U$".
\end{assumption}
Note that we allow both monotonic increasing and decreasing cases. It turns out that the PoC in Def.~\ref{def2} can be identified under a weaker assumption:\footnote{This assumption is not ''$\mathbb{P}(Y_{x_0}< y \leq Y_{x_1})=0$ for any $x_0,x_1 \in \Omega_X$ and $y \in \Omega_Y$'' or ''$\mathbb{P}(Y_{x_1}< y \leq Y_{x_0})=0$ for any $x_0,x_1 \in \Omega_X$ and $y \in \Omega_Y$.''
}
\begin{assumption}[Monotonicity over $Y_x$]
\label{MONO_A}
    The potential outcomes $Y_x$ satisfy,  for any $x_0,x_1 \in \Omega_X$ and $y \in \Omega_Y$,  ``either $\mathbb{P}(Y_{x_0}< y \leq Y_{x_1})=0$ or $\mathbb{P}(Y_{x_1}< y \leq Y_{x_0})=0$''.
%    \begin{equation}
%     \text{either } \  \mathbb{P}(Y_{x_0}< y \leq Y_{x_1})=0\text{ or }\mathbb{P}(Y_{x_1}< y \leq Y_{x_0})=0
%    \end{equation}   
\end{assumption}
Introducing a binarized outcome $O=\mathbb{I}(Y \geq y)$, Assumption~\ref{MONO_A} becomes ``$\mathbb{P}(O_{x_0}=0,O_{x_1}=1)=0$ or $\mathbb{P}(O_{x_0}=1,O_{x_1}=0)=0$''. Assumption~\ref{MONO_A} is weaker than \ref{mono-S} since $\mathbb{P}(Y_{x_0}< Y_{x_1})=0$  implies $\mathbb{P}(Y_{x_0}< y \leq Y_{x_1})=0$ but not vice versa.

Next, we discuss several related assumptions used in the literature for various identification purposes.

($\mathrm{\bf I}\hspace{-1.2pt}\mathrm{\bf I}$). 
%{\bf Monotonicity w.r.t. $U$.}
{\bf Monotonicity over $f_Y$.}
Monotonicity on $U$ over structural function $f_Y(x,U)$ has appeared in the instrumental variable (IV) literature, e.g. 
\citep{Vytlacil2002,Heckman1999,Heckman2005}.
%\yuta{We denote the structural function from $X$ to $Y$ by $f_Y(X,U)$.}
%They defined monotonicity on $U$ as below.
\begin{assumption}[Monotonicity over $f_Y$]
\label{AS1}
The function $f_Y(x,U)$ is either monotonic increasing on $U$ for all $x \in \Omega_X$ or monotonic decreasing on $U$ for all $x \in \Omega_X$ almost surely w.r.t. $\mathbb{P}_U$.
%The function $f_Y(x,U)$ is either monotonic increasing or monotonic decreasing on $U$ for all $x \in \Omega_X$ almost surely w.r.t. $\mathbb{P}_U$. 
%\jin{Do you mean: The function $f_Y(x,U)$ is either monotonic increasing on $U$ for all $x \in \Omega_X$ or monotonic decreasing on $U$ for all $x \in \Omega_X$? Otherwise, your proof argument doesn't look right. }
\end{assumption}
%\yuta{"Almost surely w.r.t. $\mathbb{P}_U$" means the probability of the violation of the above conditions is $0$ with respect to $\mathbb{P}_U$.}
%\jin{Does this mean $f_Y(x,U)$ may not be monotonic over some measure zero number of U values? } 
%\yuta{Yes, I do not requires $f_Y(x,U)$ is monotonic for all U values.}
%
%This assumption implies that two trajectories of potential outcome, i.e., $\{(x,Y_x(u_0)) \in \Omega_X \times \Omega_Y;\forall x \in \Omega_X\}$ and $\{(x,Y_x(u_1))\in \Omega_X \times \Omega_Y;\forall x \in \Omega_X\}$, do not cross over each other for $\mathbb{P}_U$-almost every $u_0, u_1 \in \Omega_U$.
%
For example, \citet{Heckman2005} introduced the latent index model  $Y:=\mathbb{I}[f_Y(X) \geq U]$ for a binary outcome,  
which satisfies the above assumption.

%\jin{I found the names of 'Monotonicity on Potential Outcomes' and  'monotonicity on SCM' misleading, because (I) is about monotonicity with respect to $X$ while (II) is about monotonicity with respect to $U$. Both kinds of monotonicity could be described in terms of potential outcomes or SCMs. In other words, the difference between these two kinds of monotonicity is NOT about potential outcomes or SCMs.   }

($\mathrm{\bf I}\hspace{-1.2pt}\mathrm{\bf I}\hspace{-1.2pt}\mathrm{\bf I}$). 
{\bf Strict monotonicity over $f_Y$.}
%Third, strict monotonicity on $U$ or structural function also have been appeared in the IV literature, such as
The following stronger monotonicity assumption have also been used 
\citep{Chesher2003,Chernozhukov2005,Chernozhukov2007,Imbens2009}.
%They defined strict monotonicity on $U$ as below.
\begin{assumption}[Strict monotonicity over $f_Y$]
\label{SAS1}
The function $f_Y(x,U)$ is either strictly monotonic increasing on $U$ for all $x \in \Omega_X$ or strictly monotonic decreasing on $U$ for all $x \in \Omega_X$ almost surely w.r.t. $\mathbb{P}_U$ with $\sup_{u \in \Omega_U}\mathfrak{p}(u)<\infty$. 
\end{assumption}
The condition $\sup_{u \in \Omega_U}\mathfrak{p}(u)<\infty$ means $U$ is continuous distribution. %\jin{Do you mean "$U$ must be continuous"? What if $U$ is discrete?} \yuta{If $U$ is the discrete variable, PDF of $U$ can be represented by delta functions, and take $\infty$ for some $u \in \Omega_U$.} \jin{If $U$ is discrete, then the above Assumption can not be defined? Why? } \yuta{If $U$ is discrete, we use Assumption \ref{AS1}.}
%This assumption is stronger than monotonicity on $U$ because strong monotonicity implies monotonicity, but not the other way around.
For example, the widely used additive noise model $Y=f_Y(X)+U$ \citep{Whitney2003,Singh2019,Hartford2017,Xu2021,Kawakami2023} satisfies this assumption.

{\bf Relationship between the three assumptions.}
%Then, we explain the relationship of the above assumptions. 
We obtain that our proposed monotonicity Assumption~\ref{MONO_A} for continuous and discrete cases is equivalent to the monotonicity Assumption~\ref{AS1} over structural function $f_Y(x,U)$ 
under the following assumption: 
\begin{assumption}
\label{SUP1}
Potential outcome $Y_x$ has PDF $p_{Y_x}$ for each $x \in \Omega_{X}$, and its support $\{y \in \Omega_Y: p_{Y_x}(y) \ne0 \}$ is
the same
%$[-\infty,\infty]$ 
for each $x \in \Omega_{X}$.
\end{assumption}
This assumption is reasonable for continuous variables. 
For example, the 
%multivariate 
linear regression model with Gaussian noise in \citep{Hannart2018} satisfies this assumption.

%, and our proposed rank preservation Assumption~\ref{RP1} for continuous and discrete cases is equivalent to the strict monotonicity Assumption~\ref{SAS1}.  
%\jin{Is the condition "Under SCM ${\cal M}_S$" required for the theorem?} 
%\yuta{[I think it is needed since we use function $f_Y$ in the definition of Asuumption \ref{AS1} and \ref{SAS1}.]}
%\yuta{[Deleted]}
%First, we show the equivalence of the first and second assumptions. 
%We have the following lemma.
%\begin{lemma}
%\label{LEM31}
  %  Under SCM ${\cal M}_{S}$ and Assumption \ref{TOT}, Assumption \ref{MONO} implies Assumption \ref{AS1}.
%\end{lemma}
%Next, we have the following lemma.
%\begin{lemma}
 %   Under SCM ${\cal M}_{S}$ and Assumption \ref{TOT}, Assumption \ref{AS1} implies Assumption \ref{MONO}.
%\end{lemma}
%We have the following theorem.
\begin{theorem}%[Equivalence of three assumptions]
\label{E12}
    Under SCM ${\cal M}_S$ and Assumption \ref{SUP1}, 
    Assumptions \ref{MONO_A} and \ref{AS1} are equivalent, and
    %and Assumptions \ref{SAS1} and \ref{RP1} are equivalent. 
    {Assumption \ref{SAS1} is a strictly stronger requirement than \ref{AS1}.}
\end{theorem}

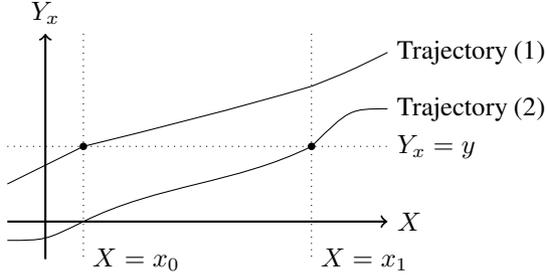
\begin{figure}
    \centering
    \scalebox{1}{
    \begin{tikzpicture}

  \draw (0,0) .. controls (1,0.25) and (2,0.5) .. (3,0.8);
 \draw (-1,-0.5) .. controls (-0.5,-0.25) .. (0,0);
  \draw (3,0.8) .. controls (3.5,1) .. (4,1.25)
    node[anchor=west] 
    {Trajectory (1)};
    %{$Y_x(u_{\rho(y;x_0)})$};

  \draw (0,-1) .. controls (1,-0.5) and (2,-0.5) .. (3,0);
 \draw (-1,-1.25) .. controls (-0.5,-1.25) .. (0,-1);
  \draw (3,0) .. controls (3.5,0.5) .. (4,0.5)
  node[anchor=west] 
    {Trajectory (2)};
    %{$Y_x(u_{\rho(y;x_1)})$};

    \draw[black,dotted] (-1, 0) -- (4, 0)
      node[anchor=west] {$Y_x=y$};

        \draw[black,dotted] (0, 1.5) -- (0, -1.5)
      node[anchor=west] {$X=x_0$};
      
   \draw[black,dotted] (3, 1.5) -- (3, -1.5)
      node[anchor=west] {$X=x_1$};

    % y-axis
    \draw[thick, black, ->] (-0.5, -1.5) -- (-0.5, 1.5)
      node[anchor=south] {$Y_x$};

      % x-axis
    \draw[thick, black, ->] (-1, -1) -- (4, -1)
      node[anchor=west] {$X$};

\node at (0,0)[circle,fill,inner sep=1pt]{};
\node at (3,0)[circle,fill,inner sep=1pt]{};
      
\end{tikzpicture}
}
    \caption{Trajectories for (1) $Y_x(u_{\rho(y;x_0)})$ and (2) $Y_x(u_{\rho(y;x_1)})$. }
    %$Y_x(u_{\rho(y;x_0)})$ and $Y_x(u_{\rho(y;x_1)})$.}
    \label{fig:1}
\end{figure}

\subsection{Identification Theorem}
Next, we present an identification theorem.
We denote the conditional CDF 
\begin{equation}
    \rho(y;x)\defeq\mathbb{P}(Y< y|X=x).%=\mathbb{P}(Y\leq y|X=x)
\end{equation}
%for any $y \in \Omega_Y$ and $x \in \Omega_X$.
%, and we have $\rho(y;x)=\mathbb{P}(Y < y|X=x)$ for any $y \in \Omega_Y$  and $x \in \Omega_X$. 
%\jin{This only holds assuming Exogeneity or under Under SCM ${\cal M}_S$.  Do you want $\rho(y;x)$ to denote $\mathbb{P}(Y_x\leq y)$ or $\mathbb{P}(Y\leq y|X=x)$? can't be both! This is confusing. Given its application in the following theorem, why not just define $\rho(y;x)=\mathbb{P}(Y\leq y|X=x)$?}\yuta{[Fixed]}
%We assume
%\begin{assumption}[Positivity]
%\label{POS}
%    $\rho(y;x_1)<1$ and $0<\rho(y;x_0)$ hold for any $x_0, x_1 \in \Omega_X$ and $y \in \Omega_Y$.
%\end{assumption}
%We do not assume the positivity of the boundary $\Omega_Y \setminus \Omega_Y$, e.g., the minimum or maximum values of $\Omega_Y$.
%Then, we have the following theorem. %\jin{Is the condition "Under SCM ${\cal M}_S$" really required for the theorem? Can it be replaced by Lemma 3.1 - that is, just assuming Exogeneity?}\yuta{[Fixed]}
\begin{theorem}[Identification of PoC]
\label{THEO1}
Under SCM ${\cal M}_{S}$ and Assumptions \ref{ASEXO}, \ref{MONO_A} (or \ref{AS1}, \ref{SAS1}), and \ref{SUP1}, 
PNS, PN, and PS are identifiable by
\begin{equation}
    \begin{aligned}
    &\text{PNS}(y;x_0,x_1)=\max\{\rho(y;x_0)-\rho(y;x_1),0\},\\
    &\text{PN}(y;x_0,x_1)=\max\left\{\frac{\rho(y;x_0)-\rho(y;x_1)}{1-\rho(y;x_1)},0\right\},\\
    &\text{PS}(y;x_0,x_1)=\max\left\{\frac{\rho(y;x_0)-\rho(y;x_1)}{\rho(y;x_0)},0\right\}
    \end{aligned}
\end{equation}
for any  $x_0,x_1 \in \Omega_X$ and $y \in \Omega_Y$ such that $\rho(y;x_1)<1$ and $\rho(y;x_0)>0$.
\end{theorem}
%This theorem consists of conditional CDF.
%If $Y$ is a binary outcome, Theorem \ref{THEO1} coincide with theorem 4 in \citep{Tian2000}. % or Proposition 4.2. \citep{Galhotra2021}. 
%Note that our theorem does not require the following positivity assumption, $\mathbb{P}(X=x_0)>0$ and $\mathbb{P}(X=x_1)>0$, since we do not use linear programming formulation as \citep{Tian2000,Galhotra2021}.
%
{We can use the trajectories of potential outcomes to visualize and explain the above identification result for PNS. 
The trajectory $\{(x,Y_x(u)) \in \Omega_X \times \Omega_Y;\forall x \in \Omega_X\}$ represents potential outcome $Y_x(u)$ vs. $X$ for the subject $U=u$.
\emph{Under Assumptions \ref{MONO_A} (or \ref{AS1}, \ref{SAS1}), the subjects' trajectories do not cross over each other} (they may overlap). 
We denote $u_{\rho(y;x)}=\sup\{u: f_Y(x,u) < y\}$ for any $x \in \Omega_X$ and $y \in \Omega_Y$, and $Y_x(u_{\rho(y;x)})$ is the potential outcome for subject $u_{\rho(y;x)}$. 
Consider the two trajectories shown in Figure \ref{fig:1}. Trajectory (1) $\{(x,Y_x(u_{\rho(y;x_0)})) \in \Omega_X \times \Omega_Y;\forall x \in \Omega_X\}$ goes through the point $(x_0,y)$, and Trajectory (2) $\{(x,Y_x(u_{\rho(y;x_1)})) \in \Omega_X \times \Omega_Y;\forall x \in \Omega_X\}$ goes through the point $(x_1,y)$. 
The trajectory of subject $u$ lies in the region between Trajectories (1) and (2) if and only if $Y_{x_0}(u) < y \leq Y_{x_1}(u)$.
Thus, we have $\text{PNS}(y;x_0,x_1)=\mathbb{P}(Y_{x_0}< y \leq Y_{x_1})=\mathbb{P}(Y_{x_0}< y)-\mathbb{P}(Y_{x_1}< y)$, where $\mathbb{P}(Y_{x_0}< y)$ represents the probability of a subject's trajectory being below Trajectory (1) and $\mathbb{P}(Y_{x_1}< y)$ represents the probability of a subject's trajectory being below Trajectory (2).
}

\section{PoC for Vector Continuous Variables}
\label{sec4}
In this section, we extend PoC to vectors of continuous or discrete variables $\boldsymbol{Y}$ and $\boldsymbol{X}$,  % under a totally ordered vector structural causal model.
and we consider PoC for a sub-population with specific covariates information. 
%We additionally consider the subject's covariates in this section, and there are two merits 
The benefits of considering the subject's covariates include (i) they reveal the heterogeneity of causal effects; and (ii) they weaken identification assumptions.

%\yuta{There exists many datasets with vector treatment and outcome as the example of  \citep{Hannart2018}.}

\begin{comment}
\begin{figure}[tb]
%\vspace{-0.5cm}
   % \hspace{0.3cm}
    \centering
    \scalebox{1}{
\begin{tikzpicture}
    % x node set with absolute coordinates
    \node[mynode] (x) at (0,0) {$\bf{X}$};
    \node[mynode] (y) at (3,0) {$\bf{Y}$};
    \node[mynode] (u) at (1.5,1) {$\bf{C}$};

    % Directed edge
    \path (x) edge[->] (y);
    \path (x) edge[dotted,<->,bend right] (y);
%    \path (z) edge[->] (x);
    \path (u) edge[->] (y);
    \path (u) edge[dotted,<->,bend left] (y);
    \path (u) edge[->]  (x);
\path (x) edge[dotted,<->,bend left] (u);
\end{tikzpicture}
}
\vspace{-0cm}
    \caption{A causal graph representing ${\cal M}_{T}$.}% Causal graph and two types of non-separability in the IV setting, ${\cal M}_{Z}^{IV}$.}
    \label{DAG1}
    \end{figure}
    \end{comment}

\subsection{Problem Setup}

\begin{comment}
{\bf Total Order.} First, we introduce total order as below.
\begin{definition}[Total Order]
   We call the partial order $\leq_{\text{p}}$ which satisfies the following property total order
    %\begin{center}
%    (Total.) 
    \begin{equation}
        \mathbb{P}({\boldsymbol A}\leq_{\text{p}} {\boldsymbol a}\vee{\boldsymbol a}\leq_{\text{p}} {\boldsymbol A})=1,
    \end{equation}
    for a random variable ${\boldsymbol A}$ and any value ${\boldsymbol a} \in \Omega$,
   % \end{center}
     and denote it simply ``$\preceq$'',and denote totally ordered set $(\Omega,\preceq)$.
\end{definition}
totally ordered set means $\Omega$ is totally ordered set almost surely w.r.t. ${\boldsymbol A}$.
Total orders are also total orders, and the converse does not hold.
\begin{lemma}
    $\mathbb{P}({\boldsymbol A}\preceq {\boldsymbol a})=1-\mathbb{P}({\boldsymbol a}\prec {\boldsymbol A})$
\end{lemma}
\begin{proof}
    Since $1=\mathbb{P}({\boldsymbol A}\preceq{\boldsymbol a}\vee{\boldsymbol a}\preceq{\boldsymbol A})=\mathbb{P}({\boldsymbol A}\preceq{\boldsymbol a})+\mathbb{P}({\boldsymbol a}\preceq{\boldsymbol A})-\mathbb{P}({\boldsymbol a}= {\boldsymbol A})=\mathbb{P}({\boldsymbol A}\preceq{\boldsymbol a})+\mathbb{P}({\boldsymbol a}\prec {\boldsymbol A})$, we have $1-\mathbb{P}({\boldsymbol a}\prec {\boldsymbol A})=\mathbb{P}({\boldsymbol A}\preceq{\boldsymbol a})$.
\end{proof}
\end{comment}

\subsection{Conditional PoC Definition} 
%{\bf Totally ordered vector structural causal model.}  
Let ${\boldsymbol X}$, ${\boldsymbol Y}$, and ${\boldsymbol C}$ be a set of continuous or discrete  treatment variables, %${\boldsymbol Y}$ be a set of %totally ordered
%continuous or discrete  
outcome variables, 
%with $\preceq$
 and %${\boldsymbol C}$ be a set of continuous or discrete  %subject's 
covariates, respectively.  
We assume the following SCM ${\cal M}_{T}$:
%represented by the causal graph in Fig \ref{DAG1}: 
%\jin{I don't think some of the bidirected edges are allowed, maybe we don't draw the causal graph? Or draw a graph satisfying Assumption 5.1 after Assumption 5.1.}\yuta{[Deleted]}
\begin{eqnarray}
\begin{aligned}
    {\boldsymbol Y}:= f_{\boldsymbol Y}({\boldsymbol X},{\boldsymbol C},{\boldsymbol U}), {\boldsymbol X}:= f_{\boldsymbol X}({\boldsymbol C},{\boldsymbol \epsilon}_{\boldsymbol X}), {\boldsymbol C}:= f_{\boldsymbol C}({\boldsymbol \epsilon}_{\boldsymbol C})
    %& {\text{where $\boldsymbol{U}$, $\boldsymbol{\epsilon}_{\boldsymbol{X}}$, and $\boldsymbol{\epsilon}_{\boldsymbol{X}}$ are mutually independent}.}
    \end{aligned}
\end{eqnarray}
The functions $f_{\boldsymbol Y}$, $f_{\boldsymbol X}$, and $f_{\boldsymbol C}$ are vector-valued functions. 
${\boldsymbol \epsilon}_{\boldsymbol X}$, ${\boldsymbol \epsilon}_{\boldsymbol C}$, and ${\boldsymbol U}$ are latent exogenous variables.
%, and ${\boldsymbol U}$ are totally ordered with $\preceq$. 
We assume that the domains $\Omega_{\boldsymbol Y}$ and $\Omega_{\boldsymbol U}$ are totally ordered sets with $\preceq$. %, or any two elements of $\Omega_{\boldsymbol Y}$ and $\Omega_{\boldsymbol U}$ are comparable with $\preceq$, respectively.}
Let the dimensions of ${\boldsymbol X}, {\boldsymbol Y}, {\boldsymbol C}, {\boldsymbol U}$ be $d_X$, $d_Y$, $d_C$, $d_U$.
%We also assume:
%\begin{assumption}[Conditional exogeneity]
%\label{EXO2}
%    ${\boldsymbol X}$ are independent of ${\boldsymbol U}$ given ${\boldsymbol C}={\boldsymbol c}$ for all ${\boldsymbol c} \in \Omega_{\boldsymbol C}$.
%\end{assumption}

We make the following assumption.
\begin{assumption}[Conditional exogeneity]
\label{ASEXO2}
   %Under SCM ${\cal M}_{T}$, 
   %${\boldsymbol X}$ are independent of ${\boldsymbol U}$ given ${\boldsymbol C}$.
${\boldsymbol Y}_{\boldsymbol x}\indep {\boldsymbol X} | {\boldsymbol C}$ for all $\boldsymbol{x} \in \Omega_{\boldsymbol X}$.   
%   ${\boldsymbol Y}_{\boldsymbol X}$ are independent of ${\boldsymbol X}$ given ${\boldsymbol C}$.
\end{assumption}
Conditional exogeneity implies $\mathbb{P}({\boldsymbol Y}_{\boldsymbol x} \prec {\boldsymbol y}|{\boldsymbol C}={\boldsymbol c})=\mathbb{P}({\boldsymbol Y} \prec {\boldsymbol y}|{\boldsymbol X}={\boldsymbol x},{\boldsymbol C}={\boldsymbol c})$ for any ${\boldsymbol c} \in \Omega_{\boldsymbol C}$.

%All proofs in this section can be given by substituting $\leq$ in the previous section to $\preceq$.
We define the multivariate  conditional PoC %for the totally ordered vector SCM 
as below:
\begin{definition}[Conditional PoC]
\label{def41}
For any ${\boldsymbol x}_0,{\boldsymbol x}_1 \in \Omega_{\boldsymbol X}$, ${\boldsymbol y} \in \Omega_{\boldsymbol Y}$, and ${\boldsymbol c} \in \Omega_{\boldsymbol C}$, we define conditional PoC by 
\begin{equation}
\begin{aligned}
    &\text{PNS}({\boldsymbol y};{\boldsymbol x}_0,{\boldsymbol x}_1,{\boldsymbol c})\defeq\mathbb{P}({\boldsymbol Y}_{{\boldsymbol x}_0} \prec {\boldsymbol y} \preceq {\boldsymbol Y}_{{\boldsymbol x}_1}|{\boldsymbol C}={\boldsymbol c}),\\
    &\text{PN}({\boldsymbol y};{\boldsymbol x}_0,{\boldsymbol x}_1,{\boldsymbol c})\defeq\mathbb{P}({\boldsymbol Y}_{{\boldsymbol x}_0} \prec {\boldsymbol y} |{\boldsymbol y} \preceq {\boldsymbol Y},{\boldsymbol X}={\boldsymbol x}_1,{\boldsymbol C}={\boldsymbol c}),\\
    &\text{PS}({\boldsymbol y};{\boldsymbol x}_0,{\boldsymbol x}_1,{\boldsymbol c})\defeq\mathbb{P}({\boldsymbol y} \preceq {\boldsymbol Y}_{{\boldsymbol x}_1} |{\boldsymbol Y} \prec {\boldsymbol y},{\boldsymbol X}={\boldsymbol x}_0,{\boldsymbol C}={\boldsymbol c}).
\end{aligned}
\end{equation}
\end{definition}
$\text{PNS}({\boldsymbol y};{\boldsymbol x}_0,{\boldsymbol x}_1,{\boldsymbol c})$ provides a measure of the sufficiency and necessity of ${\boldsymbol x}_1$ w.r.t. ${\boldsymbol x}_0$ to produce ${\boldsymbol Y}\succeq {\boldsymbol y}$ given ${\boldsymbol C}={\boldsymbol c}$.
$\text{PN}({\boldsymbol y};{\boldsymbol x}_0,{\boldsymbol x}_1,{\boldsymbol c})$ provides a measure of the necessity of ${\boldsymbol x}_1$ w.r.t. ${\boldsymbol x}_0$ to produce ${\boldsymbol Y}\succeq {\boldsymbol y}$ given ${\boldsymbol C}={\boldsymbol c}$.
$\text{PS}({\boldsymbol y};{\boldsymbol x}_0,{\boldsymbol x}_1,{\boldsymbol c})$ provides a measure of the sufficiency of ${\boldsymbol x}_1$ w.r.t. ${\boldsymbol x}_0$ to produce ${\boldsymbol Y}\succeq {\boldsymbol y}$ given ${\boldsymbol C}={\boldsymbol c}$.

%We note that the above definition of PoC depends on the choice of total order. 
%\jin{Again, any prior work on defining multivariate or conditional PoC?}
%\yuta{
\citet{Hannart2018} studied multivariate PNS where the outcomes are the space-time vectorial random variables of the Earth's surface temperatures. 
\citet{Li2019,Li2022,Li2022b} considered conditional PNS over discrete variables in their benefit function and called it z-specific PNS, but their definition of PNS is different from ours and is not suitable for continuous variables.   %in their objective function, called benefit function, and \citet{Li2019} call it z-specific PNS. They give bound and identification theorem of conditional PNS for binary or discrete treatment and outcome.

\subsection{Identification Assumptions and Theorem}

We generalize Assumptions \ref{MONO_A}, \ref{AS1}, and \ref{SAS1}
%and \ref{RP1}  
to multivariate outcomes and treatments with covariates as below, respectively.
%\yuta{We denote the structural function from ${\boldsymbol X}$ to ${\boldsymbol Y}$ by $f_{\boldsymbol Y}({\boldsymbol X},{\boldsymbol U})$.}
\begin{assumption}[Conditional monotonicity over ${\boldsymbol Y}_{{\boldsymbol x}}$]
\label{MONO2}
    The potential outcomes ${\boldsymbol Y}_{{\boldsymbol x}}$ satisfy:  for any ${\boldsymbol x}_0,{\boldsymbol x}_1 \in \Omega_{\boldsymbol X}$, ${\boldsymbol y} \in \Omega_{\boldsymbol Y}$, and ${\boldsymbol c} \in \Omega_{\boldsymbol C}$, either $\mathbb{P}({\boldsymbol Y}_{{\boldsymbol x}_0}\prec {\boldsymbol y} \preceq {\boldsymbol Y}_{{\boldsymbol x}_1}|{\boldsymbol C}={\boldsymbol c})=0$ or $\mathbb{P}({\boldsymbol Y}_{{\boldsymbol x}_1}\prec {\boldsymbol y} \preceq {\boldsymbol Y}_{{\boldsymbol x}_0}|{\boldsymbol C}={\boldsymbol c})=0$.
\end{assumption}
This assumption extends Assumptions \ref{MONO_A} to totally ordered vector variables.
%\yuta{This assumption represents the monotonicity of vector variables with the total order $\preceq$ on $\Omega_{\boldsymbol Y}$ given ${\boldsymbol C}={\boldsymbol c}$.}

\begin{assumption}[Conditional monotonicity over  $f_{\boldsymbol Y}$]
\label{AS2}
{The function $f_{\boldsymbol Y}({\boldsymbol x},{\boldsymbol c},{\boldsymbol U})$ is either (i) monotonic increasing on ${\boldsymbol U}$ with $\preceq$ for all ${\boldsymbol x} \in \Omega_{\boldsymbol X}$ and ${\boldsymbol c} \in \Omega_{\boldsymbol C}$ almost surely w.r.t. $\mathbb{P}_{\boldsymbol U}$, or (ii) monotonic decreasing on ${\boldsymbol U}$ with $\preceq$ for all ${\boldsymbol x} \in \Omega_{\boldsymbol X}$ and ${\boldsymbol c} \in \Omega_{\boldsymbol C}$ 
almost surely w.r.t. $\mathbb{P}_{\boldsymbol U}$.} 
\end{assumption}
{This assumption says that the function $f_{\boldsymbol Y}$ preserves the total order from $\Omega_{\boldsymbol U}$ to $\Omega_{\boldsymbol Y}$ given ${\boldsymbol X}={\boldsymbol x}, {\boldsymbol C}={\boldsymbol c}$.}

\begin{assumption}[Strict conditional  monotonicity over $f_{\boldsymbol Y}$]
\label{SAS2}
{The function $f_{\boldsymbol Y}({\boldsymbol x},{\boldsymbol c},{\boldsymbol U})$ is either (i) strictly monotonic increasing on ${\boldsymbol U}$ with $\preceq$ for all ${\boldsymbol x} \in \Omega_{\boldsymbol X}$ and ${\boldsymbol c} \in \Omega_{\boldsymbol C}$
almost surely w.r.t. $\mathbb{P}_{\boldsymbol U}$ with $\sup_{{\boldsymbol u} \in \Omega_{\boldsymbol U}}\mathfrak{p}({\boldsymbol u}|{\boldsymbol C}={\boldsymbol c})<\infty$ for all ${\boldsymbol c} \in \Omega_{\boldsymbol C}$, or (ii) strictly monotonic decreasing on ${\boldsymbol U}$ with $\preceq$ for all ${\boldsymbol x} \in \Omega_{\boldsymbol X}$ and ${\boldsymbol c} \in \Omega_{\boldsymbol C}$
almost surely w.r.t. $\mathbb{P}_{\boldsymbol U}$ with $\sup_{{\boldsymbol u} \in \Omega_{\boldsymbol U}}\mathfrak{p}({\boldsymbol u}|{\boldsymbol C}={\boldsymbol c})<\infty$ for all ${\boldsymbol c} \in \Omega_{\boldsymbol C}$.} 
\end{assumption}
This assumption implies that there exists a one-to-one mapping from $\Omega_{\boldsymbol U}$ to $\Omega_{\boldsymbol Y}$ given ${\boldsymbol X}={\boldsymbol x}, {\boldsymbol C}={\boldsymbol c}$.

Assumptions \ref{MONO2}, \ref{AS2}, and \ref{SAS2} reduce to Assumptions \ref{MONO_A}, \ref{AS1}, and \ref{SAS1} under SCM ${\cal M}_{S}$, respectively.
We establish the relationships between Assumptions \ref{MONO2}, \ref{AS2}, and \ref{SAS2} under the following assumption: 
\begin{assumption}
\label{SUP2}
Potential outcome ${\boldsymbol Y}_{\boldsymbol x}$ has conditional PDF $p_{{\boldsymbol Y}_{\boldsymbol x}|{\boldsymbol C}={\boldsymbol c}}$ given ${\boldsymbol C}={\boldsymbol c}$ for each ${\boldsymbol x} \in \Omega_{\boldsymbol X}$ and ${\boldsymbol c} \in \Omega_{\boldsymbol C}$, and its support $\{{\boldsymbol y} \in \Omega_{\boldsymbol Y}: p_{{\boldsymbol Y}_{\boldsymbol x}|{\boldsymbol C}={\boldsymbol c}}({\boldsymbol y}) \ne0 \}$ is the same
%$[-\infty,\infty]^{d_Y}$ 
for each ${\boldsymbol x} \in \Omega_{\boldsymbol X}$ and ${\boldsymbol c} \in \Omega_{\boldsymbol C}$.
\end{assumption}
This assumption is similar to Assumption \ref{SUP1} and reasonable for continuous variables.
For example, the multivariate linear regression model with Gaussian noise in \citep{Hannart2018} satisfies this assumption.
%We have the following results: 
\begin{theorem}
\label{prop1}
Under SCM ${\cal M}_{T}$ and Assumption \ref{SUP2}, 
Assumptions \ref{MONO2} and \ref{AS2} are equivalent, and 
%Assumptions \ref{SAS2} and \ref{RP2} are equivalent. 
{Assumption \ref{SAS2} is a strictly stronger requirement than \ref{AS2}.}
\end{theorem}
%This is the same relationship of Assumption \ref{MONO_A}, \ref{AS1}, \ref{SAS1}, and \ref{RP1}.
For example, %if $d_U=d_Y$ with the same total order, then 
the additive noise model ${\boldsymbol Y}:=f_{\boldsymbol Y}({\boldsymbol X},{\boldsymbol C})+{\boldsymbol U}$ satisfies all Assumptions \ref{MONO2}, \ref{AS2}, and \ref{SAS2}. %and \ref{RP2}.

We denote conditional CDF
\begin{equation}
\begin{aligned}
\rho({\boldsymbol y};{\boldsymbol x},{\boldsymbol c})\defeq\mathbb{P}({\boldsymbol Y}\prec {\boldsymbol y}|{\boldsymbol X}={\boldsymbol x},{\boldsymbol C}={\boldsymbol c})%\\
%&=\mathbb{P}({\boldsymbol Y}\preceq {\boldsymbol y}|{\boldsymbol X}={\boldsymbol x},{\boldsymbol C}={\boldsymbol c})
\end{aligned}
\end{equation}
for all ${\boldsymbol y} \in \Omega_{\boldsymbol Y}$, ${\boldsymbol x} \in \Omega_{\boldsymbol X}$, and ${\boldsymbol c} \in \Omega_{\boldsymbol C}$.
%We have $\rho({\boldsymbol y};{\boldsymbol x},{\boldsymbol c})=\mathbb{P}({\boldsymbol Y}\prec {\boldsymbol y}|{\boldsymbol X}={\boldsymbol x},{\boldsymbol C}={\boldsymbol c})$ for all ${\boldsymbol y} \in \Omega_{\boldsymbol Y}$, ${\boldsymbol x} \in \Omega_{\boldsymbol X}$, and ${\boldsymbol c} \in \Omega_{\boldsymbol C}$. 
%We assume
%\begin{assumption}[Positivity]
%\label{POS2}
%    $\rho({\boldsymbol y};{\boldsymbol x}_1,{\boldsymbol c})<1$ and $0<\rho({\boldsymbol y};{\boldsymbol x}_0,{\boldsymbol c})$ for all ${\boldsymbol x}_0, {\boldsymbol x}_1 \in \Omega_{\boldsymbol X}$, ${\boldsymbol c} \in \Omega_{\boldsymbol C}$ and ${\boldsymbol y} \in \Omega_{\boldsymbol Y}$ such that.
%\end{assumption}
%An order topology is a certain topology that can be defined on any totally ordered set.
Then. we have the following theorem: 
\begin{theorem}[Identification of conditional PoC]
\label{THEO41}
{Under SCM ${\cal M}_{T}$ and}  
Assumptions \ref{ASEXO2}, \ref{MONO2} (or \ref{AS2}, \ref{SAS2}), and \ref{SUP2}, 
PNS, PN, and PS are  identifiable by
\begin{equation}
    \begin{aligned}
    &\text{PNS}({\boldsymbol y};{\boldsymbol x}_0,{\boldsymbol x}_1,{\boldsymbol c})=\max\{\rho({\boldsymbol y};{\boldsymbol x}_0,{\boldsymbol c})-\rho({\boldsymbol y};{\boldsymbol x}_1,{\boldsymbol c}),0\},\\
    &\text{PN}({\boldsymbol y};{\boldsymbol x}_0,{\boldsymbol x}_1,{\boldsymbol c})=\max\left\{\frac{\rho({\boldsymbol y};{\boldsymbol x}_0,{\boldsymbol c})-\rho({\boldsymbol y};{\boldsymbol x}_1,{\boldsymbol c})}{1-\rho({\boldsymbol y};{\boldsymbol x}_1,{\boldsymbol c})},0\right\},\\
    &\text{PS}({\boldsymbol y};{\boldsymbol x}_0,{\boldsymbol x}_1,{\boldsymbol c})=\max\left\{\frac{\rho({\boldsymbol y};{\boldsymbol x}_0,{\boldsymbol c})-\rho({\boldsymbol y};{\boldsymbol x}_1,{\boldsymbol c})}{\rho({\boldsymbol y};{\boldsymbol x}_0,{\boldsymbol c})},0\right\}
    \end{aligned}
\end{equation}
for any ${\boldsymbol x}_0,{\boldsymbol x}_1 \in \Omega_{\boldsymbol X}$, ${\boldsymbol c} \in \Omega_{\boldsymbol C}$, and ${\boldsymbol y} \in \Omega_{\boldsymbol Y}$ such that $\rho({\boldsymbol y};{\boldsymbol x}_1,{\boldsymbol c})<1$ and $\rho({\boldsymbol y};{\boldsymbol x}_0,{\boldsymbol c})>0$.
\end{theorem}

%This theorem also consists of conditional CDF.

{\bf Remark.}
PoC, like $\text{PNS}({\boldsymbol y};{\boldsymbol x}_0,{\boldsymbol x}_1)$, can be computed through conditional PoC:
\begin{equation}
    \text{PNS}({\boldsymbol y};{\boldsymbol x}_0,{\boldsymbol x}_1)=\int_{{\boldsymbol c} \in \Omega_{\boldsymbol C}} \text{PNS}({\boldsymbol y};{\boldsymbol x}_0,{\boldsymbol x}_1,{\boldsymbol c}) \mathfrak{p}({\boldsymbol c})d{\boldsymbol c}
\end{equation}
where $\mathfrak{p}({\boldsymbol c})$ is PDF of ${\boldsymbol C}$.
Then, we can estimate it under weaker conditions than required by Theorem \ref{THEO1} % using observational data, e.g., the backdoor criterion \citep{Pearl09}.  \jin{Please provide justification/proof for this claim. You are essentially claiming Assumption 4.1 is weaker than Assumption 3.1.}
since the conditional version of the assumptions required by Theorem~\ref{THEO41} are weaker. 

%\yuta{Exogeneity for vector treatment and outcome is "${\boldsymbol Y}_{\boldsymbol x}$ is independent of ${\boldsymbol X}$ for all ${\boldsymbol x} \in \Omega_{\boldsymbol X}$". Exogeneity implies conditional exogeneity, Assumption \ref{ASEXO2}, and vice versa does not hold.}

\section{Variants of Probabilities of Causation}

\label{sec-variant}
%\jin{Extensions of Probabilities of Causation? Variants of Probabilities of Causation? "Complicated" doesn't sound a good title}

\label{sec5}
In this section, we study several more complicated variants of PoC.

% under a totally ordered vector structural causal model with the subject's covariates. 
%\jin{What happens to PN and PS?} \yuta{PNS with Evidence $({\boldsymbol y}',{\boldsymbol x}',{\boldsymbol c})$ is the variants of PN and PS.}

\subsection{PNS with Evidence}

% $({\boldsymbol y}',{\boldsymbol x}',{\boldsymbol c})$}

We consider PNS with evidence $({\boldsymbol Y}={\boldsymbol y}',{\boldsymbol X}={\boldsymbol x}',{\boldsymbol C}={\boldsymbol c})$ denoted by $({\boldsymbol y}',{\boldsymbol x}',{\boldsymbol c})$.\footnote{Note that PNS with evidence $({\boldsymbol y}',{\boldsymbol x}',{\boldsymbol c})$ include PN and PS with evidence as special cases.} Evidence provides the situation-specific information and restricts the attention to PNS for a sub-population.

For instance, revisiting Example~\ref{ex-1}, for a patient with a certain age and body weight, a doctor may want to know 
the probability  that the patient's  blood sugar level would be greater than or equal to the  threshold $y$ had they taken $x_1$ amount of insulin, and would be less than $y$ had they taken $x_0$ insulin, 
when the patient took $x'$ amount of insulin and had blood sugar level $y'$. This probability is given by $\mathbb{P}({ Y}_{{ x}_0}< { y} \leq { Y}_{{ x}_1}|{ Y}={ y}',{ X}={ x}',{\boldsymbol C}={\boldsymbol c})$ where ${\boldsymbol c}$ stands for the patient's  age and  body weight. % and his other diseases, were ${\boldsymbol C}={\boldsymbol c}$.

Note that for a binary treatment and outcome, PNS with evidence $(X=x_1,Y=y_1)$ coincides with PN, and PNS with evidence $(X=x_0,Y=y_0)$ coincides with PS. However, for continuous treatment and outcome, we could have  PNS with different evidence.

%Unlike the case of binary outcome and treatment, we may often have more detailed evidence $({\boldsymbol Y}={\boldsymbol y}',{\boldsymbol X}={\boldsymbol x}',{\boldsymbol C}={\boldsymbol c})$. \yuta{Evidence is the situation-specific information, and restrict the value of ${\boldsymbol U}$ \citep[p.96]{Pearl2016}.}

%\yuta{For a binary treatment and outcome, PNS with evidence $(X=x_1,Y=y_1)$ coincides with PN, and PNS with evidence $(X=x_0,Y=y_0)$ coincides with PS. However, for continuous treatment and outcome, we could have PNS with evidence $({\boldsymbol Y}={\boldsymbol y}',{\boldsymbol X}={\boldsymbol x}',{\boldsymbol C}={\boldsymbol c})$ is different from the evidence of PN or PS, $({\boldsymbol X}={\boldsymbol x}_1,{\boldsymbol y} \preceq {\boldsymbol Y},{\boldsymbol C}={\boldsymbol c})$ or $({\boldsymbol X}={\boldsymbol x}_1,{\boldsymbol Y}\prec{\boldsymbol y},{\boldsymbol C}={\boldsymbol c})$.}

%\st{ not $({\boldsymbol y} \prec {\boldsymbol Y},{\boldsymbol X}={\boldsymbol x}_1,{\boldsymbol C}={\boldsymbol c})$ or $({\boldsymbol Y}\preceq{\boldsymbol y},{\boldsymbol X}={\boldsymbol x}_0,{\boldsymbol C}={\boldsymbol c})$ of PN or PS.}
%We let $({\boldsymbol Y}={\boldsymbol y}',{\boldsymbol X}={\boldsymbol x}',{\boldsymbol C}={\boldsymbol c})$ be $({\boldsymbol y}',{\boldsymbol x}',{\boldsymbol c})$. Then, we introduce conditional PNS with evidence $({\boldsymbol y}',{\boldsymbol x}',{\boldsymbol c})$.
%${\boldsymbol E}=({\boldsymbol Y},{\boldsymbol X})$ and ${\boldsymbol e}=({\boldsymbol y}',{\boldsymbol x}')$:
\begin{definition}[Conditional PNS with evidence $({\boldsymbol y}',{\boldsymbol x}',{\boldsymbol c})$]
We define conditional PNS with evidence $({\boldsymbol y}',{\boldsymbol x}',{\boldsymbol c})$ as
\label{EV1}
\begin{equation}
\begin{aligned}
    &\text{PNS}({\boldsymbol y};{\boldsymbol x}_0,{\boldsymbol x}_1,{\boldsymbol y}',{\boldsymbol x}',{\boldsymbol c})\\
    &\hspace{0cm}\defeq \mathbb{P}({\boldsymbol Y}_{{\boldsymbol x}_0}\prec {\boldsymbol y} \preceq {\boldsymbol Y}_{{\boldsymbol x}_1}|{\boldsymbol Y}={\boldsymbol y}',{\boldsymbol X}={\boldsymbol x}',{\boldsymbol C}={\boldsymbol c})  
    \end{aligned}
\end{equation}
for any ${\boldsymbol x}_0,{\boldsymbol x}_1, {\boldsymbol x}' \in \Omega_{\boldsymbol X}$, ${\boldsymbol c} \in \Omega_{\boldsymbol C}$, and ${\boldsymbol y}, {\boldsymbol y}' \in \Omega_{\boldsymbol Y}$.
\end{definition}
$\text{PNS}({\boldsymbol y};{\boldsymbol x}_0,{\boldsymbol x}_1,{\boldsymbol y}',{\boldsymbol x}',{\boldsymbol c})$ provides a  measure of the sufficiency and necessity of ${\boldsymbol x}_1$ w.r.t. ${\boldsymbol x}_0$ to produce ${\boldsymbol Y}\succeq {\boldsymbol y}$ given the evidence $({\boldsymbol Y}={\boldsymbol y}',{\boldsymbol X}={\boldsymbol x}',{\boldsymbol C}={\boldsymbol c})$.

We denote the conditional CDF 
\begin{equation}
\begin{aligned}
    \rho^{o}({\boldsymbol y}';{\boldsymbol x}',{\boldsymbol c})\defeq\mathbb{P}({\boldsymbol Y}\preceq{\boldsymbol y}'|{\boldsymbol X}={\boldsymbol x}',{\boldsymbol C}={\boldsymbol c})
    %&=\mathbb{P}({\boldsymbol Y}\prec{\boldsymbol y}'|{\boldsymbol X}={\boldsymbol x}',{\boldsymbol C}={\boldsymbol c}),
    \end{aligned}
\end{equation}
for any ${\boldsymbol x}' \in \Omega_{\boldsymbol X}$, ${\boldsymbol y}' \in \Omega_{\boldsymbol Y}$, and ${\boldsymbol c} \in \Omega_{\boldsymbol C}$, {which differs from $\rho({\boldsymbol y}';{\boldsymbol x}',{\boldsymbol c})$ in that it includes the point ${\boldsymbol Y}={\boldsymbol y}'$.} 
%We have $\rho^{o}({\boldsymbol y}';{\boldsymbol x}',{\boldsymbol c})=\mathbb{P}({\boldsymbol Y}\preceq{\boldsymbol y}'|{\boldsymbol X}={\boldsymbol x}',{\boldsymbol C}={\boldsymbol c})$

We obtain the following theorem:
\begin{theorem}[Identification of conditional  PNS with evidence $({\boldsymbol y}',{\boldsymbol x}',{\boldsymbol c})$]
\label{THE51}
Under SCM ${\cal M}_{T}$ and Assumptions \ref{ASEXO2}, \ref{MONO2} (or \ref{AS2}, \ref{SAS2}), and \ref{SUP2}, we have %the two following statements:

{\bf (A).} If $\rho({\boldsymbol y}';{\boldsymbol x}',{\boldsymbol c})\ne\rho^{o}({\boldsymbol y}';{\boldsymbol x}',{\boldsymbol c})$, then we have
\begin{equation}
    \text{PNS}({\boldsymbol y};{\boldsymbol x}_0,{\boldsymbol x}_1,{\boldsymbol y}',{\boldsymbol x}',{\boldsymbol c})=\max\{{\alpha}/{\beta},0\},
\end{equation}
where 
\begin{equation}
\begin{aligned}
&\alpha=\min\{\rho({\boldsymbol y};{\boldsymbol x}_0,{\boldsymbol c}),\rho^{o}({\boldsymbol y}';{\boldsymbol x}',{\boldsymbol c})\}\\
&\hspace{1cm}-\max\{\rho({\boldsymbol y};{\boldsymbol x}_1,{\boldsymbol c}),\rho({\boldsymbol y}';{\boldsymbol x}',{\boldsymbol c})\},\\
%\end{aligned}
%\end{equation}
%\begin{equation}
&\beta=\rho^{o}({\boldsymbol y}';{\boldsymbol x}',{\boldsymbol c})-\rho({\boldsymbol y}';{\boldsymbol x}',{\boldsymbol c})
\end{aligned}
\end{equation}
for any ${\boldsymbol x}_0,{\boldsymbol x}_1, {\boldsymbol x}' \in \Omega_{\boldsymbol X}$, ${\boldsymbol c} \in \Omega_{\boldsymbol C}$, ${\boldsymbol y}' \in \Omega_{\boldsymbol Y}$, and ${\boldsymbol y} \in \Omega_{\boldsymbol Y}$.

{\bf (B).} If $\rho({\boldsymbol y}';{\boldsymbol x}',{\boldsymbol c})=\rho^{o}({\boldsymbol y}';{\boldsymbol x}',{\boldsymbol c})$, then we have
\begin{equation}
\begin{aligned}
&\text{PNS}({\boldsymbol y};{\boldsymbol x}_0,{\boldsymbol x}_1,{\boldsymbol y}',{\boldsymbol x}',{\boldsymbol c})\\
&=\mathbb{I}(\rho({\boldsymbol y};{\boldsymbol x}_1,{\boldsymbol c}) \leq 
%\rho^{o}({\boldsymbol y}';{\boldsymbol x}',{\boldsymbol c})=
\rho({\boldsymbol y}';{\boldsymbol x}',{\boldsymbol c})< \rho({\boldsymbol y};{\boldsymbol x}_0,{\boldsymbol c}))
\end{aligned}
\end{equation}
    for any ${\boldsymbol x}_0,{\boldsymbol x}_1, {\boldsymbol x}' \in \Omega_{\boldsymbol X}$, ${\boldsymbol c} \in \Omega_{\boldsymbol C}$, ${\boldsymbol y}' \in \Omega_{\boldsymbol Y}$, and ${\boldsymbol y} \in \Omega_{\boldsymbol Y}$.
\end{theorem}
%\jin{For case (B), PNS is either 0 or 1. Please discuss/explain why this happens.}
We provide an explanation of this result based on analyzing the trajectories of potential outcomes in Appendix \ref{app01}.

Assumption \ref{SAS2} 
%or \ref{RP2} 
implies $\rho({\boldsymbol y}';{\boldsymbol x}',{\boldsymbol c})=\rho^{o}({\boldsymbol y}';{\boldsymbol x}',{\boldsymbol c})$. % and the value of ${\boldsymbol u}$ is uniquely determined for each subject. \jin{What does "the value of ${\boldsymbol u}$ is uniquely determined for each subject" mean? Isn't a subject$=(u, \epsilon_X, \epsilon_C)$?}
Then, we have the following corollary:
\begin{corollary}
\label{COR1}
Under SCM ${\cal M}_{T}$ and Assumptions \ref{ASEXO2}, \ref{SAS2}, and \ref{SUP2},
%(or \ref{RP2}), 
%\st{if $\rho({\boldsymbol y}';{\boldsymbol x}',{\boldsymbol c})=\rho^{o}({\boldsymbol y}';{\boldsymbol x}',{\boldsymbol c})$,}
we have
\begin{equation}
\begin{aligned}
&\text{PNS}({\boldsymbol y};{\boldsymbol x}_0,{\boldsymbol x}_1,{\boldsymbol y}',{\boldsymbol x}',{\boldsymbol c})\\
&=\mathbb{I}(\rho({\boldsymbol y};{\boldsymbol x}_1,{\boldsymbol c}) \leq 
%\rho^{o}({\boldsymbol y}';{\boldsymbol x}',{\boldsymbol c})=
\rho({\boldsymbol y}';{\boldsymbol x}',{\boldsymbol c})< \rho({\boldsymbol y};{\boldsymbol x}_0,{\boldsymbol c}))
\end{aligned}
\end{equation}
    for any ${\boldsymbol x}_0,{\boldsymbol x}_1, {\boldsymbol x}' \in \Omega_{\boldsymbol X}$, ${\boldsymbol c} \in \Omega_{\boldsymbol C}$, ${\boldsymbol y}' \in \Omega_{\boldsymbol Y}$, and ${\boldsymbol y} \in \Omega_{\boldsymbol Y}$. 
\end{corollary}
\begin{comment}
\yuta{For example, multivariate and multivariable additive noise model ${\boldsymbol Y}=f_{\boldsymbol Y}({\boldsymbol X},{\boldsymbol C})+{\boldsymbol U}$.
We have $\mathbb{E}[{\boldsymbol Y}|{\boldsymbol X}={\boldsymbol x},{\boldsymbol C}={\boldsymbol c}]=f_{\boldsymbol Y}({\boldsymbol x},{\boldsymbol c})$.
Thus, the value of the ${\boldsymbol U}$ is uniquely determined by ${\boldsymbol u}={\boldsymbol y}'-f_{\boldsymbol Y}({\boldsymbol x}',{\boldsymbol c})$.}
\end{comment}

\subsection{Conditional PNS with Multi-Hypothetical Terms}

To address questions involving multiple counterfactual statements jointly, \citet{Li2022,Li2022c} considered (conditional) PNS with multi-hypothetical terms $\mathbb{P}(Y_{x_{i_1}}=y_{j_1},Y_{x_{i_2}}=y_{j_2},\ldots,Y_{x_{i_P}}=y_{j_P}|{\boldsymbol C}={\boldsymbol c})$ when $X$ and $Y$ are discrete scalar variables taking values $\{x_1,\ldots,x_P\}$ and $\{y_1,\ldots,y_Q\}$. However, their definition is not applicable to continuous outcome $Y$. Here, we define conditional PNS with multi-hypothetical terms that are applicable to both discrete and continuous cases.  

%In order to deal with multiple counterfactual statements jointly, we define conditional PoC with multi-hypothetical terms. When $X$ and $Y$ are discrete and scalar variables taking values $\{x_1,\ldots,x_P\}$ and $\{y_1,\ldots,y_Q\}$, \citet{Li2022,Li2022c} considered (conditional) PNS with multi-hypothetical terms $\mathbb{P}(Y_{x_{i_1}}=y_{j_1},Y_{x_{i_2}}=y_{j_2},\ldots,Y_{x_{i_P}}=y_{j_P}|{\boldsymbol C}={\boldsymbol c})$, where $1\leq  i_1,  \ldots, i_P \leq P$, $1\leq  j_1,  \ldots, j_P \leq Q$, $i_1 \ne \ldots \ne i_P$ and $j_1 \ne \ldots \ne j_P$. However, their definitions are not applicable to continuous outcome $Y$. 

%\citet{Li2022} defined PNS with multi-hypothetical terms by $\mathbb{P}(Y_{x_{i_1}}=y_{j_1},Y_{x_{i_2}}=y_{j_2},\ldots,Y_{x_{i_P}}=y_{j_P})$ , where $1\leq  i_1,  \ldots, i_P \leq P$, $1\leq  j_1,  \ldots, j_P \leq Q$, $i_1 \ne \ldots \ne i_P$ and $j_1 \ne \ldots \ne j_P$. \citet{Li2022c} consider consider conditional PNS with multi-hypothetical terms $\mathbb{P}(Y_{x_{i_1}}=y_{j_1},Y_{x_{i_2}}=y_{j_2},\ldots,Y_{x_{i_P}}=y_{j_P}|{\boldsymbol C}={\boldsymbol c})$, where $1\leq  i_1,  \ldots, i_P \leq P$, $1\leq  j_1,  \ldots, j_P \leq Q$, $i_1 \ne \ldots \ne i_P$ and $j_1 \ne \ldots \ne j_P$, in their objective function, i.e., Li and Pearl's benefit function. However, their definitions always takes zero if $Y$ is a continuous variable with bounded PDF and not suitable for continuous outcome.

\begin{example}  
Extending Example~\ref{ex-1},
%Returning to the example of the blood sugar level and the amount of insulin, 
the overdose of insulin may cause low blood sugar, which is also harmful to patients.
%A doctor may want to know the relationship between the blood sugar level and the amount of insulin.
%Let $X$ be the amount of insulin, $Y$ be the blood sugar level.
Then, the blood sugar level of a patient  should be between a lower bound $y_1$ and  an upper bound $y_2$. 
%We consider the problem of the optimal amount of insulin.
Let $x_0,x_1,x_2$ be three insulin amount ($x_0>x_1>x_2$). 
A doctor may conclude that the $x_1$ amount of insulin is better than $x_0$, $x_2$ if the following counterfactual situations are simultaneously true: the patient's blood sugar level 
(i) would be less than the lower bound $y_1$ had they taken $x_0$ amount of insulin, 
(ii) would be greater than or equal to the lower bound $y_1$ and less than the upper bound $y_2$ had they taken $x_1$ amount, 
and (iii) would be greater than or equal to the upper bound $y_2$ had they taken $x_2$ amount. 
The doctor wants to know the probability of the above counterfactual situations, which is given   by $\mathbb{P}(Y_{x_0}< y_1 \leq Y_{x_1} < y_2 \leq  Y_{x_2})$.
%\begin{equation}
%\begin{aligned}
%    \mathbb{P}(Y_{x_0}< y_1 \leq Y_{x_1} < y_2 \leq  Y_{x_2}).
%\end{aligned}
%\end{equation}
\end{example}

%Here, we define conditional PNS with multi-hypothetical terms that 
%\yuta{In order to deal with multiple counterfactual statements jointly, we define conditional PoC with multi-hypothetical \citep{Li2022}, which 
%are applicable to both discrete and continuous cases, as below:
%\jin{ What is the motivation for defining it?}
\begin{definition}[Conditional PNS with multi-hypothetical terms] 
\label{EV2} 
Conditional PNS with multi-hypothetical terms $\text{PNS}(\overline{{\boldsymbol y}};\overline{{\boldsymbol x}},{\boldsymbol c})$ is defined by $\mathbb{P}({\boldsymbol Y}_{{\boldsymbol x}_0} \prec {\boldsymbol y}_1 \preceq{\boldsymbol Y}_{{\boldsymbol x}_1}, {\boldsymbol Y}_{{\boldsymbol x}_1}\prec{\boldsymbol y}_2 \preceq {\boldsymbol Y}_{{\boldsymbol x}_2},\ldots, {\boldsymbol Y}_{{\boldsymbol x}_{P-1}} \prec {\boldsymbol y}_P\preceq{\boldsymbol Y}_{{\boldsymbol x}_P}|{\boldsymbol C}={\boldsymbol c})$ for any sets of values $\overline{{\boldsymbol x}}=({\boldsymbol x}_0,{\boldsymbol x}_1,\ldots,{\boldsymbol x}_P)$, $\overline{{\boldsymbol y}}=({\boldsymbol y}_1,\ldots,{\boldsymbol y}_P)$, and any ${\boldsymbol c} \in \Omega_{\boldsymbol C}$, where $\overline{{\boldsymbol y}}$ is a set of  thresholds of outcome, and $\overline{{\boldsymbol x}}$ is a set of  treatments.   
\end{definition}
%\yuta{This definition consists of $P$ couterfactual hypothetical terms ${\boldsymbol Y}_{{\boldsymbol x}_0} \preceq {\boldsymbol y}_1 \prec{\boldsymbol Y}_{{\boldsymbol x}_1}$, ${\boldsymbol Y}_{{\boldsymbol x}_1} \preceq {\boldsymbol y}_2 \prec{\boldsymbol Y}_{{\boldsymbol x}_2}$, $\ldots$, ${\boldsymbol Y}_{{\boldsymbol x}_{P-1}} \preceq {\boldsymbol y}_P \prec{\boldsymbol Y}_{{\boldsymbol x}_P}$. This is the extension of conditional PNS in Def \ref{def41}. }

For instance, when $\overline{{\boldsymbol x}}=({\boldsymbol x}_0,{\boldsymbol x}_1,{\boldsymbol x}_2)$ and $\overline{{\boldsymbol y}}=({\boldsymbol y}_1,{\boldsymbol y}_2)$, 
%given ${\boldsymbol C}={\boldsymbol c}$, 
$\text{PNS}(\overline{{\boldsymbol y}};\overline{{\boldsymbol x}},{\boldsymbol c})=\mathbb{P}({\boldsymbol Y}_{{\boldsymbol x}_0} \prec {\boldsymbol y}_1 \preceq{\boldsymbol Y}_{{\boldsymbol x}_1}\prec{\boldsymbol y}_2 \preceq {\boldsymbol Y}_{{\boldsymbol x}_2}|{\boldsymbol C}={\boldsymbol c})$ %stands for the probability that (i) ${\boldsymbol Y}\succeq {\boldsymbol y}_2$ would happen had he taken ${\boldsymbol x}_2$, (ii) ${\boldsymbol Y}\succeq {\boldsymbol y}_1$ would happen and ${\boldsymbol Y}\succeq {\boldsymbol y}_2$ would not happen had he taken ${\boldsymbol x}_1$, and (iii) ${\boldsymbol Y}\succeq {\boldsymbol y}_1$ would not happen had he taken ${\boldsymbol x}_0$. It 
measures the sufficiency and necessity of ${\boldsymbol x}_1$ w.r.t. ${\boldsymbol x}_0,{\boldsymbol x}_2$ to produce ${\boldsymbol y}_1\preceq {\boldsymbol Y}\prec {\boldsymbol y}_2$ given ${\boldsymbol C}={\boldsymbol c}$.

We have the following theorem:
\begin{theorem}[Identification of conditional PNS with multi-hypothetical terms]
\label{THE52}
Under SCM ${\cal M}_{T}$ and Assumptions \ref{ASEXO2}, \ref{MONO2} (or \ref{AS2}, \ref{SAS2}), and \ref{SUP2}, 
%PoC with multi-hypothetical terms
$\text{PNS}(\overline{{\boldsymbol y}};\overline{{\boldsymbol x}},{\boldsymbol c})$ is identifiable by
\begin{equation}
\begin{aligned}
\text{PNS}(\overline{{\boldsymbol y}};\overline{{\boldsymbol x}},{\boldsymbol c})=&\max\Big\{\min_{p=1,\ldots,P}\{\rho({\boldsymbol y}_{p};{\boldsymbol x}_{p-1},{\boldsymbol c})\}\\
&\hspace{0.9cm}-\max_{p=1,\ldots,P}\{\rho({\boldsymbol y}_{p};{\boldsymbol x}_p,{\boldsymbol c})\},0\Big\}
\end{aligned}
\end{equation}
for any $\overline{{\boldsymbol x}}=({\boldsymbol x}_0,{\boldsymbol x}_1,\ldots,{\boldsymbol x}_P) \in \Omega_{\boldsymbol X}^{P+1}$, $\overline{{\boldsymbol y}}=({\boldsymbol y}_1,\ldots,{\boldsymbol y}_P)\in \Omega_{\boldsymbol Y}^P$, and ${\boldsymbol c} \in \Omega_{\boldsymbol C}$.
\end{theorem}
We provide an explanation of this result based on analyzing the trajectories of potential outcomes in Appendix \ref{app01}.

\subsection{Conditional PNS with Multi-Hypothetical Terms and Evidence}

%$({\boldsymbol y}',{\boldsymbol x}',{\boldsymbol c})$}

%We introduce the following causal quantity combining the settings in Definitions \ref{EV1} and \ref{EV2}.
We consider PNS with multi-hypothetical terms and evidence $({\boldsymbol y}',{\boldsymbol x}',{\boldsymbol c})$, combining the settings in Definitions \ref{EV1} and \ref{EV2}.
Evidence provides the situation-specific information and restricts the attention to  a sub-population.

For instance, revisiting Example~\ref{ex-1}, for a patient with a certain age and body weight, a doctor may want to know 
the probability that
the patient's blood sugar level 
(i) would be less than the lower bound $y_1$ had they taken $x_0$ amount of insulin, 
(ii) would be greater than or equal to the lower bound $y_1$ and less than the upper bound $y_2$ had they taken $x_1$ amount, 
and (iii) would be greater than or equal to the upper bound $y_2$ had they taken $x_2$ amount,
when the patient took $x'$ amount of insulin and had blood sugar level $y'$. 
This probability is given by $\mathbb{P}(Y_{x_0}< y_1 \leq Y_{x_1} < y_2 \leq  Y_{x_2}|{ Y}={ y}',{ X}={ x}',{\boldsymbol C}={\boldsymbol c})$ 
where ${\boldsymbol c}$ stands for the patient's  age and body weight.

\begin{definition}[Conditional PNS with multi-hypothetical terms and evidence $({\boldsymbol y}',{\boldsymbol x}',{\boldsymbol c})$] 
\label{EV3}
Conditional PNS with multi-hypothetical terms and evidence $({\boldsymbol y}',{\boldsymbol x}',{\boldsymbol c})$ $\text{PNS}(\overline{{\boldsymbol y}};\overline{{\boldsymbol x}},{\boldsymbol y}',{\boldsymbol x}',{\boldsymbol c})$ is defined by $\mathbb{P}({\boldsymbol Y}_{{\boldsymbol x}_0} \prec {\boldsymbol y}_1 \preceq{\boldsymbol Y}_{{\boldsymbol x}_1}, {\boldsymbol Y}_{{\boldsymbol x}_1}\prec{\boldsymbol y}_2 \preceq {\boldsymbol Y}_{{\boldsymbol x}_2},\ldots, {\boldsymbol Y}_{{\boldsymbol x}_{P-1}} \prec {\boldsymbol y}_P\preceq{\boldsymbol Y}_{{\boldsymbol x}_P}|{\boldsymbol Y}={\boldsymbol y}',{\boldsymbol X}={\boldsymbol x}',{\boldsymbol C}={\boldsymbol c})$ for any ${\boldsymbol x}' \in \Omega_{\boldsymbol X}$, ${\boldsymbol y}' \in \Omega_{\boldsymbol Y}$, $\overline{{\boldsymbol x}}=({\boldsymbol x}_0,{\boldsymbol x}_1,\ldots,{\boldsymbol x}_P) \in \Omega_{\boldsymbol X}^{P+1}$, $\overline{{\boldsymbol y}}=({\boldsymbol y}_1,\ldots,{\boldsymbol y}_P)\in \Omega_{\boldsymbol Y}^P$, and ${\boldsymbol c} \in \Omega_{\boldsymbol C}$.
\end{definition}
%\yuta{This is the extension of PNS in both Def \ref{EV1} and \ref{EV2}. }

When $\overline{\boldsymbol x}=({\boldsymbol x}_0,{\boldsymbol x}_1,{\boldsymbol x}_2)$ and $\overline{\boldsymbol y}=({\boldsymbol y}_1,{\boldsymbol y}_2)$,
%given ${\boldsymbol C}={\boldsymbol c}$, 
$\text{PNS}(\overline{{\boldsymbol y}};\overline{{\boldsymbol x}},{\boldsymbol y}',{\boldsymbol x}',{\boldsymbol c})$ measures  the sufficiency and necessity of ${\boldsymbol x}_1$ w.r.t. ${\boldsymbol x}_0,{\boldsymbol x}_2$ to produce ${\boldsymbol y}_1\preceq {\boldsymbol Y}\prec {\boldsymbol y}_2$ given the evidence $({\boldsymbol Y}={\boldsymbol y}',{\boldsymbol X}={\boldsymbol x}',{\boldsymbol C}={\boldsymbol c})$.

We have the following theorem.
\begin{theorem}[Identification of conditional PNS with multi-hypothetical terms and evidence $({\boldsymbol y}',{\boldsymbol x}',{\boldsymbol c})$]
\label{THE53}
Under SCM ${\cal M}_{T}$ and Assumptions \ref{ASEXO2}, \ref{MONO2} (or \ref{AS2}, \ref{SAS2}), and \ref{SUP2}, we have %two following statements:

{\bf (A).} If $\rho({\boldsymbol y}';{\boldsymbol x}',{\boldsymbol c})\ne\rho^{o}({\boldsymbol y}';{\boldsymbol x}',{\boldsymbol c})$, then we have
 %($\rho^{o}({\boldsymbol y}';{\boldsymbol x}',{\boldsymbol c}) > \rho({\boldsymbol y}';{\boldsymbol x}',{\boldsymbol c})$),
    \begin{equation}
%\begin{aligned}
\label{CPNS3}
        \text{PNS}(\overline{{\boldsymbol y}};\overline{{\boldsymbol x}},{\boldsymbol y}',{\boldsymbol x}',{\boldsymbol c})=\max\left\{{\gamma}/{\delta},0\right\},
 %   \end{aligned}
\end{equation}
where 
\begin{equation}
\begin{aligned}
&\gamma=\min\Big\{\min_{p=1,\ldots,P}\{\rho({\boldsymbol y}_{p};{\boldsymbol x}_{p-1},{\boldsymbol c})\},\rho^{o}({\boldsymbol y}';{\boldsymbol x}',{\boldsymbol c}) \Big\}\\
&\hspace{0.2cm}- \max\Big\{\max_{p=1,\ldots,P}\{\rho({\boldsymbol y}_{p};{\boldsymbol x}_p,{\boldsymbol c})\},\rho({\boldsymbol y}';{\boldsymbol x}',{\boldsymbol c})\Big\},\\
&\delta=\rho^{o}({\boldsymbol y}';{\boldsymbol x}',{\boldsymbol c})-\rho({\boldsymbol y}';{\boldsymbol x}',{\boldsymbol c})
\end{aligned}
\end{equation} 
for any ${\boldsymbol x}' \in \Omega_{\boldsymbol X}$, ${\boldsymbol y}' \in \Omega_{\boldsymbol Y}$, $\overline{{\boldsymbol x}}=({\boldsymbol x}_0,{\boldsymbol x}_1,\ldots,{\boldsymbol x}_P) \in \Omega_{\boldsymbol X}^{P+1}$, $\overline{{\boldsymbol y}}=({\boldsymbol y}_1,\ldots,{\boldsymbol y}_P)\in \Omega_{\boldsymbol Y}^P$, and ${\boldsymbol c} \in \Omega_{\boldsymbol C}$.

{\bf (B).} If $\rho({\boldsymbol y}';{\boldsymbol x}',{\boldsymbol c})=\rho^{o}({\boldsymbol y}';{\boldsymbol x}',{\boldsymbol c})$, then we have
%we have
\begin{equation}
\begin{aligned}
&\text{PNS}(\overline{{\boldsymbol y}};\overline{{\boldsymbol x}},{\boldsymbol y}',{\boldsymbol x}',{\boldsymbol c})\\
&=\mathbb{I}\Big(\max_{p=1,\ldots,P}\{\rho({\boldsymbol y}_{p};{\boldsymbol x}_p,{\boldsymbol c})\} \leq
%\rho^{o}({\boldsymbol y}';{\boldsymbol x}',{\boldsymbol c})=
\rho({\boldsymbol y}';{\boldsymbol x}',{\boldsymbol c})\\
&\hspace{2.5cm}< \min_{p=1,\ldots,P}\{\rho({\boldsymbol y}_{p};{\boldsymbol x}_{p-1},{\boldsymbol c})\}\Big)
\end{aligned}
\end{equation}
 for any ${\boldsymbol x}' \in \Omega_{\boldsymbol X}$, ${\boldsymbol y}' \in \Omega_{\boldsymbol Y}$, $\overline{{\boldsymbol x}}=({\boldsymbol x}_0,{\boldsymbol x}_1,\ldots,{\boldsymbol x}_P) \in \Omega_{\boldsymbol X}^{P+1}$, $\overline{{\boldsymbol y}}=({\boldsymbol y}_1,\ldots,{\boldsymbol y}_P)\in \Omega_{\boldsymbol Y}^P$, and ${\boldsymbol c} \in \Omega_{\boldsymbol C}$.
\end{theorem}

In addition, we have the following corollary:
\begin{corollary}
\label{COR2}
Under SCM ${\cal M}_{T}$ and Assumptions \ref{ASEXO2}, \ref{SAS2}, and \ref{SUP2},
%(or \ref{RP2}), 
%\st{if $\rho({\boldsymbol y}';{\boldsymbol x}',{\boldsymbol c})=\rho^{o}({\boldsymbol y}';{\boldsymbol x}',{\boldsymbol c})$,}
we have
%we have
\begin{equation}
\begin{aligned}
&\text{PNS}(\overline{{\boldsymbol y}};\overline{{\boldsymbol x}},{\boldsymbol y}',{\boldsymbol x}',{\boldsymbol c})\\
&=\mathbb{I}\Big(\max_{p=1,\ldots,P}\{\rho({\boldsymbol y}_{p};{\boldsymbol x}_p,{\boldsymbol c})\} \leq
%\rho^{o}({\boldsymbol y}';{\boldsymbol x}',{\boldsymbol c})=
\rho({\boldsymbol y}';{\boldsymbol x}',{\boldsymbol c})\\
&\hspace{2.5cm}< \min_{p=1,\ldots,P}\{\rho({\boldsymbol y}_{p};{\boldsymbol x}_{p-1},{\boldsymbol c})\}\Big)
\end{aligned}
\end{equation}
 for any ${\boldsymbol x}' \in \Omega_{\boldsymbol X}$, ${\boldsymbol y}' \in \Omega_{\boldsymbol Y}$, $\overline{{\boldsymbol x}}=({\boldsymbol x}_0,{\boldsymbol x}_1,\ldots,{\boldsymbol x}_P) \in \Omega_{\boldsymbol X}^{P+1}$, $\overline{{\boldsymbol y}}=({\boldsymbol y}_1,\ldots,{\boldsymbol y}_P)\in \Omega_{\boldsymbol Y}^P$, and ${\boldsymbol c} \in \Omega_{\boldsymbol C}$.
\end{corollary}

\section{Application to a Real-world Dataset}

%We show an application to a real-world dataset in education.

{\bf Dataset.}
We take up an open dataset in the UC Irvine Machine Learning Repository \texttt{https://archive.ics.uci.edu/dataset/320\\/student+performance}
about student performance in mathematics in secondary education of two Portuguese schools.
Secondary education lasts three years, and students are tested once a year, three times in total.
%This data approaches student achievement in secondary education of two Portuguese schools. 
The data attributes include demographic, social, and school-related features and student grades. %and it was collected by using school reports and questionnaires.
The sample size is $649$ with no missing values. 
Prior research using this data  aimed  to predict the students' performance based on their attributes \citep{Cortez2008,Helwig2017}.
We assess the causal relationship between the students' performance, study time, and extra paid classes via estimating PoC introduced in this paper.

{\bf Variables.}
We take the scores of mathematics in the final period ($Y^1$), in the second period ($Y^2$), and in the first period ($Y^3$) as the outcome variables ${\boldsymbol Y}=(Y^1,Y^2,Y^3)$. $Y^1, Y^2, Y^3$ take values from $\{0, 1, \ldots, 20\}$.  We assume a lexicographical order $\succ_{\text{lexi}}$ on $\boldsymbol{Y}$. For example, $(Y^1,Y^2,Y^3) \succ_{\text{lexi}} (6, 6, 6)$ means ``$Y^1>6$'' or ``$Y^1=6\land Y^2>6$'' or ``$Y^1=6\land Y^2=6\land Y^3>6$''.
\begin{comment}
We pick up a tuple of the three outcome variables, ${\boldsymbol Y}=(Y^1,Y^2,Y^3)$,
\begin{itemize}
    \vspace{-0.1cm}
      \setlength{\parskip}{0.cm}
  \setlength{\itemsep}{0.15cm}
    \item[] $Y^1$: \yuta{Scores of mathematics in the final period}, \jin{What is 'final' grade?}
    \item[] $Y^2$: Scores of mathematics in the second period,
    \item[] $Y^3$: Scores of mathematics in the first period
\end{itemize}
\vspace{-0.1cm}
and they take the score from $0$ to $20$, respectively.
We introduce lexicographical order to outcomes.
Say ${\boldsymbol y}=(6,6,6)$;
the statement ``${\boldsymbol Y}\succ_{\text{lexi}}{\boldsymbol y}$'' means 
%\begin{center}
\yuta{``{\it one gets scores over 6 in the final period}'' or ``{\it one gets 6 scores in the final period and over 6 in the second period}'' or ``{\it one gets 6 scores in the final and second periods and over 6 in the first period }.''}
%    \yuta{``$Y^1>6$'' or ``$Y^1=6\land Y^2>6$'' or ``$Y^1=6\land Y^2=6\land Y^3>6$.''}
%\end{center}
\end{comment}
We consider ``\emph{study time in a week}'' ($X^1$) and ``\emph{extra paid classes within the course subject}'' ($X^2$) (yes: $X^2=2$, no: $X^2=1$) as treatment variables ${\boldsymbol X} = (X^1, X^2)$. 
%\st{Let the subject's all other attributes be the covariates ($\boldsymbol{C}$).} \jin{But in the experiments, you only selected 3 of them as covariates.} 
We select ``sex'', ``failures'', ``schoolsup'', ``famsup'', and ``goout'' as the  covariates ($\boldsymbol{C}$), which were chosen in  \citep{Helwig2017} in a previous study. % which are chosen with $p<0.05$ estimated coefficients by Table 3 in \citep{Helwig2017}, which is the previous study using this dataset.}

%\jin{We select ''school'', ``sex'', and ``age'' as the  covariates ($\boldsymbol{C}$).} \jin{Why only select these 3? Are there other attributes that have impact on both X and Y? }
%We assume Assumption \ref{RP2} which means that the counterfactual rankings of scores ${\boldsymbol Y}_{\boldsymbol x}$ among students are preserved for different ${\boldsymbol x}$ (study time and extra paid classes) given any covariates ${\boldsymbol C}={\boldsymbol c}$. 
We assume Assumption \ref{AS2} %since it is reasonable to assume 
which means that latent exogenous variables, such as the student's mental and physical conditions during the test day, have monotonic impacts on the test scores.

%\yuta{since it is reasonable to consider the counterfactual ranking of all student's scores ${\boldsymbol Y}_{\boldsymbol x}$ are preserved for any ${\boldsymbol x}$ (study time and extra paid classes) given the students covariates ${\boldsymbol C}={\boldsymbol c}$.} \jin{Justification/intuition for this assumption?}

{\bf Estimation Methods.}
%We do not discuss the estimation problems of PoC in this paper. However, each type of PoC is easily estimable through the conditional CDF, i.e., $\hat{\rho}({\boldsymbol y};{\boldsymbol x},{\boldsymbol c})$ and $\hat{\rho}^o({\boldsymbol y};{\boldsymbol x},{\boldsymbol c})$, by standard regression methods using sampled i.i.d dataset because all theorems in this paper consist of conditional CDF.
All identification theorems in the paper compute PoC through conditional CDFs, e.g. $\rho({\boldsymbol y};{\boldsymbol x},{\boldsymbol c})=\mathbb{P}({\boldsymbol Y}\prec {\boldsymbol y}|{\boldsymbol X}={\boldsymbol x},{\boldsymbol C}={\boldsymbol c})$. %$\hat{\rho}({\boldsymbol y};{\boldsymbol x},{\boldsymbol c})$. 
We estimate the conditional CDFs by logistic regression {using the ``glm'' function in R.} 
%$\mathbb{I}({\boldsymbol Y}\preceq {\boldsymbol y}) \sim {\boldsymbol X}+{\boldsymbol C}$ given values of ${\boldsymbol y}$,
%for ${\boldsymbol y}=(5,5,5)$ and $(6,6,6)$
%using R-package ``glmnet'' (\url{https://cran.r-project.org/web/packages/glmnet/index.html}). 
We conduct the bootstrapping  \citep{Efron1979} to reveal the distribution of the estimator. % of each type of PoC.
%\jin{In the appendix, you showed results by both logistic regression and logistic ridge regression. They gave very different results. Which method to use? I'm not sure it's a good idea to show results by both methods.}

%\yuta{[Comment: I have reanalyzed using glm package.]}

{\bf Results.}
We consider %the following four counterfactual statements related to 
the subject whose ID number is 1.
Let the values of her covariates be ${\boldsymbol c}_1$.
In reality, she studied for $2$ hours a week and took no extra paid classes (${\boldsymbol x}'=(2,1)$), and got $6$, $6$, and $5$ scores in the final, second, and first grades, respectively (${\boldsymbol y}'=(6,6,5)$).
The other attributes of her are shown in Appendix \ref{appB}.
%In this section, we pick up school, sex, and age as the subjects' covariates, and Appendix \ref{appB} provides the estimates by logistic regression and logistic ridge regression, including all variables. \jin{What do you mean by "including all variables"?}

In the first study, we evaluate conditional PNS, PN, and PS by setting ${\boldsymbol y}=(6,6,6)$, ${\boldsymbol x}_0=(2,1)$, ${\boldsymbol x}_1=(4,2)$, and ${\boldsymbol C}={\boldsymbol c}_1$ in Def. \ref{def41} to reveal the necessity/sufficiency of setting ${\boldsymbol x}_1$ w.r.t.
${\boldsymbol x}_0$ to produce ${\boldsymbol Y}\succeq_{\text{lexi}}{\boldsymbol y}$ in the  sub-population characterized by ${\boldsymbol C}={\boldsymbol c}_1$. 
%\yuta{Using PNS, PN, and PS, we try revealing the causal relationship between setting ${\boldsymbol x}_0$ compared to ${\boldsymbol x}_1$ and provoking  ${\boldsymbol Y}\succeq_{\text{lexi}}{\boldsymbol y}$ for a sub-population ${\boldsymbol C}={\boldsymbol c}_1$.}
The estimated values of conditional PNS, PN, and PS are 
%\jin{These numbers are different from those shown in the appendix???}
\begin{equation}
    \begin{aligned}
       &\text{PNS:} &8.862 \% &(\text{CI}: [1.122\%,19.510\%]),\\
       &\text{PN:} &9.212 \% &(\text{CI}: [1.133\%,20.647\%]),\\
      &\text{PS:}  &72.331 \% &(\text{CI}: [27.975\%,93.022\%]),
    \end{aligned}
\end{equation}
where CI represents 95$\%$ confidence intervals. 
The PNS value above represents the probability of the following statement:
\vspace{-0.25cm}
\begin{center}
    ``{\it {A student with attributes value ${\boldsymbol c}_1$}  
    would get scores ${\boldsymbol Y}\succeq_{\text{lexi}}{\boldsymbol y}$ had she studied 4 hours a week and taken extra classes and would get scores ${\boldsymbol Y}\prec_{\text{lexi}}{\boldsymbol y}$ had she studied 2 hours a week and taken no extra class.}''
\end{center}
\vspace{-0.25cm}
PN means the probability of the following statement:
\vspace{-0.25cm}
\begin{center}
    ``{\it {A student with attributes value ${\boldsymbol c}_1$} would get scores ${\boldsymbol Y}\prec_{\text{lexi}}{\boldsymbol y}$ had she studied 2 hours a week and taken no extra class when, in reality, she scored ${\boldsymbol Y}\succeq_{\text{lexi}}{\boldsymbol y}$, studied 4 hours a week, and took extra classes.}''
\end{center}
\vspace{-0.25cm} 
%For instance, 
And PS means the probability of the following statement:
\vspace{-0.25cm}
\begin{center}
    ``{\it {A student with attributes value ${\boldsymbol c}_1$} would get scores ${\boldsymbol Y}\succeq_{\text{lexi}}{\boldsymbol y}$ had she studied 4 hours a week and taken extra classes when,  in reality, she scored ${\boldsymbol Y}\prec_{\text{lexi}}{\boldsymbol y}$, studied 2 hours a week, and took no extra class.}''
\end{center}
\vspace{-0.25cm}
The results reveal that PNS and PN are relatively low, and PS is relatively high. In other words, 
studying 4 hours and taking extra classes for students with attributes value ${\boldsymbol c}_1$ are unlikely ``necessary and sufficient" or ``necessary”  to achieve ${\boldsymbol Y}\succeq_{\text{lexi}}{\boldsymbol y}$ compared to studying 2 hours and taking no extra class; however, they are highly ``sufficient".

In the second study, we consider more detailed evidence  than the first study and evaluate conditional PNS with evidence $({\boldsymbol y}',{\boldsymbol x}',{\boldsymbol c})$, letting ${\boldsymbol y}=(6,6,6)$, ${\boldsymbol y}'=(6,6,5)$, ${\boldsymbol x}_0=(2,1)$,
${\boldsymbol x}_1=(4,2)$, 
${\boldsymbol x}'=(2,1)$, and ${\boldsymbol C}={\boldsymbol c}_1$ in Def. \ref{EV1}. 
%\yuta{We consider more detailed evidence $({\boldsymbol y}',{\boldsymbol x}',{\boldsymbol c}_1)$ than the first analysis.}
The estimated value  is 
\begin{equation}
     \text{PNS:}\ \  0.024 \%\ \ \ \  (\text{CI}: [0.000\%,0.243\%]),
\end{equation}
which means 
the probability of the following statement:
\vspace{-0.25cm}
\begin{center}
    ``\it A student with attributes value ${\boldsymbol c}_1$ would get scores ${\boldsymbol Y}\succeq_{\text{lexi}}{\boldsymbol y}$ had she studied 4 hours a week and taken extra classes and
    would get scores ${\boldsymbol Y}\prec_{\text{lexi}}{\boldsymbol y}$ had she studied 2 hours a week and taken no extra class 
    when she scored ${\boldsymbol Y}={\boldsymbol y}'$, studied 2 hours a week, and took no extra class in reality.''
\end{center}
\vspace{-0.25cm}
We reveal that this probability is very low, that is, 
studying 4 hours and taking extra classes for students with $({\boldsymbol y}',{\boldsymbol x}',{\boldsymbol c}_1)$ are probably not ``necessary and sufficient" to achieve ${\boldsymbol Y}\succeq_{\text{lexi}}{\boldsymbol y}$ compared to studying 2 hours and taking no extra class.

In the third study, we evaluate conditional PNS with multi-hypothetical terms, letting ${\boldsymbol y}_1=(5,5,5)$, ${\boldsymbol y}_2=(6,6,6)$, ${\boldsymbol x}_0=(1,1)$, ${\boldsymbol x}_1=(2,1)$, ${\boldsymbol x}_2=(4,2)$, and ${\boldsymbol C}={\boldsymbol c}_1$ in Def. \ref{EV2}.
%\yuta{
%We next focus on achieving scores ${\boldsymbol y}_1 \preceq_{\text{lexi}}{\boldsymbol Y}\prec_{\text{lexi}}{\boldsymbol y}_2$.
%Using PNS with multi-hypothetical terms, we try revealing the causal relationship between setting ${\boldsymbol x}_1$ compared to ${\boldsymbol x}_0$ and ${\boldsymbol x}_2$ and provoking ${\boldsymbol y}_1 \preceq_{\text{lexi}}{\boldsymbol Y}\prec_{\text{lexi}}{\boldsymbol y}_2$ for a sub-population ${\boldsymbol C}={\boldsymbol c}_1$.}
The estimated value  is 
\begin{equation}
     \text{PNS:}\ \  0.000 \%\ \ \ \  (\text{CI}: [0.000\%,0.000\%]),
\end{equation}
which means the joint probability of the  following three  counterfactual statements:
\vspace{-0.25cm}
\begin{center}
    ``{\it (i) {A student with attributes value ${\boldsymbol c}_1$} would get scores ${\boldsymbol Y}\succeq_{\text{lexi}}{\boldsymbol y}_2$ had she studied 4 hours a week and taken extra classes,\\
    (ii) she would get scores ${\boldsymbol y}_1 \preceq_{\text{lexi}}{\boldsymbol Y}\prec_{\text{lexi}}{\boldsymbol y}_2$ had she studied 2 hours a week and taken no extra classes, and\\
    (iii) she would get scores ${\boldsymbol Y}\prec_{\text{lexi}}{\boldsymbol y}_1$ had she studied $1$ hour a week and taken no extra classes.}''
\end{center}
\vspace{-0.25cm}
We reveal that this probability is close to zero, that is, 
studying 2 hours and taking no extra class  for students with attributes value ${\boldsymbol c}_1$ are not ``necessary and sufficient" to achieve ${\boldsymbol y}_1 \preceq_{\text{lexi}}{\boldsymbol Y}\prec_{\text{lexi}}{\boldsymbol y}_2$ compared to ``studying 1 hour and taking no extra class" or ``studying 4 hours and taking extra classes".

%\yuta{[Comment: We have the folloing necessary and sufficient relationship:\\
%${\boldsymbol X}={\boldsymbol x}_1 \Rightarrow {\boldsymbol y}_1 \preceq_{\text{lexi}}{\boldsymbol Y}\prec_{\text{lexi}}{\boldsymbol y}_2$,\\
%${\boldsymbol X}={\boldsymbol x}_0 \lor {\boldsymbol X}={\boldsymbol x}_2 \Rightarrow \lnot ({\boldsymbol y}_1 \preceq_{\text{lexi}}{\boldsymbol Y}\prec_{\text{lexi}}{\boldsymbol y}_2)$,\\
%where $\lnot ({\boldsymbol X}={\boldsymbol x}_1)=({\boldsymbol X}={\boldsymbol x}_0) \lor ({\boldsymbol X}={\boldsymbol x}_2)$.
%]}

Finally, we consider more detailed evidence than the third study and  evaluate conditional PNS with multi-hypothetical terms and evidence $({\boldsymbol y}',{\boldsymbol x}',{\boldsymbol c})$, letting ${\boldsymbol y}_1=(5,5,5)$, ${\boldsymbol y}_2=(6,6,6)$, ${\boldsymbol y}'=(6,6,5)$, ${\boldsymbol x}_0=(1,1)$, ${\boldsymbol x}_1=(2,1)$,
${\boldsymbol x}_2=(4,2)$, 
${\boldsymbol x}'=(2,1)$, and ${\boldsymbol C}={\boldsymbol c}_1$ in Def. \ref{EV3}. 
%\yuta{We consider more detailed evidence $({\boldsymbol y}',{\boldsymbol x}',{\boldsymbol c}_1)$ than the third analysis.}
The estimated value is 
\begin{equation}
     \text{PNS:}\ \  96.711 \%\ \ \ \  (\text{CI}: [59.059\%,100.000\%]),
\end{equation}
which represents the probability of the above three counterfactual statements in the third study given additional information ${\boldsymbol x}'$ and ${\boldsymbol y}'$.
Unlike PNS with multi-hypothetical terms in the third study, PNS with multi-hypothetical terms and evidence $({\boldsymbol y}',{\boldsymbol x}',{\boldsymbol c}_1)$ is relatively high. That is, 
studying 2 hours and taking no extra class  with $({\boldsymbol y}',{\boldsymbol x}',{\boldsymbol c}_1)$ are highly ``necessary and sufficient" to achieve ${\boldsymbol y}_1 \preceq_{\text{lexi}}{\boldsymbol Y}\prec_{\text{lexi}}{\boldsymbol y}_2$ compared to ``studying 1 hour and taking no extra class" and ``studying 4 hours and taking extra classes". 
%\yuta{Given $({\boldsymbol y}',{\boldsymbol x}',{\boldsymbol c})$, studying 2 hours and taking no extra class can be ``necessary and sufficient" to achieve ${\boldsymbol y}_1 \preceq_{\text{lexi}}{\boldsymbol Y}\prec_{\text{lexi}}{\boldsymbol y}_2$ compared to ``studying 1 hours and taking no extra class" and ``studying 4 hours and taking extra classes".}

%The estimated values are all $0$ or $1$. \jin{What values are 0 or 1?}\yuta{[Old Results.]}

%\jin{Other potential meaningful experiments to perform? Maybe PoC of each $X^1$ and $X^2$ individually to compare with joint effects?}

We have performed additional analyses. To evaluate the effect of study time ($X^1$) only, we let ${\boldsymbol x}_1=(4,1)$ in the first and second analyses, and ${\boldsymbol x}_2=(4,1)$ in the third and fourth analyses.
The results are shown in Appendix \ref{appB}, and all estimated PoC are lower than that obtained with joint effect of study time and extra paid classes.
To evaluate the effect of extra paid classes ($X^2)$ only, we let ${\boldsymbol x}_1=(2,2)$ in the first and second analyses.
The results are shown in Appendix \ref{appB}, and all estimated PoC  are also lower than the results with joint effect.

\section{Conclusion}

We introduce new types of PoC to capture the causal effects between multiple continuous treatments and outcomes and provide identification theorems. 
The results greatly expand the range of causal questions that researchers can tackle going beyond binary treatment and outcome. 
{In this paper, we focus on the form of PoC where all treatments are intervened. 
The scenario of just intervening only a subset of all treatment variables is also useful in real life \citep{Lu2022,Li2023B}, which will be future research.} 
In settings where the monotonicity assumptions do not hold, we may explore methods for bounding PoC. 
However, for continuous variables, we can not straightforwardly apply linear programming formulation used for bounding binary PoC in \citep{Tian2000,Li2022}.
%Estimates in an application have relatively large variances as causal quantities get more complicated; thus, we need more efficient and robust estimators.
Bounding PoC introduced in this paper will be an interesting future work.

\section*{Acknowledgements}
The authors  thank the anonymous reviewers for their time and thoughtful comments. 
Yuta Kawakami was supported by JSPS KAKENHI Grant Number 22J21928. 
Manabu Kuroki was supported by JSPS KAKENHI Grant Number 21H03504 and 24K14851.
Jin Tian was partially supported by NSF grant CNS-2321786.

\bibliography{citation}
\newpage

\onecolumn

\title{Appendix to ``Probabilities of Causation for Continuous and Vector Variables”}
\maketitle

\appendix

\section{Additional Information on Background and Notation}
\label{app0}

{\bf Orders.}
We explain the orders used in this paper.
The definition of the total order is as below \citep{Harzheim2005}:
%\yuta{
%\begin{definition}[Partial order]
%    A partial order on a set $\Omega$ is a relation ``$\preceq_{\text{p}}$'' on $\Omega$ satisfying the following three conditions for all ${\boldsymbol a}_1, {\boldsymbol a}_2, {\boldsymbol a}_3 \in \Omega$:
%    \begin{enumerate}
%    \vspace{-0cm}
%      \setlength{\parskip}{0.cm}
%  \setlength{\itemsep}{0.2cm}
%        \item ${\boldsymbol a}_1\preceq_{\text{p}} {\boldsymbol a}_1$;
%        \item if ${\boldsymbol a}_1\preceq_{\text{p}} {\boldsymbol a}_2$ and ${\boldsymbol a}_2\preceq_{\text{p}} {\boldsymbol a}_3$ then ${\boldsymbol a}_1\preceq_{\text{p}} {\boldsymbol a}_3$;
 %       \item if ${\boldsymbol a}_1\preceq_{\text{p}} {\boldsymbol a}_2$ and ${\boldsymbol a}_2\preceq_{\text{p}} {\boldsymbol a}_1$ then ${\boldsymbol a}_1= {\boldsymbol a}_2$.
%    \end{enumerate}
%    \vspace{-0cm}
    %, and denote a partially ordered set $(\Omega,\preceq_{\text{p}} )$.
%\end{definition}
%In this case we say that the ordered pair $(\Omega,\preceq_{\text{p}})$ is a partially ordered set.
%}
%\jin{The above definition reads strange. Need rephrasing. The following definition needs rephrasing too.}
\begin{definition}[Total order]
    A total order on a set $\Omega$ is a relation ``$\preceq$'' on $\Omega$ satisfying the following four conditions for all ${\boldsymbol a}_1, {\boldsymbol a}_2, {\boldsymbol a}_3 \in \Omega$:
    \begin{enumerate}
    \vspace{-0cm}
      \setlength{\parskip}{0.cm}
  \setlength{\itemsep}{0.2cm}
        \item ${\boldsymbol a}_1\preceq{\boldsymbol a}_1$;
        \item if ${\boldsymbol a}_1\preceq{\boldsymbol a}_2$ and ${\boldsymbol a}_2\preceq{\boldsymbol a}_3$ then ${\boldsymbol a}_1\preceq{\boldsymbol a}_3$;
        \item if ${\boldsymbol a}_1\preceq{\boldsymbol a}_2$ and ${\boldsymbol a}_2\preceq{\boldsymbol a}_1$ then ${\boldsymbol a}_1= {\boldsymbol a}_2$;
        \item at least one of ${\boldsymbol a}_1\preceq{\boldsymbol a}_2$ and ${\boldsymbol a}_2\preceq{\boldsymbol a}_1$ holds.
    \end{enumerate}
    \vspace{-0cm}
    %, and denote a partially ordered set $(\Omega,\preceq_{\text{p}} )$.
\end{definition}
In this case we say that the ordered pair $(\Omega,\preceq)$ is a totally ordered set. 
The inequality ${\boldsymbol a}\preceq {\boldsymbol b}$ of total order means ${\boldsymbol a}\prec {\boldsymbol b}$ or ${\boldsymbol a}={\boldsymbol b}$, and the relationship $\lnot ({\boldsymbol a}\preceq {\boldsymbol b}) \Leftrightarrow {\boldsymbol a}\succ {\boldsymbol b}$ holds for a totally ordered set, where $\lnot$ means the negation.

\begin{definition}[Lexicographical order for the Cartesian product]
A lexicographic order $\prec$ %is a relation 
on the Cartesian product of two sets $\Omega_A$ and $\Omega_B$ with order relations $\preceq_A$ and $\preceq_B$ satisfies:   for all $(a_1,b_1) \in \Omega_A \times \Omega_B$ and $(a_2,b_2) \in \Omega_A \times \Omega_B$, $(a_1,b_1) \prec (a_2,b_2)$ if and only if either  
\begin{enumerate}
    \vspace{-0cm}
      \setlength{\parskip}{0.cm}
  \setlength{\itemsep}{0cm}
        \item $a_1 \prec_A a_2$, or
        \item $a_1 = a_2$ and $b_1 \prec_B b_2$.
    \end{enumerate}
    %For two vectors ${\boldsymbol a}_1=(a_1^1,a_1^2,\ldots,a_1^{d})$ and ${\boldsymbol a}_2=(a_2^1,a_2^2,\ldots,a_2^{d})$, they are ${\boldsymbol a}_1\prec_{\text{lexi}} {\boldsymbol a}_2$ if and only if ``$a_1^1<a_2^1$'' or ``both $a_1^1=a_2^1$ and $a_1^2<a_2^2$'' or $\ldots$ or ``$a_1^1=a_2^1, \ldots a_1^{d-1}=a_2^{d-1}$ and $a_1^{d} <a_2^{d}$''.
\end{definition}
The lexicographic order can be readily extended to Cartesian products of arbitrary length by recursively applying this definition, i.e., by observing that $\Omega_A\times \Omega_B \times \Omega_C =\Omega_A \times(\Omega_B \times \Omega_C)$. 
%\jin{The above definition reads awkward. All of the three definitions 2.2-2.4 are standard concepts in the literature. You may simply use the standard definitions from the literature instead of writing them yourself in an awkward way. }
This is widely used as one example of the total order for vector space.
The order defined by Mahalanobis' distance and Gaussian distribution \citep{Hannart2018} is one example of the total orders.
Briefly, they consider mapping $\phi$ from $\Omega=\mathbb{R}^{d}$ to $\mathbb{R}$, and define ${\boldsymbol a}_0 \preceq {\boldsymbol a}_1$ by the relationship $\phi({\boldsymbol a}_0) \leq \phi({\boldsymbol a}_1)$.

%Due to space constraints, all proofs are provided in the Appendix \ref{appA}. 

\section{Additional Information on Analyzing Trajectories}
\label{app01}

In this section, we give the analyzing the trajectories of potential outcomes of Theorems \ref{THE51} and \ref{THE52}.

{\bf Analyzing Trajectories
of Theorem \ref{THE51}.}
We denote ${\boldsymbol u}_{\rho({\boldsymbol y};{\boldsymbol x},{\boldsymbol c})}=\sup\{{\boldsymbol u}: f_{\boldsymbol Y}({\boldsymbol x},{\boldsymbol c},{\boldsymbol u})\prec {\boldsymbol y}\}$ and ${\boldsymbol u}_{\rho^o({\boldsymbol y};{\boldsymbol x},{\boldsymbol c})}=\sup\{{\boldsymbol u}: f_{\boldsymbol Y}({\boldsymbol x},{\boldsymbol c},{\boldsymbol u})\preceq {\boldsymbol y}\}$. 
Given ${\boldsymbol C}={\boldsymbol c}$, the trajectory $\{({\boldsymbol x},{\boldsymbol Y}_{\boldsymbol x}({\boldsymbol u})) \in \Omega_{\boldsymbol X} \times \Omega_{\boldsymbol Y};\forall {\boldsymbol x} \in \Omega_{\boldsymbol X}\}$ represents potential outcome ${\boldsymbol Y}_{\boldsymbol x}({\boldsymbol u})$ vs. ${\boldsymbol X}$ for the subject ${\boldsymbol U}={\boldsymbol u}$.
%One example of the relationship of trajectories is shown in Figure \ref{fig:2}.
\emph{Given ${\boldsymbol C}={\boldsymbol c}$, under Assumptions \ref{ASEXO2} and \ref{MONO2} (or \ref{AS2}, \ref{SAS2}), the subjects' trajectories do not cross over each other} (they may overlap).

{Consider the trajectories shown in Figure \ref{fig:2}. 
Given ${\boldsymbol C}={\boldsymbol c}$, Trajectory (1) $\{({\boldsymbol x},{\boldsymbol Y}_{\boldsymbol x}({\boldsymbol u}_{\rho({\boldsymbol y};{\boldsymbol x}_0,{\boldsymbol c})})) \in \Omega_{\boldsymbol X} \times \Omega_{\boldsymbol Y};\forall {\boldsymbol x} \in \Omega_{\boldsymbol X}\}$ goes through the point $({\boldsymbol x}_0,{\boldsymbol y})$, Trajectory (2) $\{({\boldsymbol x},{\boldsymbol Y}_{\boldsymbol x}({\boldsymbol u}_{\rho({\boldsymbol y};{\boldsymbol x}_1,{\boldsymbol c})})) \in \Omega_{\boldsymbol X} \times \Omega_{\boldsymbol Y};\forall {\boldsymbol x} \in \Omega_{\boldsymbol X}\}$ goes through the point $({\boldsymbol x}_1,{\boldsymbol y})$, Trajectory (3) $\{({\boldsymbol x},{\boldsymbol Y}_{\boldsymbol x}({\boldsymbol u}_{\rho^{o}({\boldsymbol y}';{\boldsymbol x}',{\boldsymbol c})})) \in \Omega_{\boldsymbol X} \times \Omega_{\boldsymbol Y};\forall {\boldsymbol x} \in \Omega_{\boldsymbol X}\}$ and Trajectory (4) $\{({\boldsymbol x},{\boldsymbol Y}_{\boldsymbol x}({\boldsymbol u}_{\rho({\boldsymbol y}';{\boldsymbol x}',{\boldsymbol c})})) \in \Omega_{\boldsymbol X} \times \Omega_{\boldsymbol Y};\forall {\boldsymbol x} \in \Omega_{\boldsymbol X}\}$ go through the point $({\boldsymbol x}',{\boldsymbol y}')$.
Given ${\boldsymbol C}={\boldsymbol c}$, the trajectory of subject ${\boldsymbol u}$ lies between in the region between Trajectories (1) and (2) if and only if they satisfy ${\boldsymbol Y}_{{\boldsymbol x}_0}\prec {\boldsymbol y} \preceq {\boldsymbol Y}_{{\boldsymbol x}_1}$ given ${\boldsymbol C}={\boldsymbol c}$.
Given ${\boldsymbol C}={\boldsymbol c}$, the trajectory of subject ${\boldsymbol u}$ lies in the region between Trajectories (3) and (4) if and only if they satisfy $({\boldsymbol Y}={\boldsymbol y}',{\boldsymbol X}={\boldsymbol x}')$ given ${\boldsymbol C}={\boldsymbol c}$.
Thus, we have $\mathbb{P}({\boldsymbol Y}_{{\boldsymbol x}_0}\prec {\boldsymbol y} \preceq {\boldsymbol Y}_{{\boldsymbol x}_1}|{\boldsymbol Y}={\boldsymbol y}',{\boldsymbol X}={\boldsymbol x}',{\boldsymbol C}={\boldsymbol c})$ is 
\begin{equation}
    \max\left\{\frac{\min\{\mathbb{P}({\boldsymbol Y}_{{\boldsymbol x}_0}\prec {\boldsymbol y}|{\boldsymbol C}={\boldsymbol c}),\mathbb{P}({\boldsymbol Y}_{{\boldsymbol x}'}\preceq {\boldsymbol y}'|{\boldsymbol C}={\boldsymbol c})\}-\max\{\mathbb{P}({\boldsymbol Y}_{{\boldsymbol x}_1}\prec {\boldsymbol y}|{\boldsymbol C}={\boldsymbol c}),\mathbb{P}({\boldsymbol Y}_{{\boldsymbol x}'}\prec {\boldsymbol y}'|{\boldsymbol C}={\boldsymbol c})\}}{\mathbb{P}({\boldsymbol Y}_{{\boldsymbol x}'}\preceq {\boldsymbol y}'|{\boldsymbol C}={\boldsymbol c})-\mathbb{P}({\boldsymbol Y}_{{\boldsymbol x}'}\prec {\boldsymbol y}'|{\boldsymbol C}={\boldsymbol c})},0\right\},
\end{equation}
where $\mathbb{P}({\boldsymbol Y}_{{\boldsymbol x}_0}\prec {\boldsymbol y}|{\boldsymbol C}={\boldsymbol c})$ represents the probability of a subject's trajectory being below Trajectory (1), $\mathbb{P}({\boldsymbol Y}_{{\boldsymbol x}_1}\prec {\boldsymbol y}|{\boldsymbol C}={\boldsymbol c})$ represents the probability of a subject's trajectory being below Trajectory (2),  $\mathbb{P}({\boldsymbol Y}_{{\boldsymbol x}'}\prec {\boldsymbol y}'|{\boldsymbol C}={\boldsymbol c})$ represents the probability of a subject's trajectory being below Trajectory (3) and $\mathbb{P}({\boldsymbol Y}_{{\boldsymbol x}'}\preceq {\boldsymbol y}'|{\boldsymbol C}={\boldsymbol c})$ represents the probability of a subject's trajectory being below Trajectory (4).
When $\rho({\boldsymbol y}';{\boldsymbol x}',{\boldsymbol c})=\rho^{o}({\boldsymbol y}';{\boldsymbol x}',{\boldsymbol c})$, Trajectory (3) coincides with Trajectory (4). $\text{PNS}({\boldsymbol y};{\boldsymbol x}_0,{\boldsymbol x}_1,{\boldsymbol y}',{\boldsymbol x}',{\boldsymbol c})$ represents whether Trajectory (3) or (4) lies in the region between Trajectories (1) and (2), and takes value either 0 or 1.}

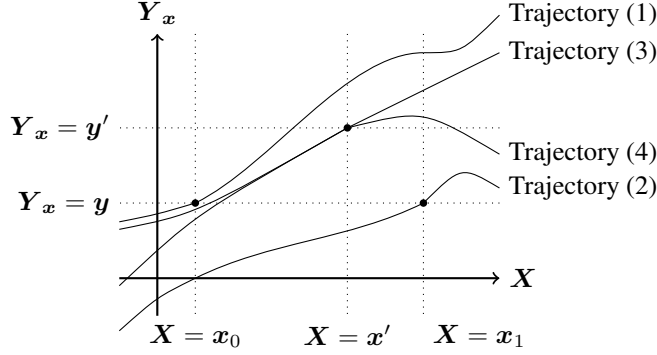
\begin{figure}
    \centering
    \scalebox{1}{
    \hspace{-0.6cm}
    \begin{tikzpicture}

  \draw (0,0) .. controls (1,0.5) and (2,2) .. (3,2);
 \draw (-1,-0.25) .. controls (-0.5,-0.15) .. (0,0);
  \draw (3,2) .. controls (3.5,2) .. (4,2.5)
    node[anchor=west] {Trajectory (1)};
    %{${\boldsymbol Y}_{\boldsymbol x}({\boldsymbol c},{\boldsymbol u}_{\rho({\boldsymbol y};{\boldsymbol x}_0,{\boldsymbol c})})$};

  \draw (0,-1) .. controls (1,-0.5) and (2,-0.5) .. (3,0);
 \draw (-1,-1.7) .. controls (-0.5,-1.25) .. (0,-1);
  \draw (3,0) .. controls (3.5,0.5) .. (4,0.2)
  node[anchor=west] {Trajectory (2)};
  %{${\boldsymbol Y}_{\boldsymbol x}({\boldsymbol c},{\boldsymbol u}_{\rho({\boldsymbol y};{\boldsymbol x}_1,{\boldsymbol c})})$};

    \draw[black,dotted] (4, 0) -- (-1, 0)
      node[anchor=east] {${\boldsymbol Y}_{\boldsymbol x}={\boldsymbol y}$};

        \draw[black,dotted] (0, 2.25) -- (0, -1.5)
      node[anchor=north] {${\boldsymbol X}={\boldsymbol x}_0$};
      
   \draw[black,dotted] (3, 2.25) -- (3, -1.5)
      node[anchor=north west] {${\boldsymbol X}={\boldsymbol x}_1$};

   \draw (-1,-0.35) .. controls (-0,-0.15) .. (2,1);

  \draw (2,1) .. controls (3,1.5) .. (4,2)
    node[anchor=west]{Trajectory (3)};
    %{${\boldsymbol Y}_{\boldsymbol x}({\boldsymbol c},{\boldsymbol u}_{\rho^{o}({\boldsymbol y}';{\boldsymbol x}',{\boldsymbol c})})$};

 \draw (-1,-1.1) .. controls (0,-0.15) .. (2,1);
  \draw (2,1) .. controls (3,1.25) .. (4,0.65)
  node[anchor=west]{Trajectory (4)};
%  {${\boldsymbol Y}_{\boldsymbol x}({\boldsymbol c},{\boldsymbol u}_{\rho({\boldsymbol y}';{\boldsymbol x}',{\boldsymbol c})})$};

    \draw[black,dotted] (4, 1) -- (-1, 1)
      node[anchor=east] {${\boldsymbol Y}_{\boldsymbol x}={\boldsymbol y}'$};

   \draw[black,dotted] (2, 2.25) -- (2, -1.5)
      node[anchor=north] {${\boldsymbol X}={\boldsymbol x}'$};

    \draw[thick, black, ->] (-0.5, -1.5) -- (-0.5, 2.25)
      node[anchor=south] {${\boldsymbol Y}_{\boldsymbol x}$};

      % x-axis
    \draw[thick, black, ->] (-1, -1) -- (4, -1)
      node[anchor=west] {${\boldsymbol X}$};

\node at (0,0)[circle,fill,inner sep=1pt]{};
\node at (3,0)[circle,fill,inner sep=1pt]{};
\node at (2,1)[circle,fill,inner sep=1pt]{};
      
\end{tikzpicture}
}
    \caption{Trajectories for (1) ${\boldsymbol Y}_{\boldsymbol x}({\boldsymbol u}_{\rho({\boldsymbol y};{\boldsymbol x}_0,{\boldsymbol c})})$, (2) ${\boldsymbol Y}_{\boldsymbol x}({\boldsymbol u}_{\rho({\boldsymbol y};{\boldsymbol x}_1,{\boldsymbol c})})$, (3) ${\boldsymbol Y}_{\boldsymbol x}({\boldsymbol u}_{\rho^{o}({\boldsymbol y}';{\boldsymbol x}',{\boldsymbol c})})$ and (4) ${\boldsymbol Y}_{\boldsymbol x}({\boldsymbol u}_{\rho({\boldsymbol y}';{\boldsymbol x}',{\boldsymbol c})})$.}
    %${\boldsymbol Y}_{\boldsymbol x}({\boldsymbol c},{\boldsymbol u}_{\rho({\boldsymbol y};{\boldsymbol x}_0,{\boldsymbol c})})$, ${\boldsymbol Y}_{\boldsymbol x}({\boldsymbol c},{\boldsymbol u}_{\rho({\boldsymbol y};{\boldsymbol x}_1,{\boldsymbol c})})$, ${\boldsymbol Y}_{\boldsymbol x}({\boldsymbol c},{\boldsymbol u}_{\rho^{o}({\boldsymbol y}';{\boldsymbol x}',{\boldsymbol c})})$ and ${\boldsymbol Y}_{\boldsymbol x}({\boldsymbol c},{\boldsymbol u}_{\rho({\boldsymbol y}';{\boldsymbol x}',{\boldsymbol c})})$.}
    \label{fig:2}
\end{figure}

{\bf Analyzing Trajectories
of Theorem \ref{THE52}.}
We provide trajectories-based explanation on Theorem \ref{THE52} when $P=2$.
{Consider the trajectories shown in Figure \ref{fig:3}. 
Given ${\boldsymbol C}={\boldsymbol c}$, Trajectory (1) $\{({\boldsymbol x},{\boldsymbol Y}_{\boldsymbol x}({\boldsymbol u}_{\rho({\boldsymbol y}_1;{\boldsymbol x}_0,{\boldsymbol c})})) \in \Omega_{\boldsymbol X} \times \Omega_{\boldsymbol Y};\forall {\boldsymbol x} \in \Omega_{\boldsymbol X}\}$ goes through the point $({\boldsymbol x}_0,{\boldsymbol y})$, Trajectory (2) $\{({\boldsymbol x},{\boldsymbol Y}_{\boldsymbol x}({\boldsymbol u}_{\rho({\boldsymbol y}_1;{\boldsymbol x}_1,{\boldsymbol c})})) \in \Omega_{\boldsymbol X} \times \Omega_{\boldsymbol Y};\forall {\boldsymbol x} \in \Omega_{\boldsymbol X}\}$ goes through the point $({\boldsymbol x}_1,{\boldsymbol y})$, Trajectory (3) $\{({\boldsymbol x},{\boldsymbol Y}_{\boldsymbol x}({\boldsymbol u}_{\rho({\boldsymbol y}_2;{\boldsymbol x}_1,{\boldsymbol c})})) \in \Omega_{\boldsymbol X} \times \Omega_{\boldsymbol Y};\forall {\boldsymbol x} \in \Omega_{\boldsymbol X}\}$ and Trajectory (4) $\{({\boldsymbol x},{\boldsymbol Y}_{\boldsymbol x}({\boldsymbol u}_{\rho({\boldsymbol y}_2;{\boldsymbol x}_2,{\boldsymbol c})})) \in \Omega_{\boldsymbol X} \times \Omega_{\boldsymbol Y};\forall {\boldsymbol x} \in \Omega_{\boldsymbol X}\}$ go through the point $({\boldsymbol x}',{\boldsymbol y}')$.
Given ${\boldsymbol C}={\boldsymbol c}$, the trajectories of subject ${\boldsymbol u}$ lies in the region between Trajectories (1) and (2) if and only if they satisfy ${\boldsymbol Y}_{{\boldsymbol x}_0}\prec {\boldsymbol y}_1 \preceq {\boldsymbol Y}_{{\boldsymbol x}_1}$.
Given ${\boldsymbol C}={\boldsymbol c}$, the trajectories of subject ${\boldsymbol u}$ lies in the region between Trajectories (3) and (4) if and only if they satisfy ${\boldsymbol Y}_{{\boldsymbol x}_1}\prec {\boldsymbol y}_2 \preceq {\boldsymbol Y}_{{\boldsymbol x}_2}$.
Thus, we have $\mathbb{P}({\boldsymbol Y}_{{\boldsymbol x}_0}\prec {\boldsymbol y}_1 \preceq {\boldsymbol Y}_{{\boldsymbol x}_1}\prec {\boldsymbol y}_2 \preceq {\boldsymbol Y}_{{\boldsymbol x}_2}|{\boldsymbol C}={\boldsymbol c})$ is 
\begin{equation}
    \max\left\{\min\{\mathbb{P}({\boldsymbol Y}_{{\boldsymbol x}_0}\prec {\boldsymbol y}_1|{\boldsymbol C}={\boldsymbol c}),\mathbb{P}({\boldsymbol Y}_{{\boldsymbol x}_1}\preceq {\boldsymbol y}_2|{\boldsymbol C}={\boldsymbol c})\}-\max\{\mathbb{P}({\boldsymbol Y}_{{\boldsymbol x}_1}\prec {\boldsymbol y}_1|{\boldsymbol C}={\boldsymbol c}),\mathbb{P}({\boldsymbol Y}_{{\boldsymbol x}_2}\prec {\boldsymbol y}_2|{\boldsymbol C}={\boldsymbol c})\},0\right\},
\end{equation}
where $\mathbb{P}({\boldsymbol Y}_{{\boldsymbol x}_0}\prec {\boldsymbol y}_1|{\boldsymbol C}={\boldsymbol c})$ represents the probability of a subject's trajectory being below Trajectory (1), $\mathbb{P}({\boldsymbol Y}_{{\boldsymbol x}_1}\prec {\boldsymbol y}_2|{\boldsymbol C}={\boldsymbol c})$ represents the probability of a subject's trajectory being below Trajectory (2),  $\mathbb{P}({\boldsymbol Y}_{{\boldsymbol x}_1}\prec {\boldsymbol y}_2|{\boldsymbol C}={\boldsymbol c})$ represents the probability of a subject's trajectory being below Trajectory (3) and $\mathbb{P}({\boldsymbol Y}_{{\boldsymbol x}_2}\preceq {\boldsymbol y}_2|{\boldsymbol C}={\boldsymbol c})$ represents the probability of a subject's trajectory being below Trajectory (4).}

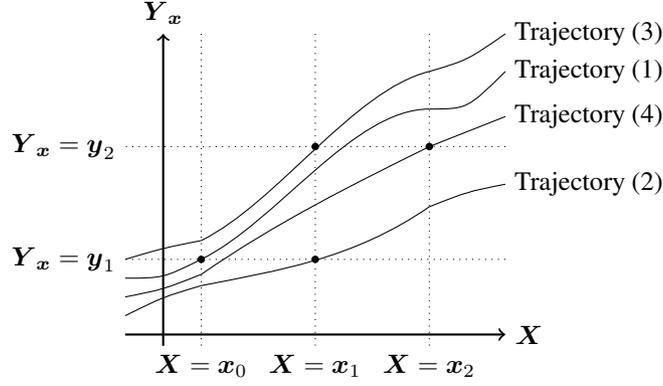
\begin{figure}
    \centering
    \scalebox{1}{
    \begin{tikzpicture}
\draw (0,0.25) .. controls (1,0.75) and (2,2.25) .. (3,2.5);
 \draw (-1,0) .. controls (-0.5,0.15) .. (0,0.25);
  \draw (3,2.5) .. controls (3.5,2.65) .. (4,3)
    node[anchor=west]{Trajectory (3)};
    %{${\boldsymbol Y}_{\boldsymbol x}({\boldsymbol c},{\boldsymbol u}_{\rho({\boldsymbol y}_2;{\boldsymbol x}_1,{\boldsymbol c})})$};

  \draw (0,0) .. controls (1,0.5) and (2,2) .. (3,2);
 \draw (-1,-0.25) .. controls (-0.5,-0.25) .. (0,0);
  \draw (3,2) .. controls (3.5,2) .. (4,2.5)
    node[anchor=west]{Trajectory (1)};%{${\boldsymbol Y}_{\boldsymbol x}({\boldsymbol c},{\boldsymbol u}_{\rho({\boldsymbol y}_1;{\boldsymbol x}_0,{\boldsymbol c})})$};

  \draw (0,-0.35) .. controls (1,-0.15) and (2,0) .. (3,0.7);
 \draw (-1,-0.75) .. controls (-0.5,-0.5) .. (0,-0.35);
  \draw (3,0.7) .. controls (3.5,0.9) .. (4,1)
  node[anchor=west]{Trajectory (2)};
  %{${\boldsymbol Y}_{\boldsymbol x}({\boldsymbol c},{\boldsymbol u}_{\rho({\boldsymbol y}_1;{\boldsymbol x}_1,{\boldsymbol c})})$};
  
  \draw (0,-0.2) .. controls (1,0.5) and (2,1) .. (3,1.5);
 \draw (-1,-0.5) .. controls (-0.5,-0.4) .. (0,-0.2);
  \draw (3,1.5) .. controls (3.5,1.7) .. (4,1.9)
  node[anchor=west]{Trajectory (4)};
  %{${\boldsymbol Y}_{\boldsymbol x}({\boldsymbol c},{\boldsymbol u}_{\rho({\boldsymbol y}_2;{\boldsymbol x}_2,{\boldsymbol c})})$};

    \draw[black,dotted] (4, 0) -- (-1, 0)
      node[anchor=east] {${\boldsymbol Y}_{\boldsymbol x}={\boldsymbol y}_1$};

      \draw[black,dotted] (4, 1.5) -- (-1, 1.5)
      node[anchor=east] {${\boldsymbol Y}_{\boldsymbol x}={\boldsymbol y}_2$};

        \draw[black,dotted] (0, 3) -- (0, -1.15)
      node[anchor=north] {${\boldsymbol X}={\boldsymbol x}_0$};
      
      \draw[black,dotted] (1.5, 3) -- (1.5, -1.15)
      node[anchor=north] {${\boldsymbol X}={\boldsymbol x}_1$};
      
   \draw[black,dotted] (3, 3) -- (3, -1.15)
      node[anchor=north] {${\boldsymbol X}={\boldsymbol x}_2$};

    % y-axis
    \draw[thick, black, ->] (-0.5, -1.15) -- (-0.5, 3)
      node[anchor=south] {${\boldsymbol Y}_{\boldsymbol x}$};

      % x-axis
    \draw[thick, black, ->] (-1, -1) -- (4, -1)
      node[anchor=west] {${\boldsymbol X}$};
      
\node at (0,0)[circle,fill,inner sep=1pt]{};
\node at (1.5,0)[circle,fill,inner sep=1pt]{};
\node at (1.5,1.5)[circle,fill,inner sep=1pt]{};
\node at (3,1.5)[circle,fill,inner sep=1pt]{};
      
\end{tikzpicture}
}
\vspace{-0cm}
    \caption{Trajectories for (1) ${\boldsymbol Y}_{\boldsymbol x}({\boldsymbol u}_{\rho({\boldsymbol y}_1;{\boldsymbol x}_0,{\boldsymbol c})})$, (2) ${\boldsymbol Y}_{\boldsymbol x}({\boldsymbol u}_{\rho({\boldsymbol y}_1;{\boldsymbol x}_1,{\boldsymbol c})})$, (3) ${\boldsymbol Y}_{\boldsymbol x}({\boldsymbol u}_{\rho({\boldsymbol y}_2;{\boldsymbol x}_1,{\boldsymbol c})})$ and (4) ${\boldsymbol Y}_{\boldsymbol x}({\boldsymbol u}_{\rho({\boldsymbol y}_2;{\boldsymbol x}_2,{\boldsymbol c})})$.}
    %${\boldsymbol Y}_{\boldsymbol x}({\boldsymbol c},{\boldsymbol u}_{\rho({\boldsymbol y}_2;{\boldsymbol x}_1,{\boldsymbol c})})$, ${\boldsymbol Y}_{\boldsymbol x}({\boldsymbol c},{\boldsymbol u}_{\rho({\boldsymbol y}_1;{\boldsymbol x}_0,{\boldsymbol c})})$, ${\boldsymbol Y}_{\boldsymbol x}({\boldsymbol c},{\boldsymbol u}_{\rho({\boldsymbol y}_2;{\boldsymbol x}_2,{\boldsymbol c})})$ and ${\boldsymbol Y}_{\boldsymbol x}({\boldsymbol c},{\boldsymbol u}_{\rho({\boldsymbol y}_1;{\boldsymbol x}_1,{\boldsymbol c})})$.}
    \label{fig:3}
\end{figure}

\section{Proofs}

\label{appA}
We give the proof of lemmas, theorems, and corollary in the body of the paper.

\subsection{Proofs in Section \ref{sec3}}

\begin{lemma}
\label{LEM31}
    Under SCM ${\cal M}_S$ and Assumption \ref{SUP1}, 
    Assumption \ref{MONO_A} implies Assumption \ref{AS1}.
\end{lemma}

\begin{proof}
    Suppose the negation of Assumption \ref{AS1}:  
    \begin{center}
        there exists a set ${\cal U}$ such that $0<\mathbb{P}({\cal U})<1$, and $$f_Y(x_0,u_0)\geq y> f_Y(x_0,u_1) \land f_Y(x_1,u_0)< y \leq f_Y(x_1,u_1)$$ for some $x_0, x_1 \in \Omega_X$ and $y \in \Omega_Y$ and for any $u_0,u_1 \in {\cal U}$ such that $u_0< u_1$.
    \end{center}
    Assumption \ref{SUP1} guarantees the existence of overlapping $y$ values in the above since ``no overlap" situation $\{f_Y(x_0,u):u \in \Omega_U\}\cap \{f_Y(x_1,u):u \in \Omega_U\}=\emptyset$ means the intersection of the support of $Y_{x_0}$ and the support of $Y_{x_1}$ is empty, which violates Assumption \ref{SUP1}.
    
Then we have
    \begin{center}
        $f_Y(x_0,u_0)\geq y > f_Y(x_1,u_0)$ and $f_Y(x_0,u_1)< y \leq f_Y(x_1,u_1)$ for some $x_0, x_1 \in \Omega_X$ and $y \in \Omega_Y$ and\\ for any $u_0, u_1 \in {\cal U}$ such that $u_0 < u_1$,
    \end{center}
    and it implies
    \begin{center}
        $f_Y(x_0,u)\geq y > f_Y(x_1,u)$ and $f_Y(x_0,u)< y \leq f_Y(x_1,u)$ for some $x_0, x_1 \in \Omega_X$ and $y \in \Omega_Y$ and  for any $u \in {\cal U}$.
    \end{center}
    This implies the negation of Assumption \ref{MONO_A} $\mathbb{P}(Y_{x_0}< y \leq Y_{x_1})\ne 0$ and $\mathbb{P}(Y_{x_1}< y \leq Y_{x_0})\ne 0$ for some $x_0, x_1 \in \Omega_X$ and $y \in \Omega_Y$ since $f_Y(x_0,u)\geq y > f_Y(x_1,u) \Leftrightarrow Y_{x_0}(u)> y \geq Y_{x_1}(u)$ and $f_Y(x_0,u)< y \leq f_Y(x_1,u) \Leftrightarrow Y_{x_1}(u)\geq y > Y_{x_0}(u)$.
In conclusion, we have lemma \ref{LEM31} by taking a contraposition of the above statements.
\end{proof}

\begin{lemma}
\label{LEM32}
     Under SCM ${\cal M}_S$ and Assumption \ref{SUP1}, 
     Assumption \ref{AS1} implies Assumption \ref{MONO_A}.
\end{lemma}

\begin{proof}

First, we denote $u_{sup}=\sup\{u:f_Y(x_0,u)< y\}$. 
We consider the situations ``the function $f_Y(x,U)$ is monotonic increasing on $U$'' and ``the function $f_Y(x,U)$ is monotonic decreasing on $U$'', separately.

{\bf (1).}
If the function $f_Y(x,U)$ is {\bf monotonic increasing} on $U$
for all $x \in \Omega_X$ almost surely w.r.t. $\mathbb{P}_U$, we have
\begin{equation}
    f_Y(x_0,u_{sup}) \leq f_Y(x_0,u) \text{ and }f_Y(x_1,u_{sup}) \leq f_Y(x_1,u)
\end{equation}
for $\mathbb{P}_U$-almost every $u \in \Omega_U$ such that $u\geq u_{sup}$.
We have the following statements:
\begin{enumerate}
    \item Supposed $f_Y(x_0,u_{sup})> f_Y(x_1,u_{sup})$, 
we have $y= f_Y(x_0,u_{sup}) > f_Y(x_1,u_{sup})\geq f_Y(x_1,u)=Y_{x_1}(u)$ for $\mathbb{P}_U$-almost every $u \in \Omega_U$ such that $f_Y(x_0,u)< y$.
It means $Y_{x_0}(u)< y \Rightarrow Y_{x_1}(u) < y$ for $\mathbb{P}_U$-almost every $u \in \Omega_U$ and  $\mathbb{P}(Y_{x_0}< y \leq Y_{x_1})=0$.
\item Supposed $f_Y(x_0,u_{sup})\leq f_Y(x_1,u_{sup})$, we have $f_Y(x_1,u) \geq f_Y(x_1,u_{sup}) \geq f_Y(x_0,u_{sup}) =y$ for $\mathbb{P}_U$-almost every $u \in \Omega_U$ such that $f_Y(x_0,u)\geq y$.
It means $Y_{x_0}(u)\geq y \Rightarrow Y_{x_1}(u) \geq y$ for $\mathbb{P}_U$-almost every $u \in \Omega_U$ and  $\mathbb{P}(Y_{x_1}< y \leq Y_{x_0})=0$.
\end{enumerate}
Then, Assumption \ref{AS1} holds.

{\bf (2).}
If the function $f_Y(x,U)$ is {\bf monotonic decreasing} on $U$
for all $x \in \Omega_X$ almost surely w.r.t. $\mathbb{P}_U$, we have
\begin{equation}
    f_Y(x_0,u_{sup}) \geq f_Y(x_0,u) \text{ and }f_Y(x_1,u_{sup}) \geq f_Y(x_1,u)
\end{equation}
for $\mathbb{P}_U$-almost every $u \in \Omega_U$ such that $u\geq u_{sup}$.
We have the following statements:
\begin{enumerate}
    \item Supposed $f_Y(x_0,u_{sup})\leq f_Y(x_1,u_{sup})$, 
we have $y= f_Y(x_0,u_{sup}) \leq f_Y(x_1,u_{sup})\leq f_Y(x_1,u)=Y_{x_1}(u)$ for $\mathbb{P}_U$-almost every $u \in \Omega_U$ such that $f_Y(x_0,u)\geq y$.
It means $Y_{x_0}(u)\geq y \Rightarrow Y_{x_1}(u) \geq y$ for $\mathbb{P}_U$-almost every $u \in \Omega_U$ and  $\mathbb{P}(Y_{x_1}< y \leq Y_{x_0})=0$.
\item Supposed $f_Y(x_0,u_{sup})> f_Y(x_1,u_{sup})$, we have $f_Y(x_1,u) \leq f_Y(x_1,u_{sup}) < f_Y(x_0,u_{sup}) =y$ for $\mathbb{P}_U$-almost every $u \in \Omega_U$ such that $f_Y(x_0,u)< y$.
It means $Y_{x_0}(u)< y \Rightarrow Y_{x_1}(u) < y$ for $\mathbb{P}_U$-almost every $u \in \Omega_U$ and  $\mathbb{P}(Y_{x_0}< y \leq Y_{x_1})=0$.
\end{enumerate}
Thus, Assumption \ref{AS1} holds.
In conclusion, Assumption \ref{AS1} implies Assumption \ref{MONO_A}
\end{proof}

{\bf Theorem \ref{E12}.} 
{\it Under SCM ${\cal M}_S$ and Assumption \ref{SUP1}, 
    Assumptions \ref{MONO_A} and \ref{AS1} are equivalent, and
    %Assumptions \ref{SAS1} and \ref{RP1} are equivalent. 
    {Assumptions \ref{SAS1} is a strictly stronger requirement than \ref{AS1}.}

\begin{proof}
We have Theorem \ref{E12} from Lemma \ref{LEM31} and \ref{LEM32}.
%, \ref{LEM33} and \ref{LEM34}.
\end{proof}

{\bf Theorem \ref{THEO1}.} (Identification of PoC)
{\it 
Under SCM ${\cal M}_{S}$ and Assumptions \ref{ASEXO}, \ref{MONO_A} (or \ref{AS1}, \ref{SAS1}), and \ref{SUP1}, 
PNS, PN, and PS are identifiable by
\begin{equation}
    \begin{aligned}
    &\text{PNS}(y;x_0,x_1)=\max\{\rho(y;x_0)-\rho(y;x_1),0\},\\
    &\text{PN}(y;x_0,x_1)=\max\left\{\frac{\rho(y;x_0)-\rho(y;x_1)}{1-\rho(y;x_1)},0\right\},\\
    &\text{PS}(y;x_0,x_1)=\max\left\{\frac{\rho(y;x_0)-\rho(y;x_1)}{\rho(y;x_0)},0\right\}
    \end{aligned}
\end{equation}
for any  $x_0,x_1 \in \Omega_X$ and $y \in \Omega_Y$ such that $\rho(y;x_1)<1$ and $\rho(y;x_0)>0$.}

\begin{proof} 
Under Assumptions \ref{ASEXO} and \ref{AS1},  
\begin{equation}
\begin{aligned}
\text{PNS}(y;x_0,x_1)
&=\mathbb{P}(Y_{x_0} < y \leq Y_{x_1})\\
&=\mathbb{P}(u_{\rho(y;x_1)}\leq u < u_{\rho(y;x_0)})\\
&=\max\{\rho(y;x_0)-\rho(y;x_1),0\}
    \end{aligned}
\end{equation}
for any  $x_0,x_1 \in \Omega_X$ and $y \in \Omega_Y$,
where $u_{\rho(y;x_0)}=\sup\{u:f_Y(x_0,u)< y\}$ and $u_{\rho(y;x_1)}=\sup\{u:f_Y(x_1,u)\leq y\}$.
Note that all $u$ such that $u \leq u_{\rho(y;x)}$ satisfy $f_Y(x,u)< y$ from Assumption \ref{AS1}.
%=\inf\{u:y < f_Y(x_1,u)\}$.
\begin{comment}
On the other hands, under Assumptions \ref{ASEXO} and \ref{MONO_A}, 
\begin{equation}
\begin{aligned}
\text{PNS}(y;x_0,x_1)
&=\mathbb{P}(Y_{x_0} < y \leq Y_{x_1})\\
&=\mathbb{P}(Y_{x_0} < y)-\mathbb{P}(Y_{x_1} < y)+\mathbb{P}(Y_{x_1} < y \leq Y_{x_0})\\
&=\max\{\rho(y;x_0)-\rho(y;x_1),0\}.
    \end{aligned}
\end{equation}
\end{comment}
In addition, $\text{PN}(y;x_0,x_1)$ and $\text{PS}(y;x_0,x_1)$ are given from the following relationship:
\begin{equation}
\begin{aligned}
\text{PN}(y;x_0,x_1)=\frac{\text{PSN}(y;x_0,x_1)}{\mathbb{P}(y\leq Y|X={x_1})}, \ \ \text{PS}(y;x_0,x_1)=\frac{\text{PSN}(y;x_0,x_1)}{\mathbb{P}(Y < y|X={x_0})},
\end{aligned}
\end{equation}
$\mathbb{P}(y\leq Y|X={x_1})=1-\rho(y;x_1)$ and $\mathbb{P}(Y < y|X={x_0})=\rho(y;x_0)$ for any  $x_0,x_1 \in \Omega_X$ and $y \in \Omega_Y$.
\end{proof}

\subsection{Proofs in Section \ref{sec4}}

%{\bf Lemma \ref{LEM2}}
%{\it  Under SCM ${\cal M}_{T}$, ${\boldsymbol X}$ are independent of ${\boldsymbol U}$ given ${\boldsymbol C}={\boldsymbol c}$ for all ${\boldsymbol c} \in \Omega_{\boldsymbol C}$, and $\mathbb{P}({\boldsymbol Y}_{\boldsymbol x} \preceq {\boldsymbol y}|{\boldsymbol C}={\boldsymbol c})=\mathbb{P}({\boldsymbol Y} \preceq {\boldsymbol y}|{\boldsymbol X}={\boldsymbol x},{\boldsymbol C}={\boldsymbol c})$ holds for any ${\boldsymbol c} \in \Omega_{\boldsymbol C}$.}

%\begin{proof}
%    We have 
%    \begin{equation}
%        \mathbb{P}({\boldsymbol Y}_{\boldsymbol x} \preceq {\boldsymbol y}|{\boldsymbol C}={\boldsymbol c})=\mathbb{P}({\boldsymbol Y}_{\boldsymbol x} \preceq {\boldsymbol y}|{\boldsymbol X}={\boldsymbol x},{\boldsymbol C}={\boldsymbol c})=\mathbb{P}({\boldsymbol Y} \preceq {\boldsymbol y}|{\boldsymbol X}={\boldsymbol x},{\boldsymbol C}={\boldsymbol c})
%    \end{equation}
%    for any ${\boldsymbol c} \in \Omega_{\boldsymbol C}$
%   since ${\boldsymbol X} \indep {\boldsymbol U}|{\boldsymbol C}$.
%\end{proof}

{\bf Theorem \ref{prop1}.}
{\it Under SCM ${\cal M}_{T}$ and Assumption \ref{SUP2}, 
Assumptions \ref{MONO2} and \ref{AS2} are equivalent, and 
%Assumptions \ref{SAS2} and \ref{RP2} are equivalent. 
{Assumptions \ref{SAS2} is a strictly stronger requirement than \ref{AS2}.}

\begin{proof}
We show the proof of equivalence of assumptions.

(Assumption \ref{MONO2} $\Rightarrow$ Assumption \ref{AS2}.)
For any ${\boldsymbol c} \in \Omega_{\boldsymbol C}$, from Assumption \ref{SUP2}, if we have the negation of Assumption \ref{AS2}
    \begin{center}
        there exists a set ${\cal U}$ such that $0<\mathbb{P}({\cal U})<1$, and $$f_{\boldsymbol Y}({\boldsymbol x}_0,{\boldsymbol c},{\boldsymbol u}_0)\succeq {\boldsymbol y}\succ f_{\boldsymbol Y}({\boldsymbol x}_0,{\boldsymbol c},{\boldsymbol u}_1) \land f_{\boldsymbol Y}({\boldsymbol x}_1,{\boldsymbol c},{\boldsymbol u}_0)\prec y \preceq f_{\boldsymbol Y}({\boldsymbol x}_1,{\boldsymbol c},{\boldsymbol u}_1)$$ for some ${\boldsymbol x}_0, {\boldsymbol x}_1 \in \Omega_{\boldsymbol X}$ and ${\boldsymbol y} \in \Omega_{\boldsymbol Y}$ and for any ${\boldsymbol u}_0,{\boldsymbol u}_1 \in {\cal U}$ such that ${\boldsymbol u}_0\preceq {\boldsymbol u}_1$,
    \end{center}
    then we have
    \begin{center}
        $f_{\boldsymbol Y}({\boldsymbol x}_0,{\boldsymbol c},{\boldsymbol u}_0)\succeq {\boldsymbol y} \succ f_{\boldsymbol Y}({\boldsymbol x}_1,{\boldsymbol u}_0)$ and $f_{\boldsymbol Y}({\boldsymbol x}_0,{\boldsymbol c},{\boldsymbol u}_1)\prec {\boldsymbol y} \preceq f_{\boldsymbol Y}({\boldsymbol x}_1,{\boldsymbol c},{\boldsymbol u}_1)$ for some ${\boldsymbol x}_0, {\boldsymbol x}_1 \in \Omega_{\boldsymbol X}$ and ${\boldsymbol y} \in \Omega_{\boldsymbol Y}$ and\\ for any ${\boldsymbol u}_0, {\boldsymbol u}_1 \in {\cal U}$ such that ${\boldsymbol u}_0 \preceq {\boldsymbol u}_1$,
    \end{center}
    and we also have 
    \begin{center}
        $f_{\boldsymbol Y}({\boldsymbol x}_0,{\boldsymbol c},{\boldsymbol u})\succeq {\boldsymbol y} \succ f_{\boldsymbol Y}({\boldsymbol x}_1,{\boldsymbol c},{\boldsymbol u})$ and $f_{\boldsymbol Y}({\boldsymbol x}_0,{\boldsymbol c},{\boldsymbol u})\prec {\boldsymbol y} \preceq f_{\boldsymbol Y}({\boldsymbol x}_1,{\boldsymbol c},{\boldsymbol u})$ for some ${\boldsymbol x}_0, {\boldsymbol x}_1 \in \Omega_{\boldsymbol X}$ and ${\boldsymbol y} \in \Omega_{\boldsymbol Y}$ and  for any ${\boldsymbol u} \in {\cal U}$.
    \end{center}
    This implies the negation of Assumption \ref{MONO_A} $\mathbb{P}({\boldsymbol Y}_{{\boldsymbol x}_0}\prec {\boldsymbol y} \preceq {\boldsymbol Y}_{{\boldsymbol x}_1}|{\boldsymbol C}={\boldsymbol c})\ne 0$ and $\mathbb{P}({\boldsymbol Y}_{{\boldsymbol x}_1}\prec {\boldsymbol y} \preceq {\boldsymbol Y}_{{\boldsymbol x}_0}|{\boldsymbol C}={\boldsymbol c})\ne 0$ for some ${\boldsymbol x}_0, {\boldsymbol x}_1 \in \Omega_{\boldsymbol Y}$ and ${\boldsymbol y} \in \Omega_{\boldsymbol Y}$ since $f_{\boldsymbol Y}({\boldsymbol x}_0,{\boldsymbol c},{\boldsymbol u})\succeq {\boldsymbol y} \succ f_{\boldsymbol Y}({\boldsymbol x}_1,{\boldsymbol c},{\boldsymbol u}) \Leftrightarrow {\boldsymbol Y}_{{\boldsymbol x}_0}({\boldsymbol c},{\boldsymbol u})\succeq {\boldsymbol y} \succ {\boldsymbol Y}_{{\boldsymbol x}_1}({\boldsymbol c},{\boldsymbol u})$ and $f_{\boldsymbol Y}({\boldsymbol x}_0,{\boldsymbol c},{\boldsymbol u})\prec {\boldsymbol y} \preceq f_{\boldsymbol Y}({\boldsymbol x}_1,{\boldsymbol c},{\boldsymbol u}) \Leftrightarrow {\boldsymbol Y}_{{\boldsymbol x}_1}({\boldsymbol c},{\boldsymbol u})\succeq {\boldsymbol y} \succ {\boldsymbol Y}_{{\boldsymbol x}_0}({\boldsymbol c},{\boldsymbol u})$.

    (Assumption \ref{AS2} $\Rightarrow$ Assumption \ref{MONO2}.)
For any ${\boldsymbol c} \in \Omega_{\boldsymbol C}$, we denote ${\boldsymbol u}_{sup}=\sup\{{\boldsymbol u}:f_{\boldsymbol Y}({\boldsymbol x}_0,{\boldsymbol c},{\boldsymbol u})\preceq {\boldsymbol y}\}$. 
We consider the situations ``the function $f_{\boldsymbol Y}({\boldsymbol x},{\boldsymbol c},{\boldsymbol U})$ is monotonic increasing on ${\boldsymbol U}$'' and ``the function $f_{\boldsymbol Y}({\boldsymbol x},{\boldsymbol c},{\boldsymbol U})$ is monotonic decreasing on ${\boldsymbol U}$'', separately.

{\bf (1).}
If the function $f_{\boldsymbol Y}({\boldsymbol x},{\boldsymbol c},{\boldsymbol U})$ is {\bf monotonic increasing} on ${\boldsymbol U}$
for all ${\boldsymbol x} \in \Omega_{\boldsymbol X}$ almost surely w.r.t. $\mathbb{P}_{\boldsymbol U}$, we have
\begin{equation}
    f_{\boldsymbol Y}({\boldsymbol x}_0,{\boldsymbol c},{\boldsymbol u}_{sup}) \preceq f_{\boldsymbol Y}({\boldsymbol x}_0,{\boldsymbol c},{\boldsymbol u}) \text{ and }f_{\boldsymbol Y}({\boldsymbol x}_1,{\boldsymbol c},{\boldsymbol u}_{sup}) \preceq f_{\boldsymbol Y}({\boldsymbol x}_1,{\boldsymbol c},{\boldsymbol u})
\end{equation}
for $\mathbb{P}_{\boldsymbol U}$-almost every ${\boldsymbol u} \in \Omega_{\boldsymbol U}$ such that ${\boldsymbol u}\succeq {\boldsymbol u}_{sup}$.
We have the following statements:
\begin{enumerate}
    \item Supposed $f_{\boldsymbol Y}({\boldsymbol x}_0,{\boldsymbol c},{\boldsymbol u}_{sup})\succ f_{\boldsymbol Y}({\boldsymbol x}_1,{\boldsymbol c},{\boldsymbol u}_{sup})$, 
we have ${\boldsymbol y}= f_{\boldsymbol Y}({\boldsymbol x}_0,{\boldsymbol c},{\boldsymbol u}_{sup}) \succ f_{\boldsymbol Y}({\boldsymbol x}_1,{\boldsymbol c},{\boldsymbol u}_{sup})\succeq f_{\boldsymbol Y}({\boldsymbol x}_1,{\boldsymbol c},{\boldsymbol u})=Y_{{\boldsymbol x}_1}({\boldsymbol c},{\boldsymbol u})$ for $\mathbb{P}_{\boldsymbol U}$-almost every ${\boldsymbol u} \in \Omega_{\boldsymbol U}$ such that $f_{\boldsymbol Y}({\boldsymbol x}_0,{\boldsymbol c},{\boldsymbol u})\prec {\boldsymbol y}$.
It means ${\boldsymbol Y}_{{\boldsymbol x}_0}({\boldsymbol c},{\boldsymbol u})\prec {\boldsymbol y} \Rightarrow {\boldsymbol Y}_{{\boldsymbol x}_1}({\boldsymbol c},{\boldsymbol u}) \prec {\boldsymbol y}$ for $\mathbb{P}_{\boldsymbol U}$-almost every ${\boldsymbol u} \in \Omega_{\boldsymbol U}$ and  $\mathbb{P}({\boldsymbol Y}_{{\boldsymbol x}_0}\prec {\boldsymbol y} \preceq {\boldsymbol Y}_{{\boldsymbol x}_1}|{\boldsymbol C}={\boldsymbol c})=0$.
\item Supposed $f_{\boldsymbol Y}({\boldsymbol x}_0,{\boldsymbol c},{\boldsymbol u}_{sup})\preceq f_{\boldsymbol Y}({\boldsymbol x}_1,{\boldsymbol c},{\boldsymbol u}_{sup})$, we have $f_{\boldsymbol Y}({\boldsymbol x}_1,{\boldsymbol c},{\boldsymbol u}) \succeq f_{\boldsymbol Y}({\boldsymbol x}_1,{\boldsymbol c},{\boldsymbol u}_{sup}) \succeq f_{\boldsymbol Y}({\boldsymbol x}_0,{\boldsymbol c},{\boldsymbol u}_{sup}) ={\boldsymbol y}$ for $\mathbb{P}_{\boldsymbol U}$-almost every ${\boldsymbol u} \in \Omega_{\boldsymbol U}$ such that $f_{\boldsymbol Y}({\boldsymbol x}_0,{\boldsymbol c},{\boldsymbol u})\succeq {\boldsymbol y}$.
It means ${\boldsymbol Y}_{{\boldsymbol x}_0}({\boldsymbol c},{\boldsymbol u})\succeq {\boldsymbol y} \Rightarrow {\boldsymbol Y}_{{\boldsymbol x}_1}({\boldsymbol c},{\boldsymbol u}) \succeq {\boldsymbol y}$ for $\mathbb{P}_{\boldsymbol U}$-almost every ${\boldsymbol u} \in \Omega_{\boldsymbol U}$ and  $\mathbb{P}({\boldsymbol Y}_{{\boldsymbol x}_1}\prec {\boldsymbol y} \preceq {\boldsymbol Y}_{{\boldsymbol x}_0}|{\boldsymbol C}={\boldsymbol c})=0$.
\end{enumerate}
Then, these imply Assumption \ref{AS1}.

{\bf (2).}
If the function $f_{\boldsymbol Y}({\boldsymbol x},{\boldsymbol c},{\boldsymbol U})$ is {\bf monotonic decreasing} on ${\boldsymbol U}$
for all ${\boldsymbol x} \in \Omega_{\boldsymbol X}$ almost surely w.r.t. $\mathbb{P}_{\boldsymbol U}$, we have
\begin{equation}
    f_{\boldsymbol Y}({\boldsymbol x}_0,{\boldsymbol c},{\boldsymbol u}_{sup}) \succeq f_{\boldsymbol Y}({\boldsymbol x}_0,{\boldsymbol c},{\boldsymbol u}) \text{ and }f_{\boldsymbol Y}({\boldsymbol x}_1,{\boldsymbol c},{\boldsymbol u}_{sup}) \succeq f_{\boldsymbol Y}({\boldsymbol x}_1,{\boldsymbol c},{\boldsymbol u})
\end{equation}
for $\mathbb{P}_{\boldsymbol U}$-almost every ${\boldsymbol u} \in \Omega_{\boldsymbol U}$ such that ${\boldsymbol u}\succeq {\boldsymbol u}_{sup}$.
We have the following statements:
\begin{enumerate}
    \item Supposed $f_{\boldsymbol Y}({\boldsymbol x}_0,{\boldsymbol c},{\boldsymbol u}_{sup})\preceq f_{\boldsymbol Y}({\boldsymbol x}_1,{\boldsymbol c},{\boldsymbol u}_{sup})$, 
we have ${\boldsymbol y}= f_{\boldsymbol Y}({\boldsymbol x}_0,{\boldsymbol c},{\boldsymbol u}_{sup}) \preceq f_{\boldsymbol Y}({\boldsymbol x}_1,{\boldsymbol c},{\boldsymbol u}_{sup})\preceq f_{\boldsymbol Y}({\boldsymbol x}_1,{\boldsymbol c},{\boldsymbol u})={\boldsymbol Y}_{{\boldsymbol x}_1}({\boldsymbol c},{\boldsymbol u})$ for $\mathbb{P}_{\boldsymbol U}$-almost every ${\boldsymbol u} \in \Omega_{\boldsymbol U}$ such that $f_{\boldsymbol Y}({\boldsymbol x}_0,{\boldsymbol c},{\boldsymbol u})\succeq {\boldsymbol y}$.
It means ${\boldsymbol Y}_{{\boldsymbol x}_0}({\boldsymbol c},{\boldsymbol u})\succeq {\boldsymbol y} \Rightarrow {\boldsymbol Y}_{{\boldsymbol x}_1}({\boldsymbol u}) \succeq {\boldsymbol y}$ for $\mathbb{P}_{\boldsymbol U}$-almost every ${\boldsymbol u} \in \Omega_{\boldsymbol U}$ and  $\mathbb{P}({\boldsymbol Y}_{{\boldsymbol x}_1}\prec {\boldsymbol y} \preceq {\boldsymbol Y}_{{\boldsymbol x}_0}|{\boldsymbol C}={\boldsymbol c})=0$.
\item Supposed $f_{\boldsymbol Y}({\boldsymbol x}_0,{\boldsymbol c},{\boldsymbol u}_{sup})\succ f_{\boldsymbol Y}({\boldsymbol x}_1,{\boldsymbol c},{\boldsymbol u}_{sup})$, we have $f_{\boldsymbol Y}({\boldsymbol x}_1,{\boldsymbol c},{\boldsymbol u}) \prec f_{\boldsymbol Y}({\boldsymbol x}_1,{\boldsymbol c},{\boldsymbol u}_{sup}) \preceq f_{\boldsymbol Y}({\boldsymbol x}_0,{\boldsymbol c},{\boldsymbol u}_{sup}) ={\boldsymbol y}$ for $\mathbb{P}_{\boldsymbol U}$-almost every ${\boldsymbol u} \in \Omega_{\boldsymbol U}$ such that $f_{\boldsymbol Y}({\boldsymbol x}_0,{\boldsymbol c},{\boldsymbol u})\prec {\boldsymbol y}$.
It means ${\boldsymbol Y}_{{\boldsymbol x}_0}({\boldsymbol c},{\boldsymbol u})\prec {\boldsymbol y} \Rightarrow {\boldsymbol Y}_{{\boldsymbol x}_1}({\boldsymbol c},{\boldsymbol u}) \prec {\boldsymbol y}$ for $\mathbb{P}_{\boldsymbol U}$-almost every ${\boldsymbol u} \in \Omega_{\boldsymbol U}$ and  $\mathbb{P}({\boldsymbol Y}_{{\boldsymbol x}_0}\prec {\boldsymbol y} \preceq {\boldsymbol Y}_{{\boldsymbol x}_1}|{\boldsymbol C}={\boldsymbol c})=0$.
\end{enumerate}
Then, Assumption \ref{AS2} holds.
In conclusion, Assumption \ref{AS2} implies Assumption \ref{MONO2}.

\end{proof}

{\bf Theorem \ref{THEO41}.}(Identification of conditional PoC)
{\it
Under SCM ${\cal M}_{T}$ and}  
Assumptions \ref{ASEXO2}, \ref{MONO2} (or \ref{AS2}, \ref{SAS2}), and \ref{SUP2}, 
PNS, PN, and PS are identifiable by
\begin{equation}
    \begin{aligned}
    &\text{PNS}({\boldsymbol y};{\boldsymbol x}_0,{\boldsymbol x}_1,{\boldsymbol c})=\max\{\rho({\boldsymbol y};{\boldsymbol x}_0,{\boldsymbol c})-\rho({\boldsymbol y};{\boldsymbol x}_1,{\boldsymbol c}),0\},\\
    &\text{PN}({\boldsymbol y};{\boldsymbol x}_0,{\boldsymbol x}_1,{\boldsymbol c})=\max\left\{\frac{\rho({\boldsymbol y};{\boldsymbol x}_0,{\boldsymbol c})-\rho({\boldsymbol y};{\boldsymbol x}_1,{\boldsymbol c})}{1-\rho({\boldsymbol y};{\boldsymbol x}_1,{\boldsymbol c})},0\right\},\\
    &\text{PS}({\boldsymbol y};{\boldsymbol x}_0,{\boldsymbol x}_1,{\boldsymbol c})=\max\left\{\frac{\rho({\boldsymbol y};{\boldsymbol x}_0,{\boldsymbol c})-\rho({\boldsymbol y};{\boldsymbol x}_1,{\boldsymbol c})}{\rho({\boldsymbol y};{\boldsymbol x}_0,{\boldsymbol c})},0\right\}
    \end{aligned}
\end{equation}
for any ${\boldsymbol x}_0,{\boldsymbol x}_1 \in \Omega_{\boldsymbol X}$, ${\boldsymbol c} \in \Omega_{\boldsymbol C}$, and ${\boldsymbol y} \in \Omega_{\boldsymbol Y}$ such that $\rho({\boldsymbol y};{\boldsymbol x}_1,{\boldsymbol c})<1$ and $\rho({\boldsymbol y};{\boldsymbol x}_0,{\boldsymbol c})>0$.
}

\begin{proof}
Under Assumptions \ref{ASEXO2} and \ref{MONO2} (or \ref{AS2}), 
\begin{equation}
\begin{aligned}
\text{PNS}({\boldsymbol y};{\boldsymbol x}_0,{\boldsymbol x}_1,{\boldsymbol c})
&=\mathbb{P}({\boldsymbol Y}_{{\boldsymbol x}_0}\prec {\boldsymbol y} \preceq {\boldsymbol Y}_{{\boldsymbol x}_1}|{\boldsymbol C}={\boldsymbol c})\\
&=\mathbb{P}({\boldsymbol u}_{\rho({\boldsymbol y};{\boldsymbol x}_0,{\boldsymbol c})}\preceq {\boldsymbol u} \prec {\boldsymbol u}_{\rho({\boldsymbol y};{\boldsymbol x}_1,{\boldsymbol c})})\\
&=\max\{\rho({\boldsymbol y};{\boldsymbol x}_0,{\boldsymbol c})-\rho({\boldsymbol y};{\boldsymbol x}_1,{\boldsymbol c}),0\}
    \end{aligned}
\end{equation}
for any ${\boldsymbol x}_0,{\boldsymbol x}_1 \in \Omega_{\boldsymbol X}$, ${\boldsymbol c} \in \Omega_{\boldsymbol C}$ and ${\boldsymbol y} \in \Omega_{\boldsymbol Y}$,
where ${\boldsymbol u}_{\rho({\boldsymbol y};{\boldsymbol x}_0,{\boldsymbol c})}=\sup\{{\boldsymbol u}:f_{\boldsymbol Y}({\boldsymbol x}_0,{\boldsymbol c},{\boldsymbol u})\prec {\boldsymbol y}\}$ and ${\boldsymbol u}_{\rho({\boldsymbol y};{\boldsymbol x}_1,{\boldsymbol c})}=\sup\{{\boldsymbol u}:f_{\boldsymbol Y}({\boldsymbol x}_1,{\boldsymbol c},{\boldsymbol u})\prec {\boldsymbol y}\}$.
%=\inf\{u:y < f_Y(x_1,u)\}$.

\begin{comment}
On the other hands, under Assumptions \ref{ASEXO} and \ref{MONO_A}, 
\begin{equation}
\begin{aligned}
\text{PNS}({\boldsymbol y};{\boldsymbol x}_0,{\boldsymbol x}_1)
&=\mathbb{P}({\boldsymbol Y}_{{\boldsymbol x}_0} \prec {\boldsymbol y} \preceq {\boldsymbol Y}_{{\boldsymbol x}_1})\\
&=\mathbb{P}({\boldsymbol Y}_{{\boldsymbol x}_0} \prec {\boldsymbol y})-\mathbb{P}({\boldsymbol Y}_{{\boldsymbol x}_1} \prec {\boldsymbol y})+\mathbb{P}({\boldsymbol Y}_{{\boldsymbol x}_1} \prec {\boldsymbol y} \preceq {\boldsymbol Y}_{{\boldsymbol x}_0})\\
&=\max\{\rho({\boldsymbol y};{\boldsymbol x}_0)-\rho({\boldsymbol y};{\boldsymbol x}_1),0\}.
    \end{aligned}
\end{equation}
\end{comment}
$\text{PN}({\boldsymbol y};{\boldsymbol x}_0,{\boldsymbol x}_1,{\boldsymbol c})$ and $\text{PS}({\boldsymbol y};{\boldsymbol x}_0,{\boldsymbol x}_1,{\boldsymbol c})$ are given from the following relationship:
\begin{equation}
\begin{aligned}
\text{PN}({\boldsymbol y};{\boldsymbol x}_0,{\boldsymbol x}_1,{\boldsymbol c})=\frac{\text{PNS}({\boldsymbol y};{\boldsymbol x}_0,{\boldsymbol x}_1,{\boldsymbol c})}{\mathbb{P}({\boldsymbol y}\preceq{\boldsymbol Y}|{\boldsymbol X}={\boldsymbol x}_1,{\boldsymbol C}={\boldsymbol c})},\ \  \text{PS}({\boldsymbol y};{\boldsymbol x}_0,{\boldsymbol x}_1,{\boldsymbol c})=\frac{\text{PNS}({\boldsymbol y};{\boldsymbol x}_0,{\boldsymbol x}_1,{\boldsymbol c})}{\mathbb{P}({\boldsymbol Y}\prec {\boldsymbol y} |{\boldsymbol X}={\boldsymbol x}_0,{\boldsymbol C}={\boldsymbol c})},
\end{aligned}
\end{equation}
and $\mathbb{P}({\boldsymbol y}\preceq{\boldsymbol Y}|{\boldsymbol X}={\boldsymbol x}_1,{\boldsymbol C}={\boldsymbol c})=1-\rho({\boldsymbol y};{\boldsymbol x}_1,{\boldsymbol c})$ and $\mathbb{P}({\boldsymbol y}\prec{\boldsymbol Y}|{\boldsymbol X}={\boldsymbol x}_0,{\boldsymbol C}={\boldsymbol c})=\rho({\boldsymbol y};{\boldsymbol x}_0,{\boldsymbol c})$ for any ${\boldsymbol x}_0,{\boldsymbol x}_1 \in \Omega_{\boldsymbol X}$, ${\boldsymbol c} \in \Omega_{\boldsymbol C}$ and ${\boldsymbol y} \in \Omega_{\boldsymbol Y}$.
\end{proof}

\subsection{Proofs in Section \ref{sec5}}
{\bf Theorem \ref{THE51}.} (Identification of conditional  PNS with evidence $({\boldsymbol y}',{\boldsymbol x}',{\boldsymbol c})$)
{\it 
Under SCM ${\cal M}_{T}$ and Assumptions \ref{ASEXO2}, \ref{MONO2} (or \ref{AS2}, \ref{SAS2}), and \ref{SUP2}, we have %the two following statements:

{\bf (A).} If $\rho({\boldsymbol y}';{\boldsymbol x}',{\boldsymbol c})\ne\rho^{o}({\boldsymbol y}';{\boldsymbol x}',{\boldsymbol c})$, then we have
\begin{equation}
    \text{PNS}({\boldsymbol y};{\boldsymbol x}_0,{\boldsymbol x}_1,{\boldsymbol y}',{\boldsymbol x}',{\boldsymbol c})=\max\{{\alpha}/{\beta},0\},
\end{equation}
where 
\begin{equation}
\begin{aligned}
&\alpha=\min\{\rho({\boldsymbol y};{\boldsymbol x}_0,{\boldsymbol c}),\rho^{o}({\boldsymbol y}';{\boldsymbol x}',{\boldsymbol c})\}-\max\{\rho({\boldsymbol y};{\boldsymbol x}_1,{\boldsymbol c}),\rho({\boldsymbol y}';{\boldsymbol x}',{\boldsymbol c})\},\\
%\end{aligned}
%\end{equation}
%\begin{equation}
&\beta=\rho^{o}({\boldsymbol y}';{\boldsymbol x}',{\boldsymbol c})-\rho({\boldsymbol y}';{\boldsymbol x}',{\boldsymbol c})
\end{aligned}
\end{equation}
for any ${\boldsymbol x}_0,{\boldsymbol x}_1, {\boldsymbol x}' \in \Omega_{\boldsymbol X}$, ${\boldsymbol c} \in \Omega_{\boldsymbol C}$, ${\boldsymbol y}' \in \Omega_{\boldsymbol Y}$, and ${\boldsymbol y} \in \Omega_{\boldsymbol Y}$.

{\bf (B).} If $\rho({\boldsymbol y}';{\boldsymbol x}',{\boldsymbol c})=\rho^{o}({\boldsymbol y}';{\boldsymbol x}',{\boldsymbol c})$, then we have
\begin{equation}
\begin{aligned}
&\text{PNS}({\boldsymbol y};{\boldsymbol x}_0,{\boldsymbol x}_1,{\boldsymbol y}',{\boldsymbol x}',{\boldsymbol c})=\mathbb{I}(\rho({\boldsymbol y};{\boldsymbol x}_1,{\boldsymbol c}) \leq 
%\rho^{o}({\boldsymbol y}';{\boldsymbol x}',{\boldsymbol c})=
\rho({\boldsymbol y}';{\boldsymbol x}',{\boldsymbol c})< \rho({\boldsymbol y};{\boldsymbol x}_0,{\boldsymbol c}))
\end{aligned}
\end{equation}
    for any ${\boldsymbol x}_0,{\boldsymbol x}_1, {\boldsymbol x}' \in \Omega_{\boldsymbol X}$, ${\boldsymbol c} \in \Omega_{\boldsymbol C}$, ${\boldsymbol y}' \in \Omega_{\boldsymbol Y}$, and ${\boldsymbol y} \in \Omega_{\boldsymbol Y}$.
}

\begin{proof}
    Under Assumptions \ref{ASEXO2} and \ref{AS2}, if $\rho({\boldsymbol y}';{\boldsymbol x}',{\boldsymbol c})\ne\rho^{o}({\boldsymbol y}';{\boldsymbol x}',{\boldsymbol c})$, we have
    \begin{equation}
\begin{aligned}
        &\mathbb{P}({\boldsymbol Y}_{{\boldsymbol x}_0}\prec {\boldsymbol y} \preceq {\boldsymbol Y}_{{\boldsymbol x}_1}|{\boldsymbol Y}={\boldsymbol y}',{\boldsymbol X}={\boldsymbol x}',{\boldsymbol C}={\boldsymbol c})\\
        &=\frac{\mathbb{P}({\boldsymbol Y}_{{\boldsymbol x}_0}\prec {\boldsymbol y} \preceq {\boldsymbol Y}_{{\boldsymbol x}_1},{\boldsymbol Y}={\boldsymbol y}',{\boldsymbol X}={\boldsymbol x}'|{\boldsymbol C}={\boldsymbol c})}{\mathbb{P}({\boldsymbol Y}={\boldsymbol y}',{\boldsymbol X}={\boldsymbol x}'|{\boldsymbol C}={\boldsymbol c})}\\
         %&=\frac{\mathbb{P}({\boldsymbol Y}_{{\boldsymbol x}_0} \prec {\boldsymbol y} \preceq{\boldsymbol Y}_{{\boldsymbol x}_1},{\boldsymbol Y}_{{\boldsymbol x}'}={\boldsymbol y}'|{\boldsymbol C}={\boldsymbol c})}{\mathbb{P}({\boldsymbol Y}_{{\boldsymbol x}'}={\boldsymbol y}'|{\boldsymbol C}={\boldsymbol c})}\\
        &=\frac{\mathbb{P}({\boldsymbol u}_{\rho({\boldsymbol y};{\boldsymbol x}_1,{\boldsymbol c})} \preceq{\boldsymbol u}\prec {\boldsymbol u}_{\rho({\boldsymbol y};{\boldsymbol x}_0,{\boldsymbol c})} ,{\boldsymbol u}_{\rho({\boldsymbol y}';{\boldsymbol x}',{\boldsymbol c})} \preceq {\boldsymbol u}\prec{\boldsymbol u}_{\rho^{o}({\boldsymbol y}';{\boldsymbol x}',{\boldsymbol c})})}{\mathbb{P}({\boldsymbol u}_{\rho({\boldsymbol y}';{\boldsymbol x}',{\boldsymbol c})} \prec {\boldsymbol u}\preceq{\boldsymbol u}_{\rho^o({\boldsymbol y}';{\boldsymbol x}',{\boldsymbol c})})}\\
        &
        =\frac{\max\{\min\{\rho({\boldsymbol y};{\boldsymbol x}_0,{\boldsymbol c}),\rho^{o}({\boldsymbol y}';{\boldsymbol x}',{\boldsymbol c})\}-\max\{\rho({\boldsymbol y};{\boldsymbol x}_1,{\boldsymbol c}),\rho({\boldsymbol y}';{\boldsymbol x}',{\boldsymbol c})\},0\}}{\rho({\boldsymbol y}';{\boldsymbol x}',{\boldsymbol c})-\rho^{o}({\boldsymbol y}';{\boldsymbol x}',{\boldsymbol c})}\\
        &=\max\left\{\frac{\min\{\rho({\boldsymbol y};{\boldsymbol x}_0,{\boldsymbol c}),\rho^{o}({\boldsymbol y}';{\boldsymbol x}',{\boldsymbol c})\}-\max\{\rho({\boldsymbol y};{\boldsymbol x}_1,{\boldsymbol c}),\rho({\boldsymbol y}';{\boldsymbol x}',{\boldsymbol c})\}}{\rho^{o}({\boldsymbol y}';{\boldsymbol x}',{\boldsymbol c})-\rho({\boldsymbol y}';{\boldsymbol x}',{\boldsymbol c})},0\right\}.
    \end{aligned}
\end{equation}
for any ${\boldsymbol x}_0,{\boldsymbol x}_1, {\boldsymbol x}' \in \Omega_{\boldsymbol X}$, ${\boldsymbol c} \in \Omega_{\boldsymbol C}$, ${\boldsymbol y}' \in \Omega_{\boldsymbol Y}$ and ${\boldsymbol y} \in \Omega_{\boldsymbol Y}$.
This represents the statement (A).
Otherwise, since ${\boldsymbol u}_{\rho({\boldsymbol y}';{\boldsymbol x}',{\boldsymbol c})}={\boldsymbol u}_{\rho^o({\boldsymbol y}';{\boldsymbol x}',{\boldsymbol c})}$, we have
\begin{equation}
    \begin{aligned}
        &\mathbb{P}({\boldsymbol Y}_{{\boldsymbol x}_0}\prec {\boldsymbol y} \preceq {\boldsymbol Y}_{{\boldsymbol x}_1}|{\boldsymbol Y}={\boldsymbol y}',{\boldsymbol X}={\boldsymbol x}',{\boldsymbol C}={\boldsymbol c})\\
        &={\mathbb{P}({\boldsymbol u}_{\rho({\boldsymbol y};{\boldsymbol x}_1,{\boldsymbol c})} \preceq{\boldsymbol u}\prec {\boldsymbol u}_{\rho({\boldsymbol y};{\boldsymbol x}_0,{\boldsymbol c})}| {\boldsymbol u}={\boldsymbol u}_{\rho({\boldsymbol y}';{\boldsymbol x}',{\boldsymbol c})})}\\
        &=\mathbb{I}({\boldsymbol u}_{\rho({\boldsymbol y};{\boldsymbol x}_1,{\boldsymbol c})} \preceq{\boldsymbol u}_{\rho({\boldsymbol y}';{\boldsymbol x}',{\boldsymbol c})}\prec {\boldsymbol u}_{\rho({\boldsymbol y};{\boldsymbol x}_0,{\boldsymbol c})})\\
        &=\mathbb{I}(\rho({\boldsymbol y};{\boldsymbol x}_1,{\boldsymbol c}) \leq 
%\rho^{o}({\boldsymbol y}';{\boldsymbol x}',{\boldsymbol c})=
\rho({\boldsymbol y}';{\boldsymbol x}',{\boldsymbol c})< \rho({\boldsymbol y};{\boldsymbol x}_0,{\boldsymbol c}))
    \end{aligned}
\end{equation}
for any ${\boldsymbol x}_0,{\boldsymbol x}_1, {\boldsymbol x}' \in \Omega_{\boldsymbol X}$, ${\boldsymbol c} \in \Omega_{\boldsymbol C}$, ${\boldsymbol y}' \in \Omega_{\boldsymbol Y}$ and ${\boldsymbol y} \in \Omega_{\boldsymbol Y}$.
This represents the statement (B).
\end{proof}

{\bf Theorem \ref{THE52}.} (Identification of conditional PNS with multi-hypothetical terms)
{\it 
Under SCM ${\cal M}_{T}$ and Assumptions \ref{ASEXO2}, \ref{MONO2} (or \ref{AS2}, \ref{SAS2}), and \ref{SUP2}, 
%PoC with multi-hypothetical terms
$\text{PNS}(\overline{{\boldsymbol y}};\overline{{\boldsymbol x}},{\boldsymbol c})$ is identifiable by
\begin{equation}
\begin{aligned}
\text{PNS}(\overline{{\boldsymbol y}};\overline{{\boldsymbol x}},{\boldsymbol c})=\max\Big\{\min_{p=1,\ldots,P}\{\rho({\boldsymbol y}_{p};{\boldsymbol x}_{p-1},{\boldsymbol c})\}-\max_{p=1,\ldots,P}\{\rho({\boldsymbol y}_{p};{\boldsymbol x}_p,{\boldsymbol c})\},0\Big\}
\end{aligned}
\end{equation}
for any $\overline{{\boldsymbol x}}=({\boldsymbol x}_0,{\boldsymbol x}_1,\ldots,{\boldsymbol x}_P) \in \Omega_{\boldsymbol X}^{P+1}$, $\overline{{\boldsymbol y}}=({\boldsymbol y}_1,\ldots,{\boldsymbol y}_P)\in \Omega_{\boldsymbol Y}^P$, and ${\boldsymbol c} \in \Omega_{\boldsymbol C}$.}

\begin{proof}
    Under Assumptions \ref{ASEXO2} and \ref{AS2}, we have 
    \begin{equation}
    \begin{aligned}
        &\text{PNS}(\overline{{\boldsymbol y}};\overline{{\boldsymbol x}},{\boldsymbol c})\\
        &=\mathbb{P}({\boldsymbol Y}_{{\boldsymbol x}_0} \prec {\boldsymbol y}_1 \preceq{\boldsymbol Y}_{{\boldsymbol x}_1}, {\boldsymbol Y}_{{\boldsymbol x}_1}\prec{\boldsymbol y}_2 \preceq {\boldsymbol Y}_{{\boldsymbol x}_2},\ldots, {\boldsymbol Y}_{{\boldsymbol x}_{P-1}} \prec {\boldsymbol y}_P\preceq{\boldsymbol Y}_{{\boldsymbol x}_P}|{\boldsymbol C}={\boldsymbol c})\\
        &=\mathbb{P}({\boldsymbol u}_{\rho({\boldsymbol y}_{1};{\boldsymbol x}_0,{\boldsymbol c})} \preceq {\boldsymbol u}\prec{\boldsymbol u}_{\rho({\boldsymbol y}_{1};{\boldsymbol x}_1,{\boldsymbol c})},{\boldsymbol u}_{\rho({\boldsymbol y}_{2};{\boldsymbol x}_1,{\boldsymbol c})} \preceq {\boldsymbol u}\prec{\boldsymbol u}_{\rho({\boldsymbol y}_{2};{\boldsymbol x}_2,{\boldsymbol c})},\ldots,{\boldsymbol u}_{\rho({\boldsymbol y}_{P};{\boldsymbol x}_{P},{\boldsymbol c})} \preceq{\boldsymbol u}\prec{\boldsymbol u}_{\rho({\boldsymbol y}_{P};{\boldsymbol x}_{P-1},{\boldsymbol c})})\\
        &=\mathbb{P}({\boldsymbol u}_{\max_p\{\rho({\boldsymbol y}_{P};{\boldsymbol x}_{P},{\boldsymbol c})\}} \prec{\boldsymbol u}\preceq{\boldsymbol u}_{\min_p\{\rho({\boldsymbol y}_{P};{\boldsymbol x}_{P-1},{\boldsymbol c})\}})\\
        &=\max\left\{\min_p\{\rho({\boldsymbol y}_{P};{\boldsymbol x}_{P-1},{\boldsymbol c})\}-\max_p\{\rho({\boldsymbol y}_{P};{\boldsymbol x}_{P},{\boldsymbol c})\},0\right\}
        \end{aligned}
        \end{equation}
for any $\overline{{\boldsymbol x}}=({\boldsymbol x}_0,{\boldsymbol x}_1,\ldots,{\boldsymbol x}_P) \in \Omega_{\boldsymbol X}^{P+1}$, $\overline{{\boldsymbol y}}=({\boldsymbol y}_1,\ldots,{\boldsymbol y}_P)\in \Omega_{\boldsymbol Y}^P$ and ${\boldsymbol c} \in \Omega_{\boldsymbol C}$.
\end{proof}

{\bf Theorem \ref{THE53}.} (Identification of conditional PNS with multi-hypothetical terms and evidence $({\boldsymbol y}',{\boldsymbol x}',{\boldsymbol c})$)
{\it
Under SCM ${\cal M}_{T}$ and Assumptions \ref{ASEXO2}, \ref{MONO2} (or \ref{AS2}, \ref{SAS2}), and \ref{SUP2}, we have %two following statements:

{\bf (A).} If $\rho({\boldsymbol y}';{\boldsymbol x}',{\boldsymbol c})\ne\rho^{o}({\boldsymbol y}';{\boldsymbol x}',{\boldsymbol c})$, then we have
 %($\rho^{o}({\boldsymbol y}';{\boldsymbol x}',{\boldsymbol c}) > \rho({\boldsymbol y}';{\boldsymbol x}',{\boldsymbol c})$),
    \begin{equation}
%\begin{aligned} 
%\label{CPNS3}
        \text{PNS}(\overline{{\boldsymbol y}};\overline{{\boldsymbol x}},{\boldsymbol y}',{\boldsymbol x}',{\boldsymbol c})=\max\left\{{\gamma}/{\delta},0\right\},
 %   \end{aligned}
\end{equation}
where 
\begin{equation}
\begin{aligned}
&\gamma=\min\Big\{\min_{p=1,\ldots,P}\{\rho({\boldsymbol y}_{p};{\boldsymbol x}_{p-1},{\boldsymbol c})\},\rho^{o}({\boldsymbol y}';{\boldsymbol x}',{\boldsymbol c}) \Big\}- \max\Big\{\max_{p=1,\ldots,P}\{\rho({\boldsymbol y}_{p};{\boldsymbol x}_p,{\boldsymbol c})\},\rho({\boldsymbol y}';{\boldsymbol x}',{\boldsymbol c})\Big\},\\
%\end{aligned}
%\end{equation}
%\begin{equation}
    &\delta=\rho^{o}({\boldsymbol y}';{\boldsymbol x}',{\boldsymbol c})-\rho({\boldsymbol y}';{\boldsymbol x}',{\boldsymbol c})
    \end{aligned}
\end{equation} for any ${\boldsymbol x}' \in \Omega_{\boldsymbol X}$, ${\boldsymbol y}' \in \Omega_{\boldsymbol Y}$, $\overline{{\boldsymbol x}}=({\boldsymbol x}_0,{\boldsymbol x}_1,\ldots,{\boldsymbol x}_P) \in \Omega_{\boldsymbol X}^{P+1}$, $\overline{{\boldsymbol y}}=({\boldsymbol y}_1,\ldots,{\boldsymbol y}_P)\in \Omega_{\boldsymbol Y}^P$, and ${\boldsymbol c} \in \Omega_{\boldsymbol C}$.

{\bf (B).} If $\rho({\boldsymbol y}';{\boldsymbol x}',{\boldsymbol c})=\rho^{o}({\boldsymbol y}';{\boldsymbol x}',{\boldsymbol c})$, then we have
%we have
\begin{equation}
\begin{aligned}
\text{PNS}(\overline{{\boldsymbol y}};\overline{{\boldsymbol x}},{\boldsymbol y}',{\boldsymbol x}',{\boldsymbol c})=\mathbb{I}\Big(\max_{p=1,\ldots,P}\{\rho({\boldsymbol y}_{p};{\boldsymbol x}_p,{\boldsymbol c})\} \leq
%\rho^{o}({\boldsymbol y}';{\boldsymbol x}',{\boldsymbol c})=
\rho({\boldsymbol y}';{\boldsymbol x}',{\boldsymbol c})< \min_{p=1,\ldots,P}\{\rho({\boldsymbol y}_{p};{\boldsymbol x}_{p-1},{\boldsymbol c})\}\Big)
\end{aligned}
\end{equation}
 for any ${\boldsymbol x}' \in \Omega_{\boldsymbol X}$, ${\boldsymbol y}' \in \Omega_{\boldsymbol Y}$, $\overline{{\boldsymbol x}}=({\boldsymbol x}_0,{\boldsymbol x}_1,\ldots,{\boldsymbol x}_P) \in \Omega_{\boldsymbol X}^{P+1}$, $\overline{{\boldsymbol y}}=({\boldsymbol y}_1,\ldots,{\boldsymbol y}_P)\in \Omega_{\boldsymbol Y}^P$, and ${\boldsymbol c} \in \Omega_{\boldsymbol C}$.
 }

\begin{proof}
Under Assumptions \ref{ASEXO2} and \ref{AS2}, 
    if 
    %$\mathbb{P}({\boldsymbol Y}={\boldsymbol y}'|{\boldsymbol X}={\boldsymbol x}', {\boldsymbol C}={\boldsymbol c})\ne 0 \Leftrightarrow \mathbb{P}({\boldsymbol Y}_{{\boldsymbol x}'}={\boldsymbol y}'| {\boldsymbol C}={\boldsymbol c})\ne 0$
    $\rho({\boldsymbol y}';{\boldsymbol x}',{\boldsymbol c})\ne\rho^{o}({\boldsymbol y}';{\boldsymbol x}',{\boldsymbol c})$, Eq. (\ref{CPNS3}) holds since we have
    \begin{equation}
    \begin{aligned}  
    &\text{PNS}(\overline{{\boldsymbol y}};\overline{{\boldsymbol x}},{\boldsymbol y}',{\boldsymbol x}',{\boldsymbol c})\\
    &=\mathbb{P}({\boldsymbol Y}_{{\boldsymbol x}_0} \prec {\boldsymbol y}_1 \preceq{\boldsymbol Y}_{{\boldsymbol x}_1}, {\boldsymbol Y}_{{\boldsymbol x}_1}\prec{\boldsymbol y}_2 \preceq {\boldsymbol Y}_{{\boldsymbol x}_2},\ldots, {\boldsymbol Y}_{{\boldsymbol x}_{P-1}} \prec {\boldsymbol y}_P\preceq{\boldsymbol Y}_{{\boldsymbol x}_P}|{\boldsymbol Y}={\boldsymbol y}',{\boldsymbol X}={\boldsymbol x}',{\boldsymbol C}={\boldsymbol c})\\
    &=\frac{\mathbb{P}({\boldsymbol Y}_{{\boldsymbol x}_0} \prec {\boldsymbol y}_1 \preceq{\boldsymbol Y}_{{\boldsymbol x}_1}, {\boldsymbol Y}_{{\boldsymbol x}_1}\prec{\boldsymbol y}_2 \preceq {\boldsymbol Y}_{{\boldsymbol x}_2},\ldots, {\boldsymbol Y}_{{\boldsymbol x}_{P-1}} \prec {\boldsymbol y}_P\preceq{\boldsymbol Y}_{{\boldsymbol x}_P},{\boldsymbol Y}={\boldsymbol y}',{\boldsymbol X}={\boldsymbol x}'|{\boldsymbol C}={\boldsymbol c})}{\mathbb{P}({\boldsymbol Y}_{{\boldsymbol x}'}={\boldsymbol y}'|{\boldsymbol C}={\boldsymbol c})}\\
    &=\frac{\mathbb{P}({\boldsymbol u}_{\rho({\boldsymbol y}_{1};{\boldsymbol x}_0,{\boldsymbol c})} \preceq{\boldsymbol u}\prec{\boldsymbol u}_{\rho({\boldsymbol y}_{1};{\boldsymbol x}_1,{\boldsymbol c})},\ldots,{\boldsymbol u}_{\rho({\boldsymbol y}_{P};{\boldsymbol x}_{P-1},{\boldsymbol c})} \preceq{\boldsymbol u}\prec{\boldsymbol u}_{\rho({\boldsymbol y}_{P};{\boldsymbol x}_{P},{\boldsymbol c})},{\boldsymbol u}_{\rho({\boldsymbol y}';{\boldsymbol x}',{\boldsymbol c})} \preceq {\boldsymbol u}\prec{\boldsymbol u}_{\rho^o({\boldsymbol y}';{\boldsymbol x}',{\boldsymbol c})})}{\mathbb{P}({\boldsymbol u}_{\rho^{o}({\boldsymbol y}';{\boldsymbol x}',{\boldsymbol c})} \preceq {\boldsymbol u}\prec{\boldsymbol u}_{\rho({\boldsymbol y}';{\boldsymbol x}',{\boldsymbol c})})}\\
    &=\frac{\max\{\min\{\min_{p=1,\ldots,P}\{\rho({\boldsymbol y}_{p};{\boldsymbol x}_{p-1},{\boldsymbol c})\},\rho^{o}({\boldsymbol y}';{\boldsymbol x}',{\boldsymbol c})\}-\max\{\max_{p=1,\ldots,P}\{\rho({\boldsymbol y}_{p};{\boldsymbol x}_p,{\boldsymbol c})\},\rho({\boldsymbol y}';{\boldsymbol x}',{\boldsymbol c})\},0\}}{\rho^{o}({\boldsymbol y}';{\boldsymbol x}',{\boldsymbol c})-\rho({\boldsymbol y}';{\boldsymbol x}',{\boldsymbol c})}\\
    &=\max\left\{\frac{\min\{\min_{p=1,\ldots,P}\{\rho({\boldsymbol y}_{p};{\boldsymbol x}_{p-1},{\boldsymbol c})\},\rho^{o}({\boldsymbol y}';{\boldsymbol x}',{\boldsymbol c})\}-\max\{\max_{p=1,\ldots,P}\{\rho({\boldsymbol y}_{p};{\boldsymbol x}_p,{\boldsymbol c})\},\rho({\boldsymbol y}';{\boldsymbol x}',{\boldsymbol c})\}}{\rho^{o}({\boldsymbol y}';{\boldsymbol x}',{\boldsymbol c})-\rho({\boldsymbol y}';{\boldsymbol x}',{\boldsymbol c})},0\right\}
    \end{aligned}
\end{equation}
for any ${\boldsymbol x}' \in \Omega_{\boldsymbol X}$, ${\boldsymbol y}' \in \Omega_{\boldsymbol Y}$, $\overline{{\boldsymbol x}}=({\boldsymbol x}_0,{\boldsymbol x}_1,\ldots,{\boldsymbol x}_P) \in \Omega_{\boldsymbol X}^{P+1}$, $\overline{{\boldsymbol y}}=({\boldsymbol y}_1,\ldots,{\boldsymbol y}_P)\in \Omega_{\boldsymbol Y}^P$ and ${\boldsymbol c} \in \Omega_{\boldsymbol C}$.
This represents the statement (A).
Otherwise, since $Y_x({\boldsymbol u}_{\rho({\boldsymbol y}';{\boldsymbol x}',{\boldsymbol c})})=Y_x({\boldsymbol u}_{\rho^o({\boldsymbol y}';{\boldsymbol x}',{\boldsymbol c})})$, we have 
\begin{equation}
    \begin{aligned}
        &\text{PNS}(\overline{{\boldsymbol y}};\overline{{\boldsymbol x}},{\boldsymbol y}',{\boldsymbol x}',{\boldsymbol c})\\
    &=\mathbb{P}({\boldsymbol Y}_{{\boldsymbol x}_0} \prec {\boldsymbol y}_1 \preceq{\boldsymbol Y}_{{\boldsymbol x}_1}, {\boldsymbol Y}_{{\boldsymbol x}_1}\prec{\boldsymbol y}_2 \preceq {\boldsymbol Y}_{{\boldsymbol x}_2},\ldots, {\boldsymbol Y}_{{\boldsymbol x}_{P-1}} \prec {\boldsymbol y}_P\preceq{\boldsymbol Y}_{{\boldsymbol x}_P}|{\boldsymbol Y}={\boldsymbol y}',{\boldsymbol X}={\boldsymbol x}',{\boldsymbol C}={\boldsymbol c})\\
    &=\mathbb{P}({\boldsymbol u}_{\rho({\boldsymbol y}_{1};{\boldsymbol x}_0,{\boldsymbol c})} \preceq{\boldsymbol u}\prec{\boldsymbol u}_{\rho({\boldsymbol y}_{1};{\boldsymbol x}_1,{\boldsymbol c})},\ldots,{\boldsymbol u}_{\rho({\boldsymbol y}_{P};{\boldsymbol x}_{P-1},{\boldsymbol c})} \preceq{\boldsymbol u}\prec{\boldsymbol u}_{\rho({\boldsymbol y}_{P};{\boldsymbol x}_{P},{\boldsymbol c})}|{\boldsymbol u}={\boldsymbol u}_{\rho({\boldsymbol y}';{\boldsymbol x}',{\boldsymbol c})})\\
    &=\mathbb{I}({\boldsymbol u}_{\rho({\boldsymbol y}_{1};{\boldsymbol x}_0,{\boldsymbol c})} \preceq{\boldsymbol u}_{\rho({\boldsymbol y}';{\boldsymbol x}',{\boldsymbol c})}\prec{\boldsymbol u}_{\rho({\boldsymbol y}_{1};{\boldsymbol x}_1,{\boldsymbol c})},\ldots,{\boldsymbol u}_{\rho({\boldsymbol y}_{P};{\boldsymbol x}_{P-1},{\boldsymbol c})} \preceq{\boldsymbol u}_{\rho({\boldsymbol y}';{\boldsymbol x}',{\boldsymbol c})}\prec{\boldsymbol u}_{\rho({\boldsymbol y}_{P};{\boldsymbol x}_{P},{\boldsymbol c})})\\
    &=\mathbb{I}({\boldsymbol u}_{\max_{p=1,\ldots,P}\{\rho({\boldsymbol y}_{p};{\boldsymbol x}_p,{\boldsymbol c})\}} \preceq{\boldsymbol u}_{\rho({\boldsymbol y}';{\boldsymbol x}',{\boldsymbol c})}\prec{\boldsymbol u}_{\min_{p=1,\ldots,P}\{\rho({\boldsymbol y}_{p};{\boldsymbol x}_{p-1},{\boldsymbol c})\}})\\
    &=\mathbb{I}\Big(\max_{p=1,\ldots,P}\{\rho({\boldsymbol y}_{p};{\boldsymbol x}_p,{\boldsymbol c})\} \leq
%\rho^{o}({\boldsymbol y}';{\boldsymbol x}',{\boldsymbol c})=
\rho({\boldsymbol y}';{\boldsymbol x}',{\boldsymbol c})< \min_{p=1,\ldots,P}\{\rho({\boldsymbol y}_{p};{\boldsymbol x}_{p-1},{\boldsymbol c})\}\Big)
    \end{aligned}
\end{equation}
for any $\overline{{\boldsymbol x}}=({\boldsymbol x}_0,{\boldsymbol x}_1,\ldots,{\boldsymbol x}_P) \in \Omega_{\boldsymbol X}^{P+1}$, $\overline{{\boldsymbol y}}=({\boldsymbol y}_1,\ldots,{\boldsymbol y}_P)\in \Omega_{\boldsymbol Y}^P$, ${\boldsymbol x}' \in \Omega_{\boldsymbol X}$, ${\boldsymbol c} \in \Omega_{\boldsymbol C}$, ${\boldsymbol y}' \in \Omega_{\boldsymbol Y}$ and ${\boldsymbol y} \in \Omega_{\boldsymbol Y}$.
This represents the statement (B).
\end{proof}

\section{Additional Information on Application}
\label{appB}

%\subsection{Additional Information on Application}

In this section, we provide additional information on the application.

\subsection{Details of Dataset}
First, we explain all variables in the application. 
We pick up the following variables as {\bf outcomes}.
\begin{enumerate}
%these grades are related with the course subject, Math or Portuguese:
\item G1 - first period grade (numeric: from 0 to 20)
\item G2 - second period grade (numeric: from 0 to 20)
\item G3 - final grade (numeric: from 0 to 20, output target)
\end{enumerate}
We pick up the following variables as {\bf treatments}.
\begin{enumerate}
    \item studytime - weekly study time (numeric: 1 - < 2 hours, 2 - 2 to 5 hours, 3 - 5 to 10 hours, or 4 - >10 hours)
    \item paid - extra paid classes within the course subject (Math or Portuguese) (binary: yes or no)
\end{enumerate}
We show the other variables as potential {\bf covariates}.
\begin{enumerate}
\item school - student's school (binary: 'GP' - Gabriel Pereira or 'MS' - Mousinho da Silveira)
\item sex - student's sex (binary: 'F' - female or 'M' - male)
\item age - student's age (numeric: from 15 to 22)
\item address - student's home address type (binary: 'U' - urban or 'R' - rural)
\item famsize - family size (binary: 'LE3' - less or equal to 3 or 'GT3' - greater than 3)
\item Pstatus - parent's cohabitation status (binary: 'T' - living together or 'A' - apart)
\item Medu - mother's education (numeric: 0 - none,  1 - primary education (4th grade), 2 “ 5th to 9th grade, 3 “ secondary education or 4 “ higher education)
\item Fedu - father's education (numeric: 0 - none,  1 - primary education (4th grade), 2 “ 5th to 9th grade, 3 “ secondary education or 4 “ higher education)
\item Mjob - mother's job (nominal: 'teacher', 'health' care related, civil 'services' (e.g. administrative or police), 'at home' or 'other')
\item Fjob - father's job (nominal: 'teacher', 'health' care related, civil 'services' (e.g. administrative or police), 'at home' or 'other')
\item reason - reason to choose this school (nominal: close to 'home', school 'reputation', 'course' preference or 'other')
\item guardian - student's guardian (nominal: 'mother', 'father' or 'other')
\item traveltime - home to school travel time (numeric: 1 - <15 min., 2 - 15 to 30 min., 3 - 30 min. to 1 hour, or 4 - >1 hour)
\item failures - number of past class failures (numeric: n if 1<=n<3, else 4)
\item schoolsup - extra educational support (binary: yes or no)
\item famsup - family educational support (binary: yes or no)
\item activities - extra-curricular activities (binary: yes or no)
\item nursery - attended nursery school (binary: yes or no)
\item higher - wants to take higher education (binary: yes or no)
\item internet - Internet access at home (binary: yes or no)
\item romantic - with a romantic relationship (binary: yes or no)
\item famrel - quality of family relationships (numeric: from 1 - very bad to 5 - excellent)
\item freetime - free time after school (numeric: from 1 - very low to 5 - very high)
\item goout - going out with friends (numeric: from 1 - very low to 5 - very high)
\item Dalc - workday alcohol consumption (numeric: from 1 - very low to 5 - very high)
\item Walc - weekend alcohol consumption (numeric: from 1 - very low to 5 - very high)
\item health - current health status (numeric: from 1 - very bad to 5 - very good)
\item absences - number of school absences (numeric: from 0 to 93)
\end{enumerate}
We show the attributes of ID number $1$ in Table \ref{Atab1}.

\begin{table}[tb]
\renewcommand{\arraystretch}{1.1}
\centering
\caption{Attributes of the ID number $1$ subject.}
\label{Atab1}
\begin{tabular}{lllllllllll}
\hline \hline
school & sex & age & address & famsize & Pstatus & Medu & Fedu & Mjob     & Fjob    & reason \\
\hline
GP     & F   & 18  & U       & GT3     & A       & 4    & 4    & at\_home & teacher & course\\
\hline
\hline
guardian & traveltime & studytime & failures & schoolsup & famsup & paid & activities & nursery & higher & internet \\
\hline
mother   & 2          & 2         & 0        & yes       & no     & no   & no         & yes     & yes    & no      \\
\hline
\hline
romantic & famrel & freetime & goout & Dalc & Walc & health & absences & G1 & G2 & G3 \\
\hline
no       & 4      & 3        & 4     & 1    & 1    & 3      & 6        & 5  & 6  & 6 \\
\hline
\end{tabular}
\end{table}

\subsection{Additional Analyses of Application}
We give three additional analyses of the four applications in the body of the paper.

{\bf Effect of study time only.}
First, we evaluate conditional PNS, PN, and PS, letting ${\boldsymbol y}=(6,6,6)$, ${\boldsymbol x}_0=(2,1)$, ${\boldsymbol x}_1=(4,1)$, and ${\boldsymbol c}_1$ in Def. \ref{def41}.
The estimated values of conditional PNS, PN, and PS are 
\begin{equation}
    \begin{aligned}
       &\text{PNS:} &2.491 \% &(\text{CI}: [0.000\%,7.395\%]),\\
       &\text{PN:} &2.709 \% &(\text{CI}: [0.000\%,8.060\%]),\\
      &\text{PS:}  &25.864 \% &(\text{CI}: [0.000\%,73.544\%]),
    \end{aligned}
\end{equation}
respectively.
Second, we evaluate conditional PNS with evidence $({\boldsymbol y}',{\boldsymbol x}',{\boldsymbol c})$, letting ${\boldsymbol y}=(6,6,6)$, ${\boldsymbol y}'=(6,6,5)$, ${\boldsymbol x}_0=(2,1)$,
${\boldsymbol x}_1=(4,1)$, 
${\boldsymbol x}'=(2,1)$, and ${\boldsymbol c}_1$ in Def. \ref{EV1}.
Then, the estimated value of it is 
\begin{equation}
     \text{PNS:}\ \  0.000 \%\ \ \ \  (\text{CI}: [0.000\%,0.000\%]).
\end{equation}
Third, we evaluate conditional PNS with multi-hypothetical terms, letting ${\boldsymbol y}_1=(5,5,5)$, ${\boldsymbol y}_2=(6,6,6)$, ${\boldsymbol x}_0=(1,1)$, ${\boldsymbol x}_1=(2,1)$, ${\boldsymbol x}_2=(4,1)$, and ${\boldsymbol c}_1$ in Def. \ref{EV2}.
The estimated value of it is 
\begin{equation}
     \text{PNS:}\ \  0.000 \%\ \ \ \  (\text{CI}: [0.000\%,0.000\%]).
\end{equation}
Finally, we evaluate conditional PNS with multi-hypothetical terms and evidence $({\boldsymbol y}',{\boldsymbol x}',{\boldsymbol c})$, letting ${\boldsymbol y}_1=(5,5,5)$, ${\boldsymbol y}_2=(6,6,6)$, ${\boldsymbol y}'=(6,6,5)$, ${\boldsymbol x}_0=(1,1)$, ${\boldsymbol x}_1=(2,1)$,
${\boldsymbol x}_2=(4,1)$, 
${\boldsymbol x}'=(2,1)$, and ${\boldsymbol c}_1$ in Def. \ref{EV3} .
We eliminate the results of NA, then the estimated value of it is 
\begin{equation}
     \text{PNS:}\ \  42.489 \%\ \ \ \  (\text{CI}: [0.000\%,100.000\%]).
\end{equation}

{\bf Effect of extra paid classes only.}
First, we evaluate conditional PNS, PN, and PS, letting ${\boldsymbol y}=(6,6,6)$, ${\boldsymbol x}_0=(1,1)$, ${\boldsymbol x}_1=(2,2)$, and ${\boldsymbol c}_1$ in Def. \ref{def41}.
The estimated values of conditional PNS, PN, and PS are 
\begin{equation}
    \begin{aligned}
       &\text{PNS:} &7.700 \% &(\text{CI}: [1.072\%,16.614\%]),\\
       &\text{PN:} &8.132\% &(\text{CI}: [1.090\%,18.139\%]),\\
      &\text{PS:}  &65.398 \% &(\text{CI}: [37.015\%,89.795\%]),
    \end{aligned}
\end{equation}
respectively.
Second, we evaluate conditional PNS with evidence $({\boldsymbol y}',{\boldsymbol x}',{\boldsymbol c})$, letting ${\boldsymbol y}=(6,6,6)$, ${\boldsymbol y}'=(6,6,5)$, ${\boldsymbol x}_0=(1,1)$,
${\boldsymbol x}_1=(2,2)$, 
${\boldsymbol x}'=(2,1)$, and ${\boldsymbol c}_1$ in Def. \ref{EV1}.
Then, the estimated value of it is 
\begin{equation}
     \text{PNS:}\ \  0.009 \%\ \ \ \  (\text{CI}: [0.000\%,0.139\%]).
\end{equation}

\end{document}